\nonstopmode
\documentclass[twocolumn]{svjour3}          
\usepackage{amsmath}
\usepackage{amsmath}
\usepackage{amsfonts}
\usepackage{amssymb}
\usepackage{graphicx}
\usepackage{lmodern}
\usepackage{kpfonts}
\usepackage{color}
\usepackage{algorithm}
\usepackage{algorithmicx}
\usepackage{algpseudocode}
\usepackage{caption}
\usepackage{url}
\usepackage{bbm}

\newcommand{\algrule}[1][.2pt]{\par\vskip.5\baselineskip\hrule height #1\par\vskip.5\baselineskip}

\usepackage{subfigure}
\usepackage{amsbsy}
\usepackage{bm}
\usepackage{epsfig}
\usepackage{topcapt}

\newcommand{\hide}[1]{}
\newcommand{\LB}{LB}
\newcommand{\UB}{UB}
\newcommand{\R}{\mathbb{R}}
\newcommand{\C}{\mathbb{C}}
\newcommand{\DB}{\mathcal{DB}}
\newcommand{\F}{\mathcal{F}}
\newcommand{\IF}{\mathcal{F}^{-1}}
\newcommand{\x}{\mathbf{x}}
\newcommand{\X}{\mathbf{X}}
\newcommand{\y}{\mathbf{y}}
\newcommand{\U}{\mathbf{U}}
\newcommand{\Y}{\mathbf{Y}}

\newcommand{\q}{\mathbf{q}}
\newcommand{\Q}{\mathbf{Q}}
\newcommand{\z}{\mathbf{z}}
\newcommand{\opt}{\text{opt}}

\newcommand{\norm}[1]{\left|\left|#1\right|\right|_2}
\newcommand{\tab}{\hspace*{2em}}
\newcommand{\tabb}{\hspace*{4em}}
\newcommand{\tabbb}{\hspace*{6em}}
\DeclareMathOperator*{\minmax}{min\,(max)}

\begin{document}

\title{Compressive Mining: Fast and Optimal Data Mining \\ in the Compressed Domain}
\titlerunning{Compressive Mining}        
\author{Michail Vlachos \and Nikolaos M. Freris \and Anastasios Kyrillidis}

%

\institute{M. Vlachos and A. Kyrillidis \at
            IBM-Research Z\"{u}rich, S\"{a}umerstrasse 4, CH-8803, R\"{u}schlikon, Switzerland\\  
           \email{\{mvl,nas\}@zurich.ibm.com}           
          \and
          N. Freris \at
          School of Computer and Communication Sciences, \'{E}cole Polytechnique F\'{e}d\'{e}rale de Lausanne (EPFL), CH-1015 Lausanne, Switzerland\\ 
          \email{nikolaos.freris@epfl.ch}
}          
\date{Received: date / Accepted: date}

\sloppy

\maketitle
\begin{abstract}
Real-world data typically contain repeated and periodic patterns.
This suggests that they can be effectively represented and compressed 
using only a few coefficients of an appropriate
basis (e.g., Fourier, Wavelets, etc.).
However, distance estimation when the data are represented using different sets of coefficients
is still a largely unexplored area. This work studies the optimization problems related to obtaining the \emph{tightest} lower/upper bound on Euclidean distances when each data object is potentially compressed using a different set of orthonormal coefficients.
Our technique leads to tighter distance estimates, which translates into more accurate search, learning and mining operations 
\textit{directly} in the compressed domain.

We formulate the problem of estimating lower/upper distance bounds as an optimization problem.
We establish the properties of optimal solutions, and leverage the theoretical analysis to develop a fast algorithm to obtain an \emph{exact} solution to the problem. The suggested solution provides the tightest estimation of the $L_2$-norm or the correlation. 
We show that typical data-analysis operations, such as k-NN search or k-Means clustering, can operate more accurately
using the proposed compression and distance reconstruction technique. We compare it with many other prevalent compression
and reconstruction techniques, including random projections and PCA-based techniques. We highlight a surprising result, namely
that when the data are highly sparse in some basis, our technique may even outperform PCA-based compression.

The contributions of this work are generic as our methodology is applicable to any sequential or high-dimensional data as well as to any orthogonal data transformation used for the underlying data compression scheme.

\end{abstract}
\section{Introduction}

Increasing data sizes are a perennial problem for data analysis.
This dictates the need not only for more efficient data-compression schemes, but
also for analytic operations that work directly in the compressed domain.
Compression schemes exploit inherent patterns and structures in the data. 
In fact, many natural and industrial processes exhibit patterns
and periodicities. Periodic behavior is omnipresent; be it
in environmental and natural processes
\cite{periodicOceanography,periodicAstronomy}, in medical and physiological measurements (e.g., ECG data \cite{frerisECG}),
weblog data \cite{immorlicaWWW05,RosieJones05}, or network measurements \cite{AKAMAI}.
The aforementioned are only a few of the numerous scientific and industrial fields that exhibit repetitions.
Examples from some of these areas are shown in Fig. \ref{fig:periodicDataExamples}.

\begin{figure}[!htpb]
\centerline{\includegraphics[width=\linewidth]{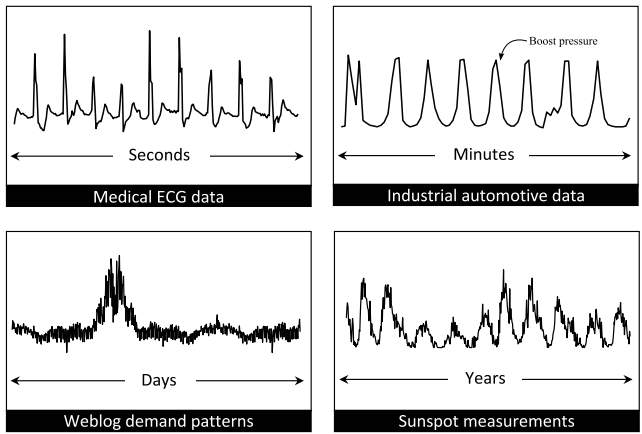}}
\vspace{-\baselineskip}
\caption{Many scientific fields entail periodic data. Examples from 
medical, industrial, web and astronomical measurements.}
\label{fig:periodicDataExamples}
\end{figure}

\hide{
the majority of environmental, medical and industrial data exhibit strong periodic patterns, among others:
	\rule[1]{0.1in}{0.1in} Environmental and natural processes such as tidal patterns
(oceanography), sunspots \cite{} (astronomy), temperature changes \cite{periodicMeteorology} (meteorology), etc. manifest strong or weak periodic behavior. Such data are constantly recorded and monitored nowadays using sensors \cite{sensorsSurvey}.
For storage considerations, there is great need to compress the data, and yet perform data analysis in the compressed domain.
	\item Medical and physiological measurements such as heart-beat rate,
	respiration, body part movement, and so on, are characterized by
	a precise periodic behavior, as well.
	\item Most of the transmitted electrical and wireless signals
	are also encoded on top of diverse periodic bases.
	\item Industrial measurements including engine vibrations and rotary movement of mechanical parts clearly manifest repeated patterns.
\end{itemize}
}

When data contain an inherent structure, more efficient compression can be performed
with minimal loss in data quality (see Fig.~\ref{fig:overview} for an example). 
The bulk of related work on compression and distance estimation
used the same sets of coefficients for all objects \cite{Agrawal93,rafiei98,adaFuWavelet,Eruhimov2006}. This simplified
the distance estimation in the compressed domain. However, by encoding the data using only a few and potentially disjoint sets of \emph{high-energy coefficients} (i.e., coefficients of highest absolute value) in an
orthonormal basis, one can achieve better reconstruction performance. Nonetheless, it was not known how to compute tight distance estimates using such a representation.
Our work exactly addresses this issue: given data that are compressed  
using disjoint coefficient sets of an orthonormal basis (for reasons of higher fidelity), \textit{how can distances among the compressed
objects be estimated with the highest fidelity}?

Here, we provide the tightest possible upper and lower bounds on the original distances, based only on the compressed
objects. By \textit{tightest}, we mean that, given the information available, no better estimate can be derived. Distance estimation is fundamental for data mining: the majority of mining and learning tasks are distance-based, including clustering (e.g. k-Means or hierarchical), k-NN classification, outlier detection, pattern matching, etc. 
This work focuses on the case where the distance is the widely used Euclidean distance ($L_2$-norm), 
but makes no assertions on the underlying transform used to compress the data:
As long as the transform is orthonormal, our methodology is applicable.
In the experimental section, we use both Fourier and Wavelets Decomposition as a data compression technique. 
Our main \textbf{contributions} are summarized below:
\vspace{.3\baselineskip}
  
  - We formulate the problem of tight distance estimation in the compressed domain
  as two optimization problems for obtaining lower/upper bounds.
  We show that both problems can be solved simultaneously by solving a single convex optimization program.
  
  - We derive the necessary and sufficient Karush-Kuhn-Tucker (KKT) conditions and study the properties of optimal solutions. We use the analysis to devise exact closed-form solution algorithms for the optimal distance bounds.
      
  - We evaluate our analytical findings experimentally;
  we compare the proposed algorithms with prevalent distance estimation schemes, and demonstrate significant improvements in terms of estimation accuracy.  We further compare the performance of our optimal algorithm with that of a numerical scheme based on convex optimization, and show that our scheme is at least two orders of magnitude faster, while also providing more accurate results. 
  
  - We also provide extensive evaluations with mining tasks in the compressed domain using our approach
  and many other prevalent compression and distance reconstruction schemes used in the literature (random projections, SVD, etc).



\section{Related Work}

\begin{figure*}[!ht]
\centerline{\includegraphics[width=0.77\linewidth]{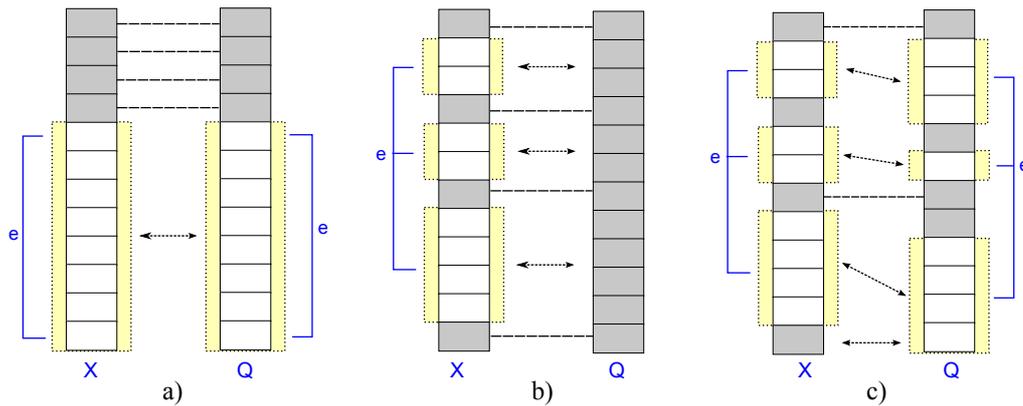}}
\vspace{-\baselineskip}
\caption{Comparison with previous work.
Distance estimation between a compressed sequence (X) and a query (Q) represented in \emph{any} complete orthonormal basis.
A compressed sequence is represented by a set of stored coefficients (gray) as well as the error $e$ incurred because of compression (yellow). a) Both X,Q are compressed by storing the first coefficients. b) The highest-energy coefficients are used for X, whereas Q is uncompressed as in \cite{vlachosTWEB,vlachosSDM09}.
c) The problem we address: both sequences are compressed using the highest-energy coefficients;
note that in general for each object a different set of coefficients is used.}
\label{fig:visualComparisonFirstBest}
\end{figure*}

We briefly position our work in the context of other similar approaches in the area. 
The majority of data-compression techniques for sequential data use the \emph{same} set of low-energy coefficients
whether using Fourier \cite{Agrawal93,rafiei98}, Wavelets \cite{adaFuWavelet,Eruhimov2006} or Chebyshev polynomials \cite{ChebyshevNg} as the orthogonal basis for representation and compression. Using the same set of orthogonal coefficients has several advantages: a) it is straightforward to compare the respective coefficients;
b) space partitioning and indexing structures (such as R-trees) can be directly used on the compressed data;
c) there is no need to store also the indices (position) of the basis functions to which the stored coefficients correspond.
The disadvantage is that both object reconstruction and distance estimation may be far from optimal for a given fixed compression ratio.


One can also record side information, such as the energy of the discarded coefficients, to better approximate the distance between compressed sequences by leveraging the Cauchy--Schwartz inequality \cite{WangMultiLevel}. This is shown in Figure \ref{fig:visualComparisonFirstBest}a). In~\cite{vlachosSDM09,vlachosTWEB}, the authors advocated the use of high-energy coefficients and side information on the discarded coefficients for weblog sequence
repositories; in that setting one of the sequences was
compressed, whereas the query was uncompressed, i.e., all coefficients were available as illustrated in Figure \ref{fig:visualComparisonFirstBest}b).
This work examines the most general and challenging case when both objects
are compressed. In such case, we record a (generally) different set of high-energy coefficients and also store
aggregate side information, such as the energy of the omitted data; this is depicted  in Figure \ref{fig:visualComparisonFirstBest}c).
We are not aware of any previous art addressing this problem to derive either optimal or suboptimal bounds on distance estimation.

The above approaches consider determining distance estimation in the compressed domain.
There is also a big body of work that considers probabilistic distance estimation
via low-dimensional embeddings. Several projection techniques for dimensionality reduction can preserve the geometry of the points \cite{dasgupta2000experiments,calderbank2009compressed}. These results heavily depend on the  work of Johnson and Lindenstrauss \cite{johnson1984extensions}, according to which any set of points can be projected onto a logarithmic (in the cardinality of the data points) dimensional subspace, while still retaining the relative distances between the points, thus preserving an approximation of their nearest neighbors \cite{indyk2007nearest,ailon2006approximate} or clustering \cite{boutsidis2010random,cardoso2012iterative}. Both random \cite{achlioptas2001database} and deterministic \cite{bingham2001random} constructions have been proposed in the literature. 

This paper extends and expands the work of \cite{SIAM12}. Here we include additional experiments that
show the performance of our methodology for $k$-NN-search, and $k$-Means clustering directly
in the compressed domain. We also compare our approach with the performance of Principal Components and Random Projection techniques
(both in the traditional and in the compressive sensing setting).
Finally, we also conduct experiments using other orthonormal bases (namely, wavelets) to demonstrate the generality of our technique. 
In the experimental section of this work, we compare our methodology to 
both deterministic and probabilistic techniques.




\section{Searching Data Using Distance Estimates}

\noindent

We consider a database $\DB$ that stores sequences as $N$-dimensional complex vectors
$\x^{(i)} \in \C^N, i =1,\dots,V$. A search problem that we examine is abstracted as follows:
a user is interested in finding the $k$ most `similar' sequences
to a given query sequence $\q \in \DB$, under a certain distance metric $d(\cdot,\cdot) : \C^{N\times N} \to \R_+$.
This is an elementary, yet fundamental operation known
as $k$-Nearest-Neighbor ($k$-NN)-search. It is a core function in database-querying, data-mining and machine-learning algorithms including classification (NN classifier), clustering, etc.

In this paper, we focus on the case where $d(\cdot, \cdot)$ is the Euclidean distance.
We note that other measures, e.g., time-invariant matching, can be formulated
as Euclidean distance on the periodogram \cite{vlachosSDM05}. Correlation can also be expressed
as an instance of Euclidean distance on properly normalized sequences \cite{MueenSigmod2010}. 
Therefore, our approach is applicable on a wide range of distance measures with little or no modification.
However, for ease of exposition, we focus on the Euclidean distance
as the most used measure in the literature~\cite{keoghBenchmark}.

Search operations can be quite costly, especially for cases when
the dimensionality $N$ of data is high: sequences
need to be retrieved from the disk for comparison against the query $\q$.
An effective way to mitigate this is to retain a compressed
representation of the sequences to be used as an initial
pre-filtering step. The set of compressed sequences could be small enough
to keep in-memory, hence enabling a significant performance speedup.
In essence, this is a \textit{multilevel} filtering mechanism.
With only the compressed sequences available, we obviously cannot infer the exact
distance between the query $\q$ and a sequence $\x^{(i)}$ in the database.
However, it is still plausible to obtain \textit{lower} and \textit{upper bounds} of the distance. 
Using these bounds, one might request a superset of the $k$-NN answers, which will be then verified using the uncompressed
sequences that will need to be fetched and compared with the query, so that the exact distances can be computed.
Such filtering ideas are used in the majority of the data-mining literature
for speeding up search operations \cite{Agrawal93,rafiei98,keogh2009}.



\section{Notation}

\begin{figure*}[!htpb]
\centerline{\includegraphics[width=\linewidth]{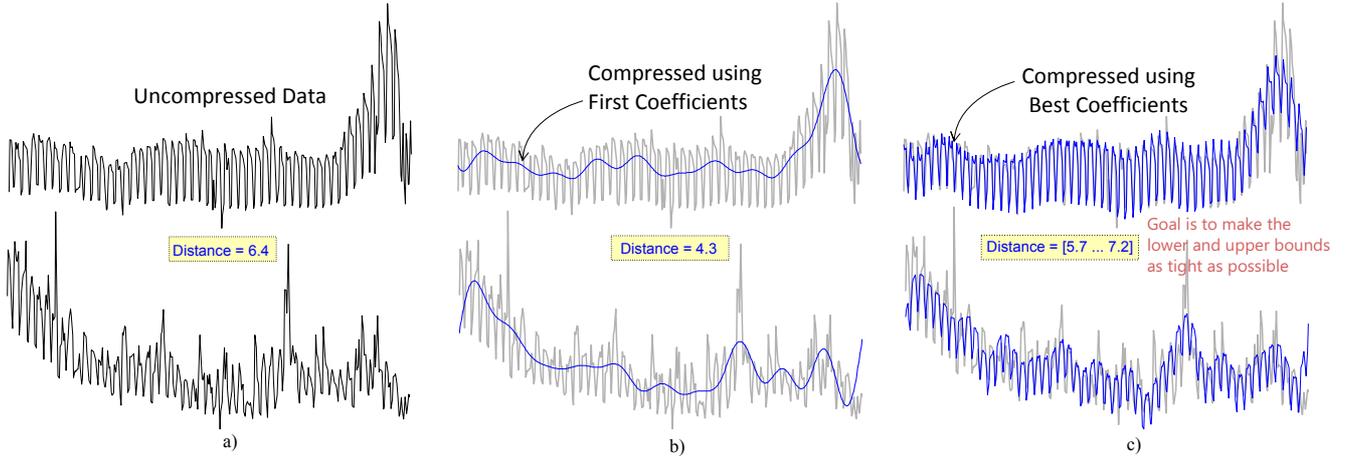}}
\vspace{-\baselineskip}
\caption{Motivation for using the high-energy (best) coefficients for compression. Using the best 10 coefficients
(c) results in significantly better sequence approximation than when using the first coefficients (b).}
\label{fig:overview}
\end{figure*}

Consider an $N$-dimensional sequence $\x = [x_1~ x_2~ \dots ~ x_N]^T \in \R^N$.
For compression purposes, $\x$ is first transformed using a sparsity-inducing (i.e., compressible) basis $\mathcal{F}(\cdot)$ in $\mathbb{R}^N$ or $\mathbb{C}^N$, such that $\X = \mathcal{F}(\x)$. We denote the \emph{forward} linear mapping $\x \to \X$ by $\F$, whereas the inverse linear map $\X \to \x$ is denoted by $\IF$, i.e., we say $\X =\F(\x)$ and $\x =\IF(\X)$.
A nonexhaustive list of invertible linear transformations includes Discrete Fourier Transform (DFT), Discrete Cosine Transform, 
Discrete Wavelet Transform, etc. 

As a running example for this paper, we assume that a sequence is compressed using DFT.
In this case, the basis represent sinusoids of different frequencies, and the pair $(\x,\X)$ satisfies
\vspace{-1.6\baselineskip}
\begin{center}
$$ X_l = \frac{1}{\sqrt{N}} \sum_{k=1}^N x_k e^{i 2 \pi (k-1)(l-1)/N}, \  l=1, \ldots, N$$
$$ x_k = \frac{1}{\sqrt{N}} \sum_{l=1}^N X_le^{i 2 \pi (k-1)(l-1)/N}, \  k=1, \ldots, N$$
\end{center}
\vspace{-.4\baselineskip}
\noindent
where $i$ is the imaginary unit $i^2=-1$. 

Given the above, we assume the $L_2$-norm as the distance between two sequences $\x$, $\q$, 
which can easily be translated into distance in the frequency domain because of Parseval's theorem \cite{oppenheim1999discrete}:
$$d(\x,\q) := \norm{\x - \q} =   \norm{\X - \Q}\nonumber$$

In the experimental section, we also show applications of our methodology when
wavelets are used as the signal decomposition transform.

\section{Motivation}

The choice of which coefficients to use has a direct impact on the data approximation quality.
It has long been recognized that sequence approximation is indeed superior when using
high-energy coefficients \cite{keogh01,vlachosSDM09}; in fact, using high-energy coefficients corresponds to optimal $L_2$ compression--as indicated by Parseval's theorem--hence, we also use the term `'best coefficients'' to refer to the high-energy coefficients maintained during compression; see also Figure \ref{fig:overview} for an illustrative example - However, a barrier
still has to be overcome when using optimal $l_2$ compression: the efficiency of solution for distance estimation.

Consider a sequence represented using its high-energy coefficients.
Then, the compressed sequence will be described
by a set of $C_x$ coefficients that hold the largest energy.
We denote the vector describing the positions of those
coefficients in $\X$ as $p^+_x$, 
and the positions of the
remaining ones as $p^-_x$ (that is, $p^+_x \cup p^-_x = \lbrace 1, \dots, N\rbrace$).
For any sequence $\X$, we store the
vector $\X(p^+_x)$ in the database, which we denote simply by $\X^+ := \{X_i\}_{i\in p^+_x}$.
We denote the vector of discarded coefficients by $\X^- := \{X_i\}_{i\in p^-_x}$.
In addition to the best coefficients of a sequence, we can
also record one additional value for the energy of the compression error,
$e_x = \norm{\X^-}^2$, i.e., the sum of squared magnitudes of
the omitted coefficients.

Then, one needs to solve the following minimization (maximization) problem
for calculating the lower (upper) bounds on the distance between two sequences based on their compressed versions:
\begin{align}\label{opt_d}
\begin{aligned}
	\minmax_{\X^{-} \in \C^{|p^-_x|} ,~\Q^{-} \in \C^{|p^-_q|}} & \tab \norm{\X-\Q}\\
\mbox{s.t. } \ \ \  & |X_l^{-}| \le \min_{j\in p^+_x} |X_j|, \ \ \forall l \in p^-_x\\
& |Q_l^{-}| \le \min_{j\in p^+_q} |Q_j|, \ \ \forall l \in p^-_q\\
& \sum_{l\in p^-_x} |X_l^{-}|^2 = e_x , \sum_{l\in p^-_q} |Q_l^{-}|^2 = e_q
\end{aligned}
\end{align}

\noindent
The inequality constraints are due to the fact that we use the high-energy components for the compression.
Hence, any of the omitted components must have an energy lower than the minimum energy
of any kept component.


The optimization problem presented is a complex-valued program: we show a single real-valued \emph{convex program} that is \emph{equivalent} to both the minimization and maximization problems. 
This program can be solved efficiently with numerical methods~\cite{BV04}, cf. Sec. \ref{algo:numerical}.
However, as we show in the experimental section, evaluating an instance of this problem is not efficient in practice, even for a single pair
of sequences. 
Therefore, although a solution can be found \textit{numerically}, it is generally costly and not suitable for large mining tasks,
where one would like to evaluate thousands or millions of lower/upper bounds on compressed sequences.

In this paper, we show how to solve this problem \textbf{analytically} by exploiting the derived optimality conditions. 
In this manner we can solve the problem in a fraction of the time required by numerical methods.
We solve this problem as a `double-waterfilling' instance.
Vlachos et al. have shown how the optimal lower and upper
distance bounds between a compressed and an uncompressed sequence
can be relegated to a single waterfilling problem \cite{vlachosSDM09}.
We revisit this approach as it will be used as a building block for our solution.
In addition, we later derive optimality properties for our solution.



\section{An Equivalent Convex Optimization Problem}
\label{real-opt}

For ease of notation, we consider the partition $\mathcal{P} = \{P_0,P_1,P_2,P_3\}$ of $\{1,\hdots,N\}$
(see Fig. \ref{fig:pSets}), where we set the following:

\begin{itemize}
	\item $P_0 = p^+_x\cap p^+_q$ are the common known components in two compressed sequences $\X,\Q$.
	\item $P_1 = p^-_x\cap p^+_q$ are the components unknown for $\X$ but known for $\Q$.
	\item $P_2 = p^+_x\cap p^-_q$ are the components known for $\X$ but unknown for $\Q$.
	\item $P_3 = p^-_x\cap p^-_q$ are the components unknown for both sequences.
\end{itemize}
\begin{figure}[!h]
\centerline{\includegraphics[width=0.5\linewidth]{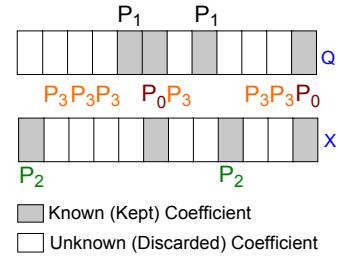}}
\vspace{-0.5\baselineskip}
\caption{Visual illustration of sets $P_0,P_1,P_2,P_3$ between two compressed objects.}
\label{fig:pSets}
\end{figure}

Using the standard notation $\x^*$ for the \emph{conjugate transpose} of a complex vector $\x$, $\Re\lbrace \cdot \rbrace$ to denote the real part of a complex number, and considering all vectors as column vectors, we have that the squared Euclidean distance is given by:
\begin{eqnarray*}
\norm{\x-\q}^2 &=& \norm{\X-\Q}^2 = (\X-\Q)^*(\X-\Q)\nonumber\\
&=& \norm{\X}^2 + \norm{\Q}^2 - 2\X^*\Q\nonumber\\
&=& \norm{\X}^2 + \norm{\Q}^2 - 4\sum_{i=1}^N\Re\{X_i Q_i\}\nonumber\\
&=& \norm{\X}^2 + \norm{\Q}^2 - 4(\!\!\sum_{l\in P_0} \!\!\Re\{X_l Q_l\}\nonumber\\
& & + \!\!\sum_{l\in P_1} \!\!\Re\{X_l Q_l\} + \!\!\sum_{l\in P_2} \!\!\Re\{X_l Q_l\}\nonumber\\
& & + \!\!\sum_{l\in P_3} \!\!\Re\{X_l Q_l\}).
\end{eqnarray*}

Note that $\norm{\X}, \norm{\Q}$ can be inferred by summing the squared magnitudes of the known coefficients with the energy of the compression error. Also, the term $\sum_{l\in P_0} \!\!\Re\{X_l Q_l\}$ is known,
whereas the last three sums are unknown. Considering the polar form, i.e.,
absolute value $|\cdot|$ and argument $\text{arg}(\cdot)$
$$X_l = |X_l|e^{i\text{arg}(X_l)}, \ \ Q_l = |Q_l|e^{i\text{arg}(Q_l)},$$
\noindent
we have that the decision variables are vectors $|X_l|, \text{arg}(X_l), l\in p^-_x$ as well as $|Q_l|, \text{arg}(Q_l), l\in p^-_q$.
Observe that for $x,y \in \C$ with $|x|,|y|$ known, we have that
$-|x||y|\le\Re\{xy\}\le|x||y|$, where the upper bound is attained when $\text{arg}(x)+\text{arg}(y) = 0$ and the lower bound when $\text{arg}(x)+\text{arg}(y)=\pi$. 

Therefore, both problems (\ref{opt_d}) boil down to the real-valued optimization problem
\begin{align}
\label{opt2}
\min & \tab -\!\!\sum_{l\in P_1} \!\!a_l b_l - \!\!\sum_{l\in P_2} \!\!a_l b_l - \!\!\sum_{l\in P_3} \!\! a_l b_l\\
\mbox{s.t.} &  \tabbb 0\le a_l  \le A, \ \ \forall l \in p^-_x\nonumber\\
& \tabbb 0 \le b_l \le B, \ \ \forall l \in p^-_q\nonumber\\
& \tabbb \sum_{l\in p^-_x} a_l^2 \le e_x\nonumber\\
& \tabbb \sum_{l\in p^-_q} b_l^2 \le e_q\nonumber ,
\end{align}
\noindent
where $a_l, b_l$ represent $|X_l|, |Q_l|$, respectively, and $A:=\min_{j\in p^+_q} |X_j|, B:=\min_{j\in p^+_q} |Q_j|$.
Note also that we have relaxed the equality constraints to inequality constraints
as the objective function of (\ref{opt2}) is decreasing in all $a_i, b_i$, so the optimum of (\ref{opt2}) has to satisfy
the relaxed inequality constraints  with equality, because of the elementary property that $|p^-_x|A^2 \ge e_x,
|p^-_q|B^2 \ge e_q$. Recall that in the first sum only $\{a_i\}$ are known and in the second only $\{b_i\}$, whereas in the third all variables are unknown.

We have reduced the original problem to a single optimization program, which, however, is not convex unless $p^-_x\cap p^-_q = \emptyset$. It is easy to check that the constraint set is convex and compact; however, the bilinear function $f(x,y) := xy$ is convex in each argument alone, but \emph{not} jointly. We consider the re-parametrization of the decision variables $z_i = a_i^2$ for $i\in p^-_x$, and $y_i = b_i^2$ for $i\in p^-_q$, we set $Z := A^2, Y := B^2$ and get the equivalent problem:
\begin{align}
\label{opt3}
\min & \tab -\!\!\sum_{i\in P_1} \!\!b_i \sqrt{z}_i - \!\!\sum_{i\in P_2} \!\!a_i \sqrt{y}_i  - \!\!\sum_{i\in P_3} \!\! \sqrt{z}_i\sqrt{y}_i\\
\mbox{s.t.} & \tabbb 0\le z_i  \le Z, \ \ \forall i \in p^-_x\nonumber\\
& \tabbb 0 \le y_i \le Y, \ \ \forall i \in p^-_q\nonumber\\
& \tabbb \sum_{i\in p^-_x} z_i \le e_x\nonumber\\
& \tabbb \sum_{i\in p^-_q} y_i \le e_q\nonumber \;.
\end{align}

The necessary and sufficient conditions on optimality are presented in appendix \ref{app:properties}. 

\medskip
\noindent \textbf{Optimal lower/upper bounds:} Let us denote the optimal value of (\ref{opt3}) by $v_{\opt} \le 0$. Then the optimal lower bound (LB) and upper bound (UB) for the distance estimation problem under consideration are given by
\begin{eqnarray}\label{bounds}
\tab \LB &=&   \sqrt{\hat{D}+ 4v_{\opt}}\label{LB}\\
\tab \UB &=&   \sqrt{\hat{D} - 4v_{\opt}}\label{UB}\\
\hat{D} &:=& \norm{X}^2 + \norm{Q}^2 - 4\!\!\sum_{l\in P_0}\!\!\Re\{X_l Q_l\} \;.  \nonumber
\end{eqnarray}

\begin{remark}\label{rem:lb}
Interestingly, the widely used convex solver cvx~\cite{cvx} cannot directly address \eqref{opt3}--the issue is that it fails to recognize convexity of the objective functions. For a numerical solution, we consider solving a relaxed version of the minimization problem \eqref{opt_d}, where equality constraints are replaced by $\le$ inequalities. We note that this problem is not equivalent to \eqref{opt_d}, but still provides a viable lower bound. An upper bound can be obtained by (cf. \eqref{LB}, \eqref{UB}):
$$ \UB = \sqrt{2\hat{D} - \LB^2}.$$
We test the tightness of such approach in the experimental section \ref{sec:experiments}. 
\end{remark}



\section{Exact Solutions}

In this section, we study algorithms for obtaining exact solutions for the optimization problem (\ref{opt3}).
By \emph{exact}, we mean that the optimal value is obtained in a finite number of computations as opposed to when using a numerical scheme for convex optimization. In the latter case, an approximate solution is obtained by means of an iterative scheme which converges with finite precision.
Before addressing the general problem, we briefly recap a special case that was dealt with in~\cite{vlachosSDM09},
where the sequence $\Q$ was assumed to be uncompressed. In this case, an exact solution is provided via the \emph{waterfilling} algorithm,
which constitutes a key building block for obtaining exact solutions to the general problem later on.
We then proceed to study the properties of optimal solutions; our theoretical analysis gives rise to an \emph{exact} algorithm, cf. Sec. \ref{algo:exact}.

\subsection{Waterfilling Algorithm.}

The case that $\Q$ is uncompressed is a special instance of our problem with $p^-_q=\emptyset$, whence also $P_2=P_3=\emptyset$. The problem is strictly convex, and (\ref{O}) yields
%
\begin{equation}
z_i = \big(\frac{b_i}{\lambda + \alpha_i}\big)^2 \Leftrightarrow a_i = \frac{b_i}{\lambda + \alpha_i} 
\end{equation}
In such a case, the strict convexity guarantees the existence of a \emph{unique} solution satisfying the KKT conditions as given by the \emph{waterfilling} algorithm, cf. Fig. \ref{fig:waterfill}.
The algorithm progressively increases the unknown coefficients $a_i$ until saturation, i.e., until they reach $A$, in which case they are fixed. The set $C$ is the set of non saturated coefficients at the beginning of each iteration, whereas $R$ denotes the ``energy reserve,'' i.e., the energy that can be used to increase the non saturated coefficients; $v_{opt}$ denotes the optimal value.

\begin{figure}[th]
\fbox{
\begin{minipage}[l]{8cm}
\noindent \textbf{Waterfilling algorithm}
\\
\noindent \textbf{Inputs:} $\{b_i\}_{i\in p^-_x}, e_x, A$
\\
\noindent \textbf{Outputs:} $\{a_i\}_{i\in p^-_x}, \lambda, \{\alpha_i\}_{i\in p^-_x}, v_{\opt}, R$

\begin{enumerate}
\item Set $R=e_x, \ C = p^-_x$

\item \textbf{while} $R>0$ and $C\ne \emptyset$ do

\item \tab set $\lambda = \sqrt{\frac{\sum_{i\in C}b_i^2}{R}}, \ \  a_i = \frac{b_i}{\lambda}, \ i\in C$

\item \tab \textbf{if} for some $i\in C, \ a_i > A$ then

\item \tabb $a_i=A, \ C \leftarrow C - \{i\}$

\item \tab \textbf{else} \textbf{break};

\item \tab \textbf{end if}

\item \tab $R = e_x - (|p^-_x| - |C|)A^2$

\item \textbf{end while}

\item Set $v_{\opt} = -\sum_{i\in p^-_x} a_ib_i$ and
\begin{equation*}
\alpha_i =
\left\{
  \begin{array}{ll}
    0, & \hbox{if } a_i < A\\
    \frac{b_i}{A} - \lambda, & \hbox{if } a_i = A
  \end{array}
\right.
\end{equation*}

\end{enumerate}
\end{minipage}
}
\caption{Waterfilling algorithm for optimal distance estimation between a compressed and an uncompressed sequence}
\label{fig:waterfill}
\end{figure}
As a shorthand notation, we write $\mathbf{a} = \text{waterfill}(\mathbf{b},e_x,A)$.
Note that in this case the problem (\ref{opt2}) for $P_2=P_3=\emptyset$ is convex, so the solution can be obtained via the KKT conditions to (\ref{opt2}), which are different from those for the re-parameterized problem (\ref{opt3}); this was done in~\cite{vlachosSDM09}. The analysis and straightforward extensions are summarized in Lemma \ref{exact_lemma}.

\begin{lemma}[Exact solutions]\label{exact_lemma}~\\
\begin{enumerate}
\item If either $p^-_x = \emptyset$ or $p^-_q = \emptyset$ (i.e., when at least one of the sequences is uncompressed) we can obtain an exact solution to the optimization problem (\ref{opt2}) via the waterfilling algorithm.

\item If $P_3 = p^-_x\cap p^-_q = \emptyset$, i.e., when the two compressed sequences do not have any common unknown coefficients, the problem is decoupled in $\mathbf{a},\mathbf{b}$, and the waterfilling algorithm can be used separately to obtain exact solutions to both unknown vectors.

\item If $P_1=P_2=\emptyset$, i.e., when both compressed sequences have the same discarded coefficients, the optimal value is simply equal to $-\sqrt{e_x}\sqrt{e_q}$, but there is no unique solution for $\mathbf{a},\mathbf{b}$.
\end{enumerate}
\end{lemma}
\begin{proof}
The first two cases are obvious. For the third one, note that it follows immediately from the Cauchy--Schwartz inequality that $-\sum_{l\in P_3}a_l b_l \ge -\sqrt{e_x}\sqrt{e_q}$, and in this is case this is also attainable. Just consider for example, $a_l = \sqrt{\frac{e_x}{|P_3|}}, b_l = \sqrt{\frac{e_q}{|P_3|}}$,
which is feasible because $|p^-_x|A^2 \ge e_x, |p^-_q|B^2 \ge e_q$, as follows by compression with the high-energy coefficients.  \hfill$\blacksquare$
\end{proof}

We have shown how to obtain exact optimal solutions for special cases.
To derive efficient algorithms for the general case, we first study and establish some properties of the optimal solution of (\ref{opt3}). 

\begin{theorem} [Properties of optimal solutions]\label{prop_thm}~\\
Let an augmented optimal solution of (\ref{opt2}) be denoted by ($\mathbf{a}^{\opt},\mathbf{b}^{\opt}$); where $\mathbf{a}^{\opt} := \{a_i^{\opt}\}_{i \in p^-_x \cup p^-_q}$ denotes the optimal solution extended to include the known values $|X_l|_{l\in P_2}$, and $\mathbf{b}^{\opt} := \{b_i^{\opt}\}_{i \in p^-_x \cup p^-_q}$ denotes the optimal solution extended to include the known values $|Q_l|_{l\in P_1}$. Let us further define $e_x' = e_x - \sum_{l\in P_1} a_l^2, e_q' = e_q - \sum_{l\in P_2}b_l^2$. 
We then have the following:
\begin{enumerate}
  \item The optimal solution satisfies\footnote{This has a natural interpretation as the \emph{Nash equilibrium} of a 2-player game~\cite{basar} in which Player 1 seeks to minimize the objective of (\ref{opt3}) with respect to $\mathbf{z}$, and Player 2 seeks to minimize the same objective with respect to $\mathbf{y}$.}
\begin{subequations}\label{nec_kkt_waterfill}
\begin{eqnarray}
\tab \mathbf{a}^{\opt} = \text{waterfill }(\mathbf{b}^{\opt},e_x,A),\\
\tab \mathbf{b}^{\opt} = \text{waterfill }(\mathbf{a}^{\opt},e_q,B).
\end{eqnarray}
\end{subequations}
In particular, it follows that $a_i^{\opt}>0$ iff $b_i^{\opt}>0$ and that $\{a_i^{\opt}\}, \{b_i^{\opt}\}$ have the same ordering. In addition, $\min_{l\in P_1}a_l \ge \max_{l\in P_3}a_l, \min_{l\in P_2}b_l \ge \max_{l\in P_3}b_l$.

\item If at optimality it holds that $e_x' e_q' >0$  
there exists a multitude of solutions. One solution $(\mathbf{a},\mathbf{b})$ satisfies $a_l = \sqrt{\frac{e_x'}{|P_3|}}, b_l = \sqrt{\frac{e_q'}{|P_3|}}$ for all $l\in P_3$, whence
\begin{subequations}
\begin{eqnarray}\label{lagrange_mult}
\lambda = \sqrt{\frac{e_q'}{e_x'}} \ \ \  \mu = \sqrt{\frac{e_x'}{e_q'}},\\
\alpha_i = \beta_i =0 \ \forall i \in P_3.
\end{eqnarray}
\end{subequations}
In particular, $\lambda \mu = 1$ and
the values $e_x',e_q'$ need to be solutions to the following set of nonlinear equations:
\begin{subequations}\label{non_lin}
\begin{eqnarray}
\tab \sum_{l\in P_1} \min\Big(b_l^2\frac{e_x'}{e_q'},A^2\Big) = e_x - e_x',\\
\tab \sum_{l\in P_2} \min\Big(a_l^2\frac{e_q'}{e_x'},B^2\Big) = e_q - e_q'.
\end{eqnarray}
\end{subequations}

\item At optimality, it is not possible to have $e_x'e_q' = 0$ unless $e_x' = e_q' = 0$.

\item Consider the vectors $\mathbf{a}, \mathbf{b}$ with $a_l = |X_l|, \ l\in P_2, a_l = |X_l|, \ l\in P_1$ and
\begin{subequations}\label{double_waterfill}
\begin{eqnarray}
\tab\tab \{a_l\}_{l\in P_1} =  \text{waterfill }(\{b_l\}_{l\in P_1},e_x,A),\\
\tab\tab \{b_l\}_{l\in P_2} = \text{waterfill }(\{a_l\}_{l\in P_2},e_q,B).
\end{eqnarray}
\end{subequations}
If $e_x \le |P_1|A^2$ and $e_q \le |P_2|B^2$,
whence $e_x' = e_q' = 0$, 
then by defining $a_l=b_l=0$ for $l\in P_3$, we obtain a \emph{globally} optimal solution $(\mathbf{a},\mathbf{b})$. 
\end{enumerate}
\end{theorem}

\begin{proof}
See appendix \ref{app:proof}.
\end{proof}

\begin{remark} One may be tempted to think that an optimal solution can be derived by waterfilling for the coefficients of $\{a_l\}_{l \in P_1},\{b_l\}_{l \in P_2}$ separately, and then allocating the remaining energies $e_x',e_q'$ to the coefficients in $\{a_l,b_l\}_{l\in P_3}$ leveraging the Cauchy--Schwartz inequality, the value being $-\sqrt{e_x'}\sqrt{e_q'}$. However, the third and fourth parts of Theorem \ref{prop_thm} state that this is \emph{not} optimal unless $e_x'=e_q'=0$.
\end{remark}

We have shown that there are two possible cases for an optimal solution of (\ref{opt2}): either $e_x' = e_q' = 0$ or $e_x',e_q' >0$. The first case is easy to identify by checking whether
(\ref{double_waterfill}) yields $e_x' = e_q' = 0$. If this is not the case, we are in the latter case and need to find a solution to the set of nonlinear  equations (\ref{non_lin}).

Consider the mapping $T: \mathbb{R}_+^2 \to \mathbb{R}_+^2$ defined by
\begin{small}
\begin{equation}\label{T_map}
T((x_1,x_2)) := \bigg(e_x - \sum_{l \in P_1}\min\Big(b_l^2\frac{x_1}{x_2},A^2\Big), \ \ e_q-\sum_{l \in P_2}\min\Big(a_l^2\frac{x_2}{x_1},B^2\Big)\bigg)
\end{equation}
\end{small}

\noindent
The set of nonlinear equations of $(\ref{non_lin})$ corresponds to a positive \emph{fixed point} of $T$, i.e., $(e_x',e_q') = T(e_x',e_q'), e_x',e_q' >0$. 
Calculating a fixed point of $T$ may at first seem involved, but it turns out that this can be accomplished \emph{exactly} and with minimal overhead. 
The analysis can be found in appendix \ref{app:fixed_point}, where we prove that this problem is no different from the simplest numerical problem: finding the root of a scalar linear equation.

\begin{remark}
We note that coefficients $\{a_l\}_{P_2},\{b_l\}_{P_1}$ are already sorted because of the way we perform compression--by storing high-energy coefficients. It is plain to check that all other operations can be efficiently implemented, with the average complexity being \emph{linear} ($\Theta(N)$)--hence the term \emph{minimum-overhead} algorithm. 
\end{remark}

\begin{remark}[Extensions]
It is straightforward to see that our approach can be  applied without modification to the case that each point is compressed using a different number of coefficients.  Additionally, we note here two important extensions of our problem formulation and optimal algorithm. 
\begin{enumerate}
\item Consider the case that a data point is a  \emph{countably infinite} sequence. For example, we may express a continuous function via its Fourier series, or Chebyshev polynomials expansion. In that case, surprisingly enough, our algorithm can be applied without any alteration. This is because $P_0,P_1,P_2$ are finite sets as defined above, whereas $P_3$ is now infinite. The energy allocation in $P_3$ can be computed exactly by the same procedure (cf. appendix \ref{app:fixed_point}) and then in $P_3$ the Cauchy--Schwartz inequality is again applied\footnote{The proof of optimality in this case assumes a finite subset of $P_3$ and applies the same conditions of optimality that were leveraged before; in fact particular selection of the subset is not of importance, as long as its cardinality is large enough to accommodate the computed energy allocation $(e'_x,e'_q)$.}. 
\item Of particular interest is the case that data are compressed using an \emph{over-complete} basis (also known as \emph{frame} in the signal-processing literature~\cite{oppenheim1999discrete}); this approach has recently become popular in the compressed-sensing framework~\cite{donoho2006compressed}. Our method can be extended to handle this important general case, by storing the compression errors corresponding to each given basis, and calculating the lower/upper bounds in each one separately using our approach. We leave this direction for future research work.

\end{enumerate}
\end{remark}

\section{Algorithm for Optimal Distance Estimation}
In this section, we present an algorithm for obtaining the \emph{exact optimal} upper and lower bounds on the distance between the original sequences, when fully leveraging all information available given their compressed counterparts.
First, we present a simple numerical scheme using a convex solver such as cvx~\cite{cvx} and then
use our theoretical findings to derive an analytical algorithm which we call `double-waterfilling'.

\subsection{Convex Programming}\label{algo:numerical}
~\\
We let $M := N -|P_0|$, and consider the nontrivial case $M>0$.
Following the discussion in Sec. \ref{real-opt}, we set the $2M\times 1$ vector $\mathbf{v} = (\{a_l\}_{l\in P_1 \cup P_2 \cup P_3},\{b_l\}_{l\in P_1 \cup P_2 \cup P_3})$
and consider the following convex problem directly amenable to a numerical solution via a solver such as cvx:
\begin{eqnarray*}
\min & \sum_{\l \in P_1 \cup P_2 \cup P_3} (a_l - b_l)^2\\
\mbox{s.t. } \ \ \  & a_l \le A, \ \forall l \in p^-_x, \ \ b_l \le B, \  \forall l \in p^-_q\\
& \sum_{l\in p^-_x} a_l^2 \le e_x, \ \ \sum_{l\in p^-_q} b_l^2 \le e_q\\
& a_l = |X_l|, \ \forall l \in P_2, \ \  b_l = |Q_l|, \ \forall l \in    P_1
\end{eqnarray*}
The lower bound ($\LB$) can be obtained by adding $D' := \sum_{l\in P_0} |X_l - Q_l|^2$ to the optimal value of (\ref{opt_d}) and taking the square root; then the upper bound is given by $\UB = \sqrt{2D' -\LB^2}$, cf. (\ref{bounds}).

\begin{figure}[!ht]
\fbox{
\begin{minipage}[l]{8cm}
\noindent \textbf{Double-waterfilling algorithm}
\\
\noindent \textbf{Inputs:} $\{b_i\}_{i\in P_1}, \{a_i\}_{i\in P_2}, e_x, e_q, A, B$
\\
\noindent \textbf{Outputs:} $\{a_i,\alpha_i\}_{i\in p^-_x}, \{b_i,\beta_i\}_{i\in p^-_q}, \lambda, \mu, v_{\opt}$
\vspace{-0.2cm}
\begin{enumerate}

\item \textbf{if} $p^-_x \cap p^-_q = \emptyset$ 
\textbf{then} use waterfilling algorithm (see Lemma \ref{exact_lemma} parts 1,2); \textbf{return}; \textbf{endif}

\item \textbf{if} $p^-_x = p^-_q$ \textbf{then} set $a_l = \sqrt{\frac{e_x}{|P_3|}}, b_l = \sqrt{\frac{e_q}{|P_3|}}$, $\alpha_l=\beta_l =0$ for all $l\in p^-_x$, $v_{\opt}=-\sqrt{e_x}\sqrt{e_q}$; \textbf{return}; \textbf{endif}

\item  \textbf{if} $e_x \le |P_1|A^2$ \textbf{and} $e_q \le |P_2|B^2 $ \textbf{then}
\begin{eqnarray*}
\{a_l\}_{l\in P_1} = \text{waterfill }(\{b_l\}_{l\in P_1},e_x,A)\\
\{b_l\}_{l\in P_2} = \text{waterfill }(\{a_l\}_{l\in P_2},e_q,B)
\end{eqnarray*}
with optimal values $v^{(a)}_{\opt},v^{(b)}_{\opt}$, respectively.

\item \tab Set $a_l =  b_l = \alpha_l = \beta_l =0$ for all $l\in P_3$, $v_{\opt} = -v^{(a)}_{\opt} - v^{(b)}_{\opt}$; \textbf{return};




\item \textbf{endif}


\item Calculate the root $\bar{\gamma}$ as in Remark \ref{root_calc} (appendix 12.3) and define $e_x',e_q'$  as in (\ref{exeq_res}).

\item Set
\begin{eqnarray*}
\{a_l\}_{l\in P_1} = \text{waterfill }(\{b_l\}_{l\in P_1},e_x-e_x',A)\\
\{b_l\}_{l\in P_2} = \text{waterfill }(\{a_l\}_{l\in P_2},e_q-e_q',B)
\end{eqnarray*}
with optimal values $v^{(a)}_{opt},v^{(b)}_{opt}$, respectively.

\item Set $a_l = \sqrt{\frac{e_x'}{|P_3|}}, b_l = \sqrt{\frac{e_q'}{|P_3|}}, \alpha_l=\beta_l=0, l\in P_3$ and set $v_{\opt} = -v^{(a)}_{opt} - v^{(b)}_{opt} - \sqrt{e_x'}\sqrt{e_q'}$

\end{enumerate}
\end{minipage}
}
\caption{Double-waterfilling algorithm for optimal distance estimation between two compressed sequences.}
\label{algo:double_waterfill}
\end{figure}

\subsection{Double-waterfilling}\label{algo:exact}~\\
Leveraging our theoretical analysis, we derive a simple efficient algorithm to obtain an \emph{exact} solution to the problem of finding tight lower/upper bound on the distance of two compressed sequences; we call this the ``double-waterfilling algorithm.''  The idea is to obtain an exact solution of (\ref{opt2}) based on the results of Theorems \ref{prop_thm}, \ref{opt_lem}, and Remark \ref{root_calc}; then the lower/upper bounds are given by (\ref{LB}), (\ref{UB}). The algorithm is described in Fig. \ref{algo:double_waterfill}; its proof of optimality follows immediately from the preceding theoretical analysis.


\section{Mining in the compressed domain}

In the experimental section, we will demonstrate the performance of our methodology
when operating directly in the compressed domain for distance-based operations.
We will use two common search and mining tasks to showcase our methodology: 
$(i)$ the $k$-NN search and $(ii)$ the $k$-Means clustering. 
Performing such operations in the compressed domain may require modifications
in the original algorithms, because of the uncertainty introduced in the distance
estimation in the compressed domain. We discuss these modifications in the sections that follow. 
We also elaborate on previous state-of-art approaches.

\subsection{$k$-NN search in the compressed domain}
Finding the closest points to a given query point is an elementary subroutine to many problems in search, classification, and prediction. 
A brute-force approach via exhaustive search on the uncompressed data can incur a prohibitive cost \cite{achlioptas2001database}. 
Thus, the capability to work directly in a compressed domain provides a very practical advantage to any algorithm.

In this context, the $k$-NN problem \cite{cover1967nearest} in the compressed domain can be succinctly described as:
\vskip.1in
\noindent \textsc{$k$-NN Problem:} {\it Given a compressed query representation $\Y_\Q$ and $k \in \mathbb{Z}_{+}$, find the $k$ closest elements $\X^{(i)} \in \mathcal{DB}$ with respect to the $\ell_2$-norm through their compressed representations $\Y^{(i)}$.}
\vskip.1in

\medskip
Various diverse approaches exist to tackle this problem efficiently and robustly; cf., \cite{jones2011randomized}. 
Here, we compare our methodology with two algorithmic approaches of $k$-NN search in a low-dimensional space: $(i)$ Randomized projection-based $k$-NN and $(ii)$ PCA-based $k$-NN. We describe these approaches in more detail.

\begin{algorithm*}[!htpb]
\caption{RP-based approximate $k$-NN algorithm}\label{algo:rand}
\begin{small}
\begin{algorithmic}[1]
\Statex ~~~\textbf{Input:} $k \in \mathbb{Z}_{+}$, $\Q$, $\epsilon \in (0, 1)$ \hfill \Comment{$k$: \# of Nearest Neighbors, $\Q$: the transformed \emph{uncompressed} query}
\Statex ~~~~~~~~~~~~~~~$\mathcal{X}$ with the transformed elements $\X^{(i)} = \mathcal{F}(\x^{(i)})$.
\algrule
\State ~~~ Select appropriate $d = \mathcal{O}(\epsilon^{-2} \log (V))$
\State ~~~ Construct $\boldsymbol{\Phi} \in \mathbb{R}^{d \times N}$ where $\Phi_{ij}$ is i.i.d. $\sim \mathcal{N}(0, \frac{1}{\sqrt{d}})$ or $\sim \text{\texttt{Ber}}\lbrace \pm \frac{1}{\sqrt{d}} \rbrace$ or according to \eqref{eq:ach}.
\State ~~~ \textbf{for each} $\X^{(i)} \in \mathcal{X}$ \textbf{do} \hfill (\textsc{Pre-processing step})
\State ~~~~~~ Compute $\Y^{(i)} = \boldsymbol{\Phi} \X^{(i)} \in \mathbb{C}^d$
\State ~~~ \textbf{end for} 
\algrule
\State ~~~ Compute $\Y_\Q = \boldsymbol{\Phi} \Q \in \mathbb{C}^d$.
\State ~~~ \textbf{for each} $\Y^{(i)}$ \textbf{do} \hfill (\textsc{Real-time execution step})
\State ~~~~~~ Compute $\|\Y_\Q - \Y^{(i)}\|_2$
\State ~~~ \textbf{end for}
\State ~~~ Sort and keep the $k$-closest $\Y^{(i)}$'s, in the $\ell_2$-norm sense.
\end{algorithmic}
\end{small}
\end{algorithm*}

\vskip.1in
\noindent \textit{Approximate $k$-NN using Randomized Projections (RP):} One of the most established approaches for low-dimensional data processing is through the {\it Johnson Lindenstrauss (JL) Lemma}\footnote{While JL Lemma applies for any set of points $\lbrace \X^{(1)}, \dots, \X^{(V)} \rbrace$ in high dimensions, more can be achieved if sparse representations of $\X^{(i)}, \forall i$, are known to exist \emph{a priori}. Compressive sensing (CS) \cite{candes2006stable} \cite{donoho2006compressed} roughly states that a sparse signal, compared with its ambient dimension, can be \emph{perfectly reconstructed} from far fewer samples than dictated by the well-known Nyquist--Shannon theorem. To this extent, CS theory exploits the sparsity to extend the JL Lemma to more general signal classes, not restricted to a collection of points $\mathcal{X}$. As a by-product of this extension, the CS version of the JL Lemma constitutes the Restricted Isometry Property (RIP).}:
\begin{lemma}[JL Lemma]
Let $\mathcal{X} = \lbrace \X^{(1)}, \dots, \X^{(V)} \rbrace$ be any arbitrary collection of $V$ points in $N$ dimensions. For an isometry constant $\epsilon \in (0, 1)$, we can construct with high probability a linear mapping $\boldsymbol{\Phi}: \mathbb{C}^N \rightarrow \mathbb{C}^d$, where $d = \mathcal{O}(\epsilon^{-2}\log(V))$, such that 
$${\label{eq:JL}}
(1 - \epsilon) \leq \frac{\|\boldsymbol{\Phi}(\X^{(i)}) - \boldsymbol{\Phi}(\X^{(j)})\|_2^2}{\|\X^{(i)} - \X^{(j)}\|_2^2 } \leq (1 + \epsilon)
$$ for all $\X^{(i)}, \X^{(j)} \in \mathcal{X}$.
\end{lemma}
Therefore, instead of working in the ambient space of $N$ dimensions, we can construct a linear, nearly {\it isometric map} $\boldsymbol{\Phi}$ that projects the data onto a lower subspace, {\it approximately preserving their relative distances} in the Euclidean sense. 
Variants of this approach have also been proposed in \cite{kushilevitz2000efficient}.
As we are not interested in {\it recovering} the entries of $\X^{(i)}$ in the $d$-dimensional space, rather than just performing data manipulations in this domain, one can control the distortion $\epsilon$ so that tasks such as classification and clustering can be performed quite accurately with low computational cost. However, we underline that the JL guarantees are probabilistic and asymptotic, whereas the lower and upper bounds provided by our technique cannot be violated.

The topic of constructing matrices $\boldsymbol{\Phi}$ that satisfy the JL Lemma with nice properties (e.g., deterministic construction, low space-complexity for storage, cheap operations using $\boldsymbol{\Phi}$) is still an open question, although many approaches have been proposed.
Similar questions also appear under the tag of {\it locality-sensitive hashing}, where {\it sketching} matrices ``sense'' the signal under consideration. Fortunately, many works have proved the existence of random universal matrix ensembles that satisfy the JL Lemma with \emph{overwhelming probability}, thus ignoring the deterministic construction property.\footnote{Recent developments \cite{hegde2012convex} describe \emph{deterministic constructions} of $\boldsymbol{\Phi}$ in polynomial time, based on the fact that the data $\mathcal{X}$ is known \emph{a priori} and fixed. The authors in \cite{hegde2012convex} propose the \texttt{NuMax} algorithm, a SemiDefinite Programming (SDP) solver for convex nuclear norm minimization over $\ell_{\infty}$-norm and positive semidefinite constraints. However, \texttt{NuMax} has $\mathcal{O}(C + N^3 + N^2C^2)$ time-complexity per iteration and overall $\mathcal{O}(C^2)$ space-complexity, where $C := {V \choose 2}$; this renders \texttt{NuMax} prohibitive for real-time applications. Such an approach is not included in our experiments, but we mention it here for completeness.}
Representative examples include random Gaussian matrices \cite{dasgupta1999learning} and random binary (Bernoulli) matrices \cite{arriaga1999algorithmic}\cite{achlioptas2001database}. 
Ailon and Chazelle \cite{ailon2006approximate} propose a fast JL transform for the $k$-NN problem with faster execution time than its predecessors. 
Achlioptas \cite{achlioptas2001database} proposes a randomized construction for $\boldsymbol{\Phi}$, both simple and fast, which is suitable for standard SQL-based database environments; each entry $\Phi_{ij}$ independently takes one of the following values:
\begin{align}\label{eq:ach}
\Phi_{ij} =
\left\{
	\begin{array}{ll} \vspace{0.05in}
		1  & \mbox{with probability } \frac{1}{6}, \\ \vspace{0.05in}
		0  & \mbox{with probability } \frac{2}{3}, \\ 
		-1 & \mbox{with probability } \frac{1}{6}.
	\end{array}
\right.
\end{align}
This has the additional advantage of producing a \emph{sparse} transformation matrix $\mathbf{\Phi}$ which results in computation savings at the data-compression step. 

In our experiments, we shall refer to the Gaussian distributed projection matrix as \texttt{GRP}, the Bernoulli distributed projection matrix as \texttt{BRP} and Achlioptas' construction \texttt{ARP}. We omit sophisticated constructions of $\boldsymbol{\Phi}$ because of their increased implementation complexity. 

Using linear maps $\boldsymbol{\Phi} \in \mathbb{R}^{d \times N}$ or $\boldsymbol{\Phi} \in \mathbb{C}^{d \times N}$ satisfying the JL Lemma, one can compress $\X^{(i)}$ as $\Y^{(i)} = \boldsymbol{\Phi} \X^{(i)} \in \mathbb{C}^{d}$ and store only $\Y^{(i)}$ for further processing. 
Given a compressed representation of a query $\Q$, $\Y_\Q = \boldsymbol{\Phi} \Q$, one can compute the distances of each $\Y^{(i)}$ to $\Y_\Q$ and pick the $k$ nearest points in the Euclidean sense. 
We provide a pseudo-code description of the above in Algorithm \ref{algo:rand}.

Overall, Algorithm \ref{algo:rand} requires $\mathcal{O}(\epsilon^{-2} V N \log (V))$ time to \emph{preprocess} the entries of $\mathcal{DB}$ and $\mathcal{O}(\epsilon^{-2}(V + N)\log(V))$ double-sized space-complexity, in the Gaussian case.
In the other two cases, the space complexity can be further reduced thanks to the binary representation of $\boldsymbol{\Phi}$. Given a query $\Y_\Q$, Algorithm \ref{algo:rand} requires $\mathcal{O}(\max\lbrace V, N \rbrace \cdot \epsilon^{-2} \log(V))$ time-cost.

\begin{algorithm*}[!htpb]
\caption{Optimal bounds-based $k$-Means}\label{algo:opt_kmeans}
\begin{small}
\begin{algorithmic}[1]
\Statex ~~~\textbf{Input:} $k \in \mathbb{Z}_{+}$, $\mathcal{Y} = \lbrace \Y^{(1)}, \dots, \Y^{(V)} \rbrace$  \hfill \Comment{$k$: \# of clusters, $\Y^{(i)}$: Compressed data vectors.}
\algrule
\State ~~~ Select randomly $k$ centroids $C^{(t)}, t = 1, \dots, k, $ in the compresed domain. \hfill (\textsc{Initialization step})
\algrule
\State ~~~ \textbf{while} $\Y^{(i)}$ assignment to $C^{(t)}$ changes \textbf{do}
\State ~~~~~~ Compute optimal lower $\ell_b$ and upper $u_b$ bounds between each $\Y^{(i)}$ and $C^{(t)}$, $\forall i, t$.
\State ~~~~~~ For each pair $\Y^{(i)}, C^{(t)}$ calculate a distance metric $m_{i,t}$, based on $\ell_b$ and $u_b$.
\State ~~~~~~ Assign $\Y^{(i)}$'s to groups $G^{(t)}$ such that: $G^{(t)} = \left \{ \Y^{(i)} ~|~ m_{i, t} \le m_{i, q}, \forall q \neq t \right\}.$ \hfill (\textsc{Assignment step})
\State ~~~~~~ Update the centroids: $C^{(t)} = \frac{1}{|G^{(t)}|} \sum_{\Y^{(i)} \in G^{(t)}} \Y^{(i)}.$ \hfill (\textsc{Update step})
\State ~~~ \textbf{end while} 
\end{algorithmic}
\end{small}
\end{algorithm*}

\medskip
\noindent \textit{Approximate $k$-Nearest Neighbors using PCA:} Instead of projecting onto a randomly chosen low-dimensional subspace, one can use the {\it most informative} subspaces to construct a projection matrix, based on $\mathcal{X}$. PCA-based $k$-NN relies on this principle: let $\mathbf{X} := [\X^{(1)}~~\X^{(2)}~~\dots~~\X^{(V)}] \in \mathbb{C}^{N \times n}$ be the data matrix. Given $\mathbf{X}$, one can compute the Singular Value Decomposition (SVD) $\mathbf{X} = \U \boldsymbol{\Sigma} \mathbf{V}^T$ to identify the $d$ most important subspaces of $\mathbf{X}$, spanned by the $d$ dominant singular vectors in $\U$. In this case, $\boldsymbol{\Phi} := \U(1:d, :)$ works as a low-dimensional linear map, biased by the information contained in $\mathbf{X}$. 

The main shortcoming of PCA-based $k$-NN search is the computation of the SVD of $\mathbf{X}$; generally, such an operation has cubic complexity in the number of entries in $\mathbf{X}$. Moreover, PCA-based projection provides no guarantees on the order of distortion in the compressed domain: While in most cases $\boldsymbol{\Phi} := \U(1:d, :)$ outperforms RP-based approaches with JL guarantees, one can construct test cases where the pairwise point distances are heavily distorted such that points in $\mathcal{X}$ might be mapped to a single point \cite{achlioptas2001database} \cite{hegde2012convex}. Finally, note that computation of the SVD requires 
the presence of the entire dataset, whereas approaches such as ours operate on a per-object basis.

\medskip
\noindent \textit{Optimal bounds-based $k$-NN:} Our approach can easily be adapted to perform $k$-NN search operations
in the compressed domain. Similar to Algorithm \ref{algo:rand}, instead of computing random projection matrices, we keep the high-energy coefficients for each transformed signal representation $\X^{(i)}$ (in Fourier, Wavelet or other basis) and also record the total discarded energy per object. Following a similar approach to compress the input query $\Q$, say $\Y_\Q$, we perform the optimal bounds procedure to obtain upper ($u_b$) and lower bounds ($\ell_b$) of the distance in the original domain. Therefore, we do not have only one
distance, but can compute three \textit{distance proxies} based on the upper and lower bounds on the distance:

\begin{itemize}
\item [$(i)$] We use the lower bound $\ell_b$ as indicator of how close  the uncompressed $\x^{(i)}$ is to the uncompressed query $\q$. 
\item [$(ii)$] We use the upper bound $u_b$ as indicator of how close the uncompressed $\x^{(i)}$ is to the uncompressed query $\q$. 
\item [$(iii)$] We define the average metric $\frac{\ell_b + u_b}{2}$ as indicator of how close the uncompressed $\x^{(i)}$ is to the uncompressed query $\q$. 
\end{itemize} 

In the experimental section, we evaluate the performance of these three metrics, and show that the last metric based
on the average distance bound provides the most robust performance.

\subsection{$k$-Means clustering in the compressed domain}
Clustering is a rudimentary mining task for summarizing and visualizing large amounts of data.
Specifically, the $k$-clustering problem is defined as follows:
\vskip.1in
\noindent \textsc{$k$-Clustering Problem:} {\it Given a $\mathcal{DB}$ containing $V$ compressed representations of $\x^{(i)}, \forall i,$ and a target number of clusters $k$, group the compressed data into $k$ clusters in an accurate way through their compressed representations.} 

This is an \emph{assignment} problem and is in fact NP-hard~\cite{freris2012cluster}. 
Many approximations to this problem exist, one of the most widely-used algorithms
being the $k$-Means clustering algorith \cite{lloyd1982least}. 
Formally, $k$-Means clustering involves partitioning the $V$ vectors into $k$ clusters, i.e., into $k$ disjoint subsets $G^{(t)}$ ($1{\leq}t{\leq}k$) with $\cup_{t}G^{(t)}=V$, such that the sum of intraclass variances
\begin{equation}\label{eq:metric}
V := \sum_{t=1}^k\sum_{\x^{(i)}\in G^{(t)}}||\x^{(i)} - C^{(t)}||^2,
\end{equation}
is minimized, where $C^{(t)}$ is the \emph{centroid} of the $k$-th cluster.

\noindent There also exist other formalizations for data clustering, based on either hierarchical clustering (``top-down'' and ``bottom-up'' constructions, cf. \cite{TSS04}); flat or centroid-based clustering, or on spectral-based clustering \cite{Dhi01}. 
 In our subsequent discussions, we focus on the $k$-Means algorithm because of its widespread use and fast runtime.
 Note also that $k$-Means is easily amenable for use by our methodology owning to its computation of distances between
 objects and the derived centroids.


\medskip
Similar to the $k$-NN problem case, we consider low-dimensional embedding matrices based on both PCA and randomized constructions \cite{huber1985projection}\cite{dasgupta2000experiments} \cite{boutsidis2010random}. We note that \cite{boutsidis2010random} theoretically proves that a specific random matrix construction achieves a $(2 + \epsilon)$-optimal $k$-partition of the points in the compressed domain in $\mathcal{O}(V N \frac{k}{\epsilon^2 \log(N)})$ time. Based on simulated annealing clustering heuristics, \cite{cardoso2012iterative} proposes an iterative procedure where sequential $k$-means clustering is performed, with increasing projection dimensions $d$, for better clustering performance. We refer the reader to \cite{boutsidis2010random} for a recent discussion of the above approaches.

Similar, in spirit, to our approach is the work of \cite{freris2012cluster}. There, the authors propose 1-bit Minimum Mean Square Error (MMSE) quantizers per dimension and cluster, and provide guarantees for cluster preservation in the compressed domain.  

\medskip
\noindent \textbf{Optimal bounds-based $k$-Means:} To describe how to use our proposed bounds in a $k$-clustering task, let $G^{(t)}, t = 1, \dots, k, $ be the $k$ groups of a partition with centroids $C^{(t)}, t = 1, \dots, k$. We use a modification of Lloyd's algorithm~\cite{lloyd1982least} which consists of the following steps: 
\vskip.05in
\noindent \textit{Assignment step:} Let $C^{(t)}, t = 1, \dots, k,$ be the current centroids.\footnote{For centroid initialization, one can choose $C^{(t)}$ to be $(i)$ completely random points in the compressed domain; $(ii)$ set randomly to one of the compressed representations of $\X^{(i)} \in \mathcal{DB}$, or (iii) use the better performing k-Means++ initialization algorithm~\cite{kMeans++,freris2012cluster}.} For each compressed sequence $\Y^{(i)} \in \mathcal{DB}$, we compute the corresponding upper $u_b$ and lower $\ell_b$ bounds with respect to every centroid $ C^{(t)}$. We use a distance metric $m_{i,t}$, based on $u_b$ and $\ell_b$ to decide the assignment of each $\Y^{(i)}$ to one of the centroids, i.e., 
$$
G^{(t)} = \left \{ \Y^{(i)} ~|~ m_{i, t} \le m_{i, q}, \forall q \neq t \right\}.
$$ Here, we use $m_{i,t} = \frac{u_b + \ell_b}{2}$ where $u_b,~\ell_b$ denote the upper- and lower-bounds between the compressed sequence $\Y^{(i)}$ and the centroid $C^{(t)}$; other distance metrics can be used depending on the nature and the requirements of the problem at hand.

\medskip
\noindent \textit{Update step:} To update the centroids $C^{(t)}$, we use the average rule using the current distribution of points at each cluster, i.e., 
$${\label{eq:update_step}}
C^{(t)} = \frac{1}{|G^{(t)}|} \sum_{\Y^{(i)} \in G^{(t)}} \Y^{(i)}.
$$

\medskip
\noindent \textit{Distance to new centroids:} 
Recall that each of the compressed objects has information only about its high-energy coefficients.
This set of coefficients may be different across objects. So, during the above averaging operation for computing the centroids, we may end up with the new centroids having (potentially) all coefficient positions filled with some energy.
However, this does not pose a problem for the distance computation because the waterfilling algorithm
can compute distance estimates even between compressed sequences with different number of coefficients.
Therefore, we exploit all information available in the computed centroid, as this does not increase
the space complexity of our technique. Alternatively, one could keep only the top high-energy coefficients
for the new centroid. However, there is no need to discard this extra information.

\noindent
The above steps are summarized in Algorithm \ref{algo:opt_kmeans}.


\section{Experiments}\label{sec:experiments}
Here we conduct a range of experiments to showcase a)
the tightness of bounds that we calculate; b) the low runtime to compute those bounds and,
c) the comparative quality of various mining tasks when using the optimal distance estimates.

Our intention in the experimental section is not to focus on a specific application, rather to evaluate different methodologies under implementation
invariant settings. This makes the contribution of our work for compression and search more generic and fundamental.

\medskip \noindent 
\textbf{Datasets:} We use two datasets: $(i)$ a weblog time-series dataset and $(ii)$ an image dataset consisting of Very Large Scale Integration (VLSI) layouts.  

\begin{figure*}[!htpb]
\centering
\includegraphics[width=0.305\textwidth]{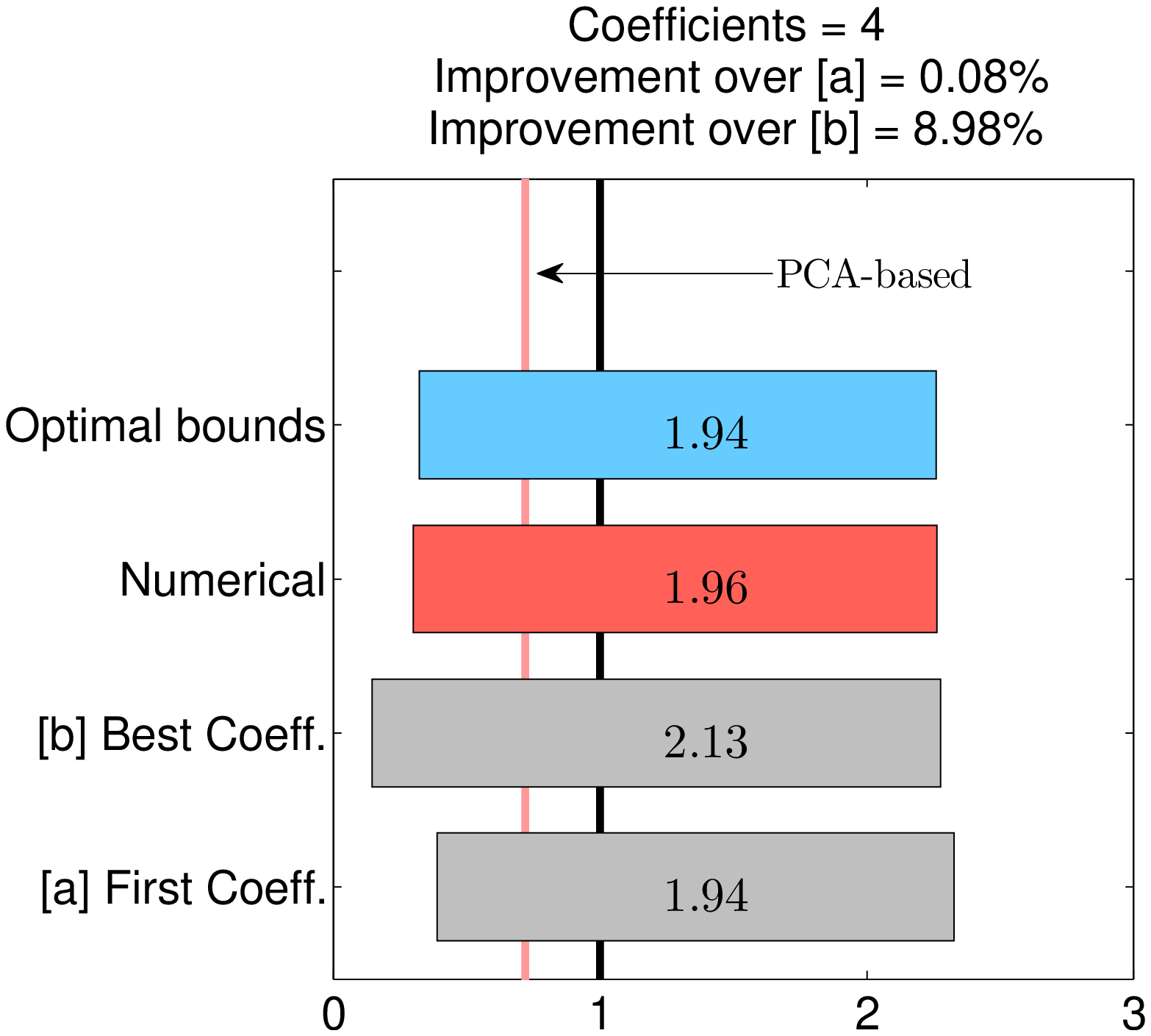} \includegraphics[width=0.22\textwidth]{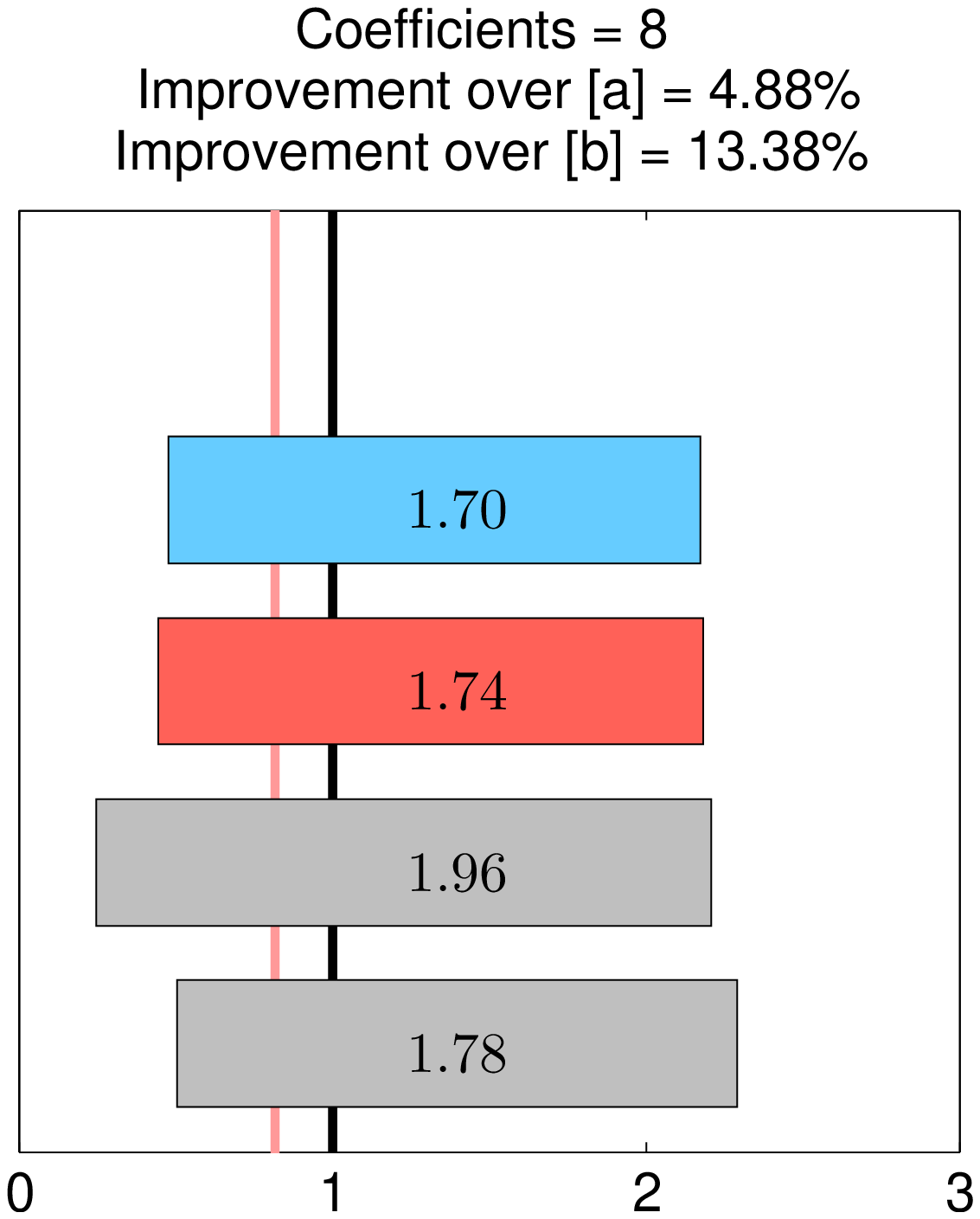} \includegraphics[width=0.22\textwidth]{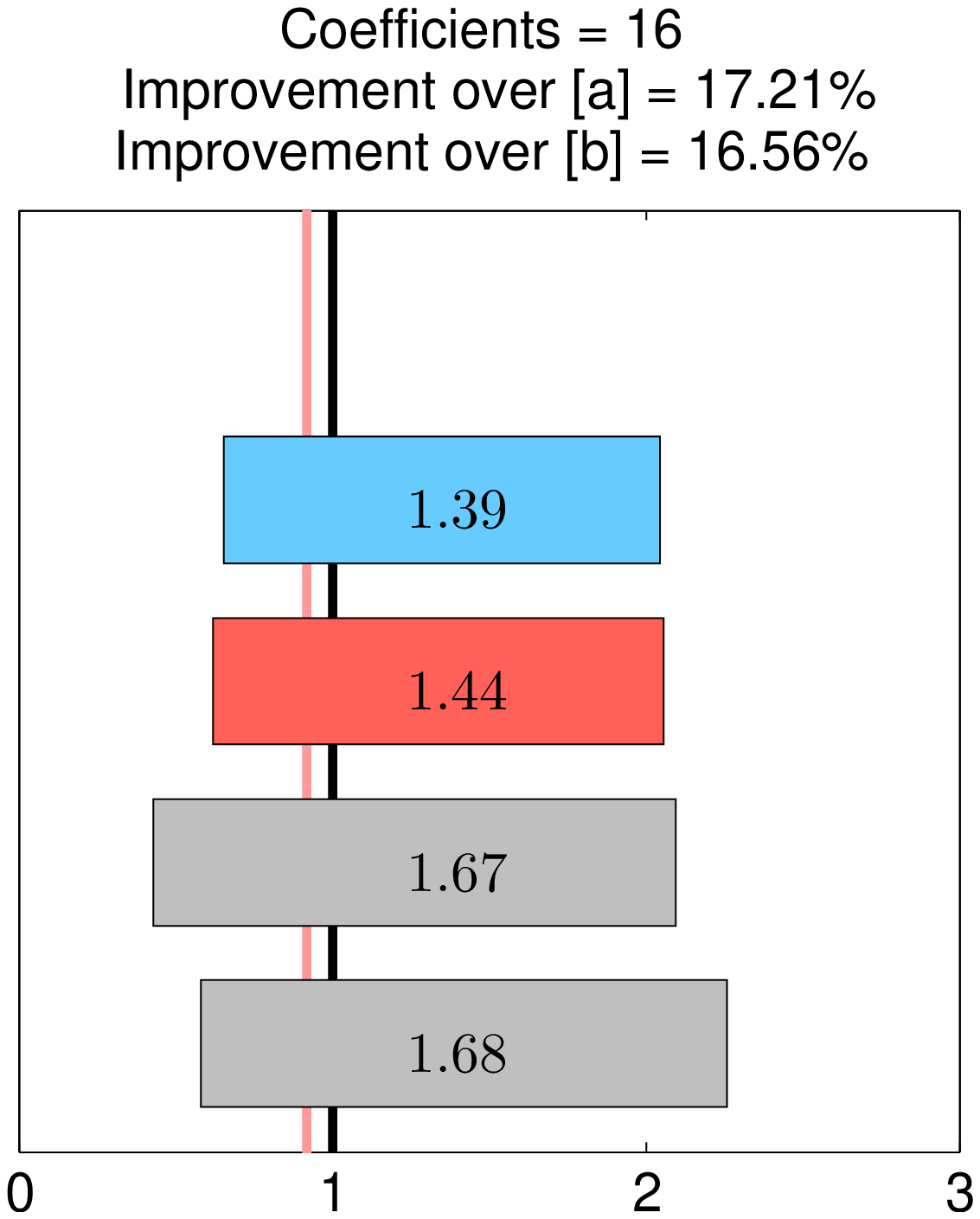} \includegraphics[width=0.22\textwidth]{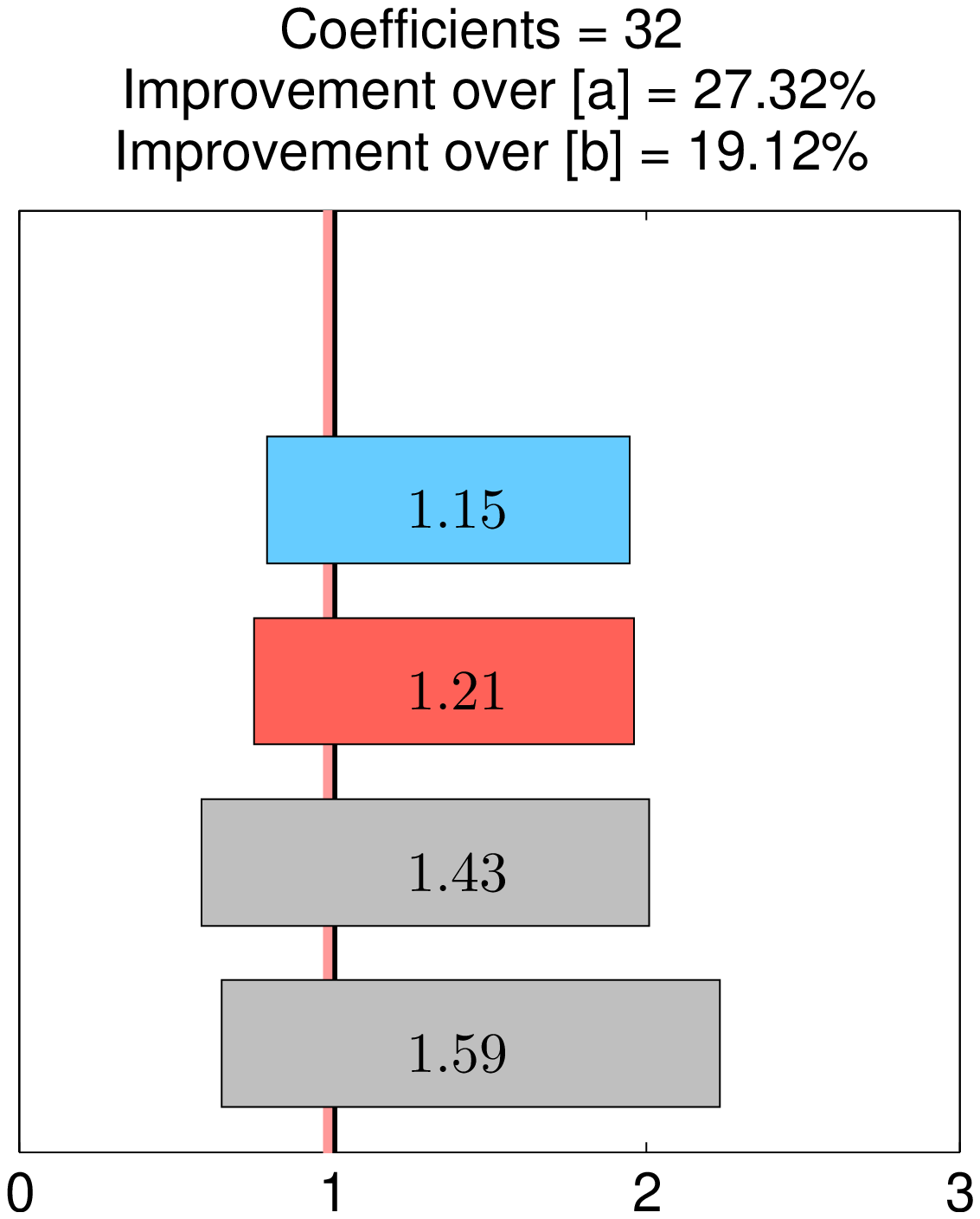} \\ \vspace{0.5cm}
\caption{Comparison of lower and upper bounds for distances for various compression approaches. Distances are shown normalized with respect to the original distance (black vertical line) on the uncompressed data. The red vertical
line indicates the bound given the PCA-based approach, which uses all the dataset to compute the appropriate basis.}
\label{fig:tight_sparse}
\end{figure*}

The weblog dataset consists of approximately 2.2 million data values, distilled from past logs of the IBM.com search page.
We created time-series by assembling how many times important keywords were entered at the search page. We constructed $2,150$ time-series
$\x^{(i)} \in \mathbb{R}^{1024}$, corresponding to an equal number of text queries. Each time-series object captures how many times a particular text query was posed for 1024 consecutive days. We transform the data into a compressible form: for each $\x^{(i)}$, we compute its Fourier representation $\X^{(i)} \in \mathbb{C}^{N}$. 

The second dataset considers images consisting of patterns from a database of VLSI layouts. The dataset is comprised of $\sim 150,000$ images of size 512x512. To support translation-invariant matching from each image, we extract a signature $\x^{(i)} \in \mathbb{R}^{512}$. We describe this process in more
detail later. Finally, we represent the resulting signature using the high-energy wavelet coefficients as basis. Detailed information on this application will be provided later on.

We take special care to perform a fair comparison of all techniques.
For our approach, per compressed object we need to store the following: $(i)$ the values $\X(p_x^{+})$ of each high-energy coefficient, as $s$ double complex values for the Fourier case (16 bytes each) and $s$ double values for the Wavelet case (8 bytes each); $(ii)$ The positions $p_{x}^{+}$ of the high-energy coefficients, as $s$ integer values (4 bytes each) and, $(iii)$ the total remaining energy $e_x$ of the discarded coefficients, which can be represented with one double variable (8 bytes). Overall, our approach allocates space for $r = \left \lceil 2s + \frac{s}{2} + 1\right \rceil$ double values for the Fourier case and $r = \left \lceil s + \frac{s}{4} + \frac{1}{2}\right \rceil$ double values for the Wavelet case. We make sure that all approaches use the \textit{same space}. So, methods that do not require recording of the explicit position of a coefficient, in essence, use more coefficients than our technique.


\subsection{Tightness of bounds and time complexity}

First, we illustrate how tight the bounds computed by both $(i)$ deterministic and $(ii)$ probabilistic approaches are. 
\vskip.1in
\noindent \textit{Deterministic approaches}: We consider the following schemes:
\begin{itemize}
\item [$(i)$] \texttt{First Coeffs.}: this scheme only exploits the \emph{first $s$} coefficients of each Fourier-transformed $\X^{(i)}$ to succinctly represent the uncompressed data $\x^{(i)}$. No further computation is performed. 
\item [$(ii)$] \texttt{Best Coeffs.}: this scheme only exploits the \emph{best $s$} coefficients (in magnitude sense) of each Fourier-transformed $\X^{(i)}$ to succinctly represent the uncompressed data $\x^{(i)}$. Similarly to $(i)$, no further computation is performed. 
\item [$(iii)$] \texttt{PCA-based}. This technique uses the PCA-based dimensionality reduction approach. Note that this approach requires as input the complete data, and not each object separately. Given the complete dataset, one can compute its SVD to extract the most dominant subspaces that explain the most variance in the data. To achieve dimensionality reduction, one projects the time-series vectors onto the best $d$-dimensional subspace by multiplying the data points with the set of $d$ dominant left singular vectors.
\item [$(iv)$] \texttt{Optimal bounds - Numerical}: Here, we use off-the-shelf convex solvers to numerically solve problem \eqref{opt_d} through second-order optimization schemes. Numerical approaches are not exact and the minimizer lies within a predefined numerical tolerance $\epsilon$. In our experimental setup, we use the well established \texttt{CVX} library where \eqref{opt_d} is solved with tolerance $\epsilon = 10^{-8}$. \vskip-.3in
\item [$(v)$] \texttt{Optimal bounds}: our approach in which the upper and lower bounds 
on the distance are solved using the closed-form waterfilling ideas described.
\end{itemize} To provide a fair comparison, all approaches use the same amount of space per compressed object for all experiments.
\begin{figure*}[!htpb]
\centering
\includegraphics[width=0.295\textwidth]{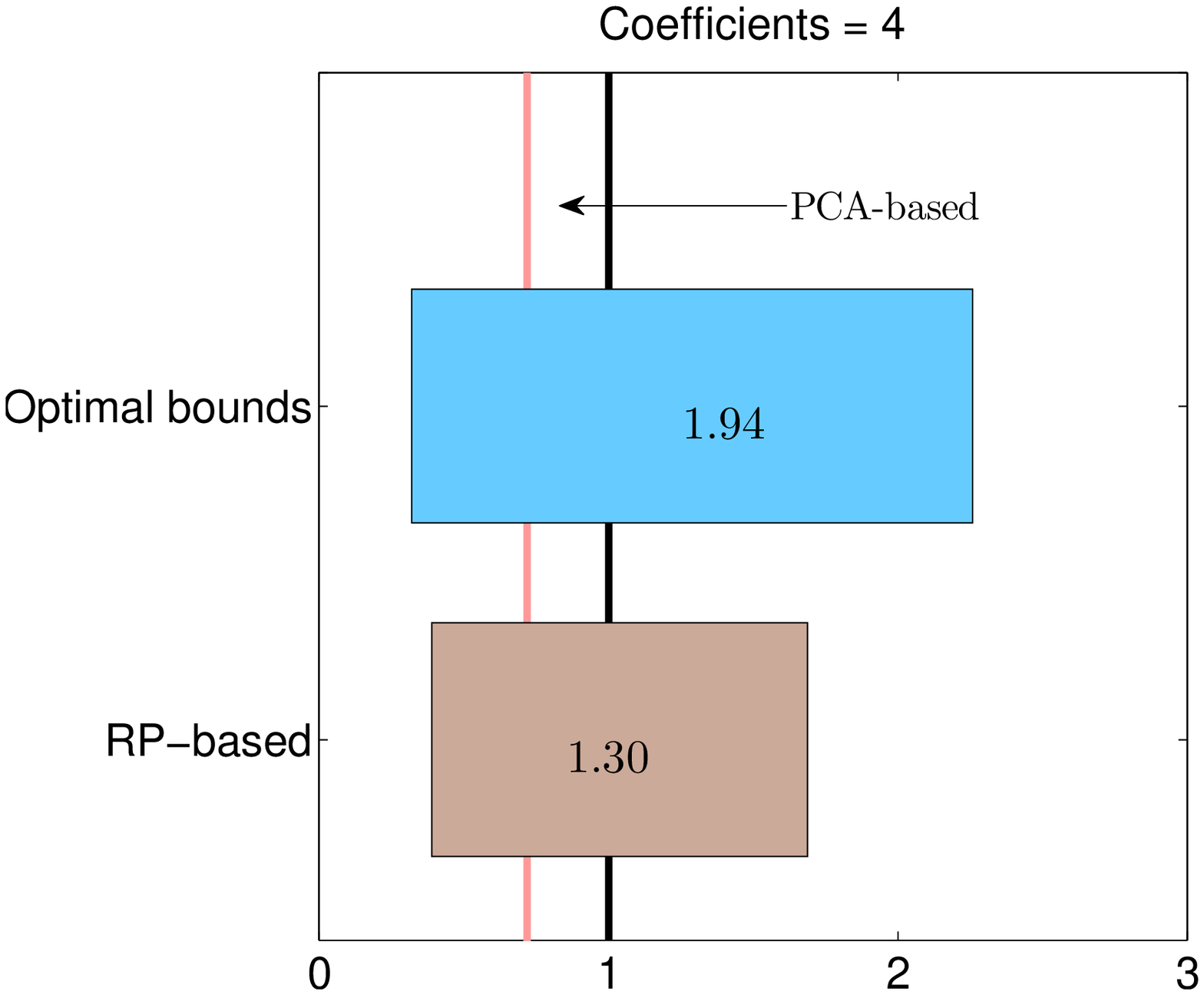} \includegraphics[width=0.22\textwidth]{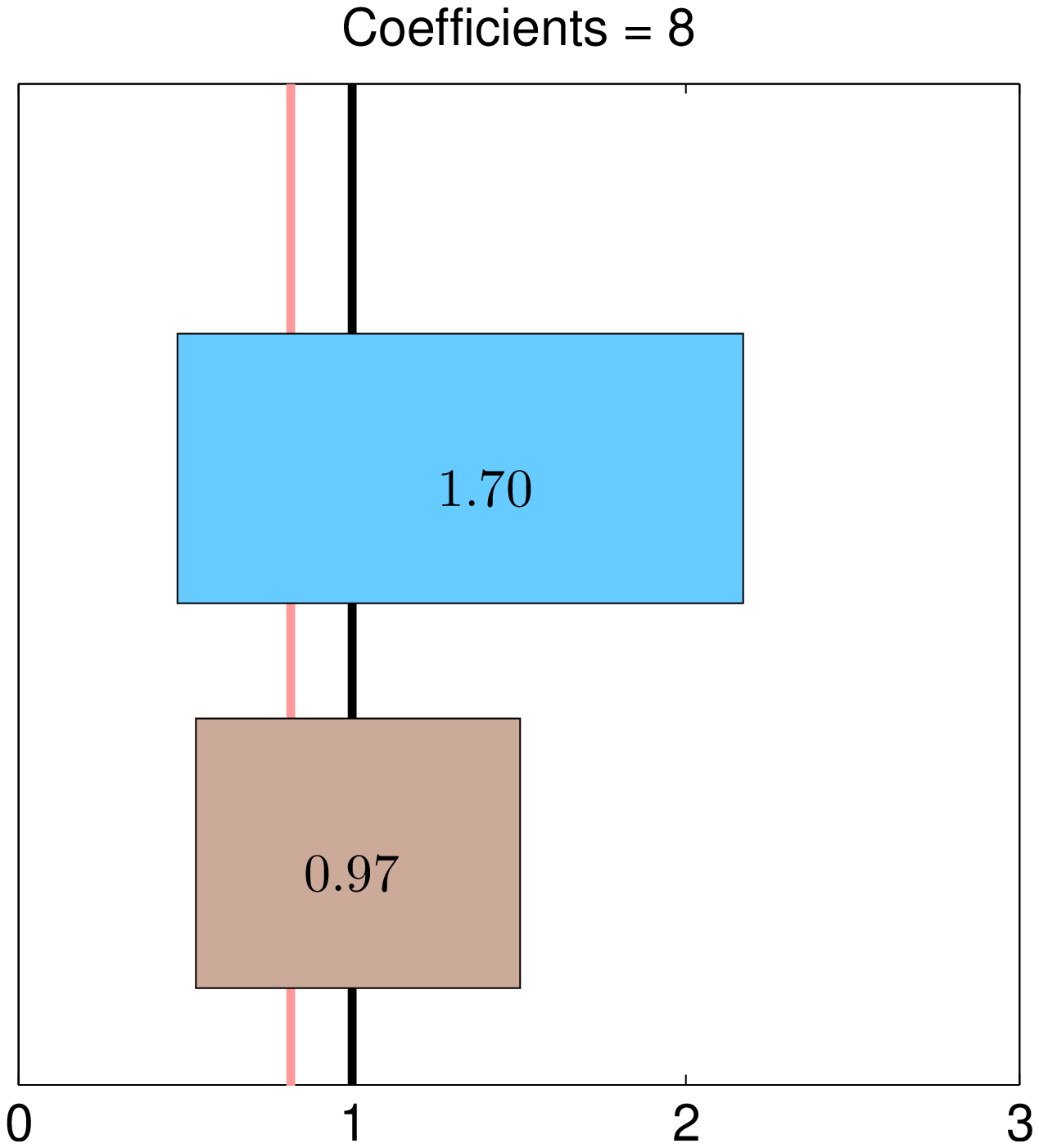} \includegraphics[width=0.22\textwidth]{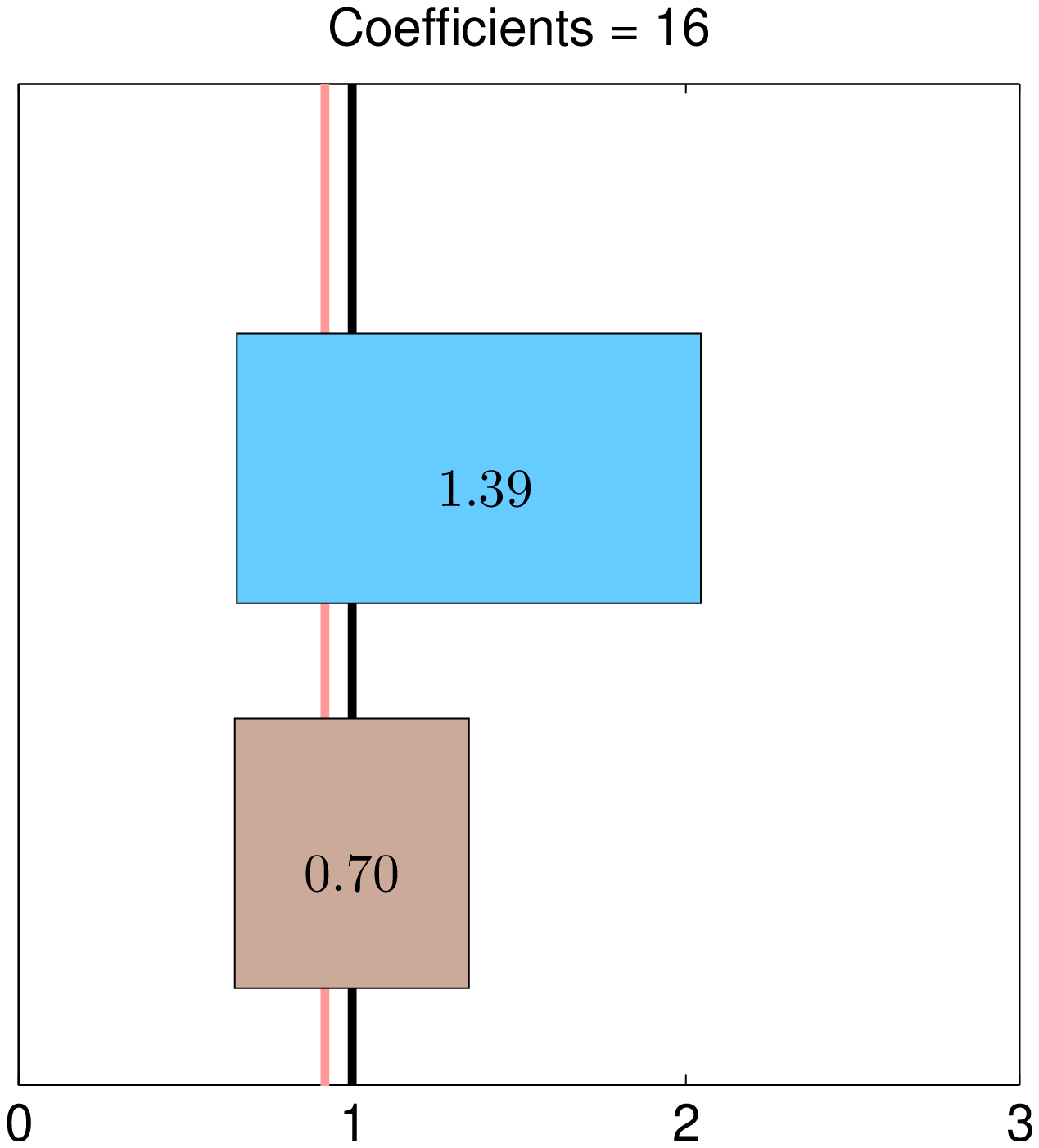} \includegraphics[width=0.22\textwidth]{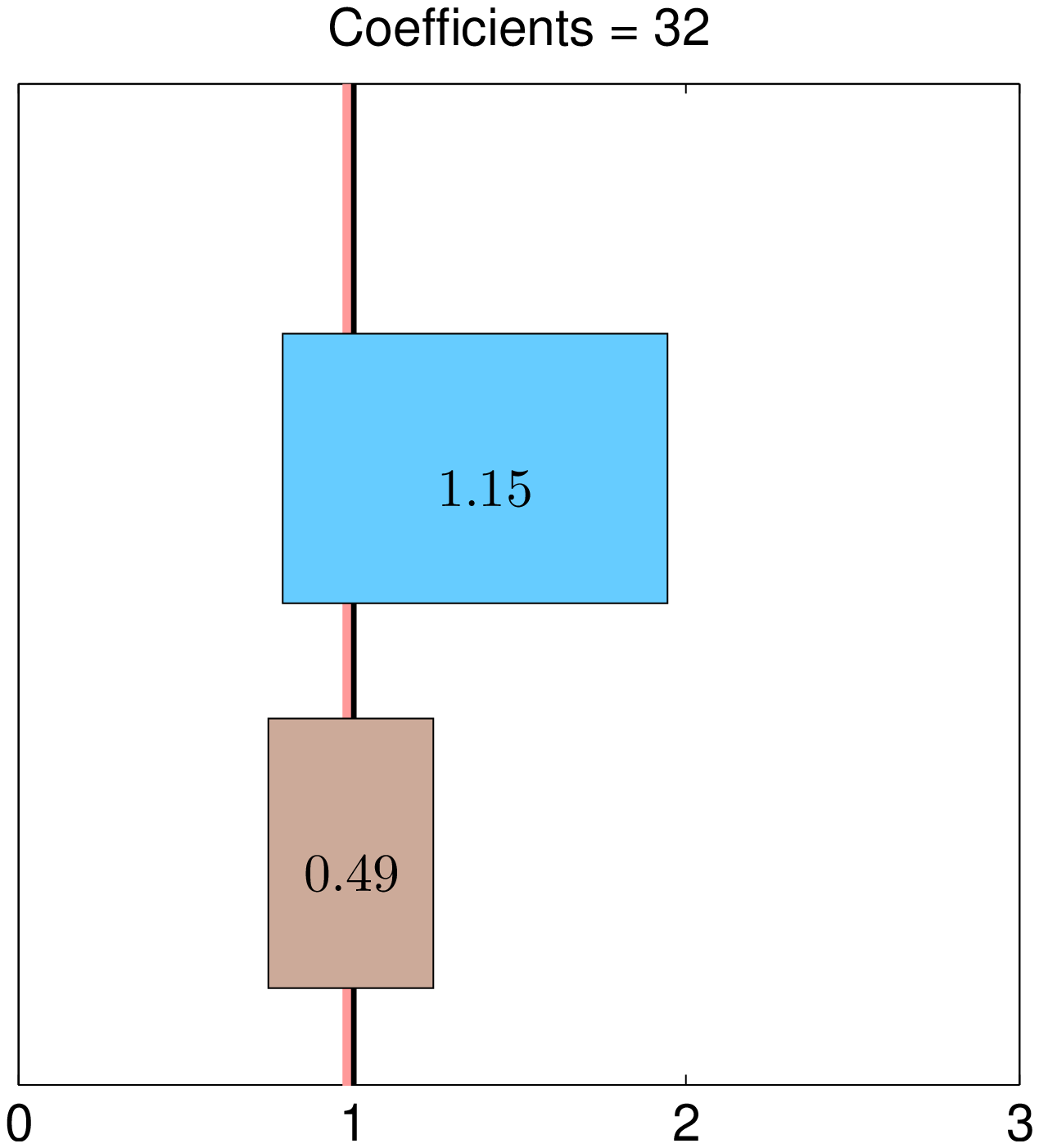} \\
\includegraphics[width=0.295\textwidth]{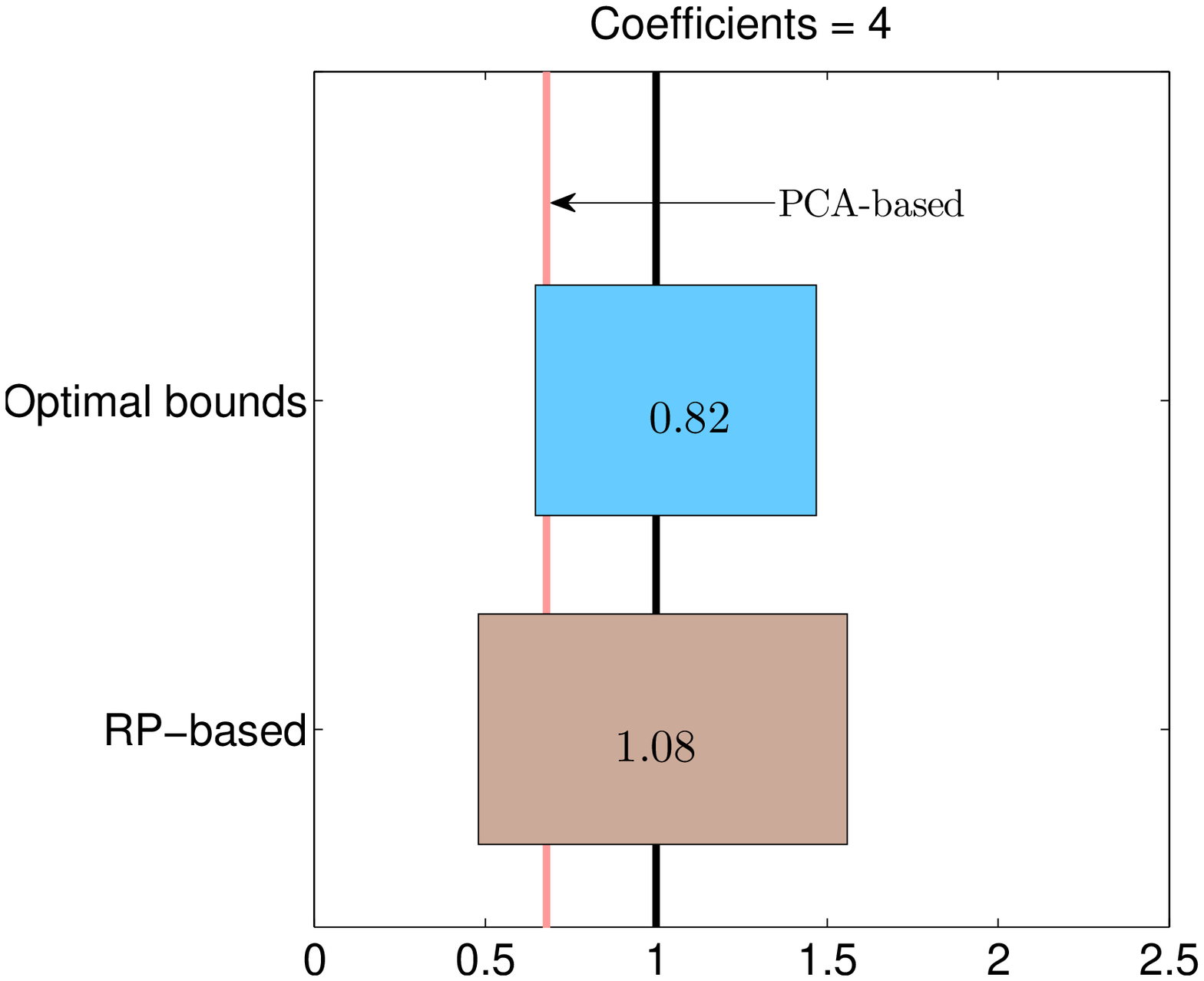} \includegraphics[width=0.222\textwidth]{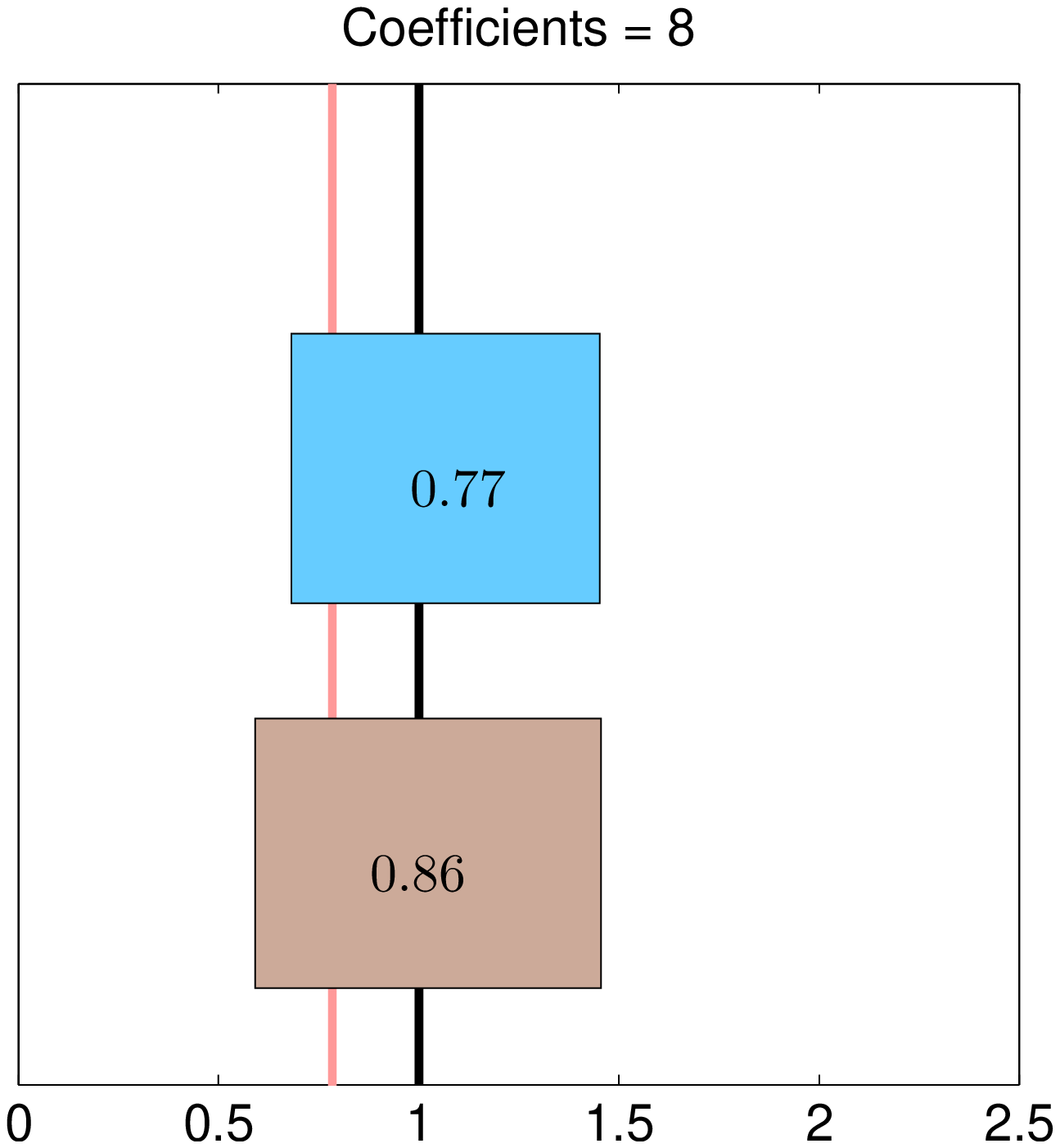} \includegraphics[width=0.222\textwidth]{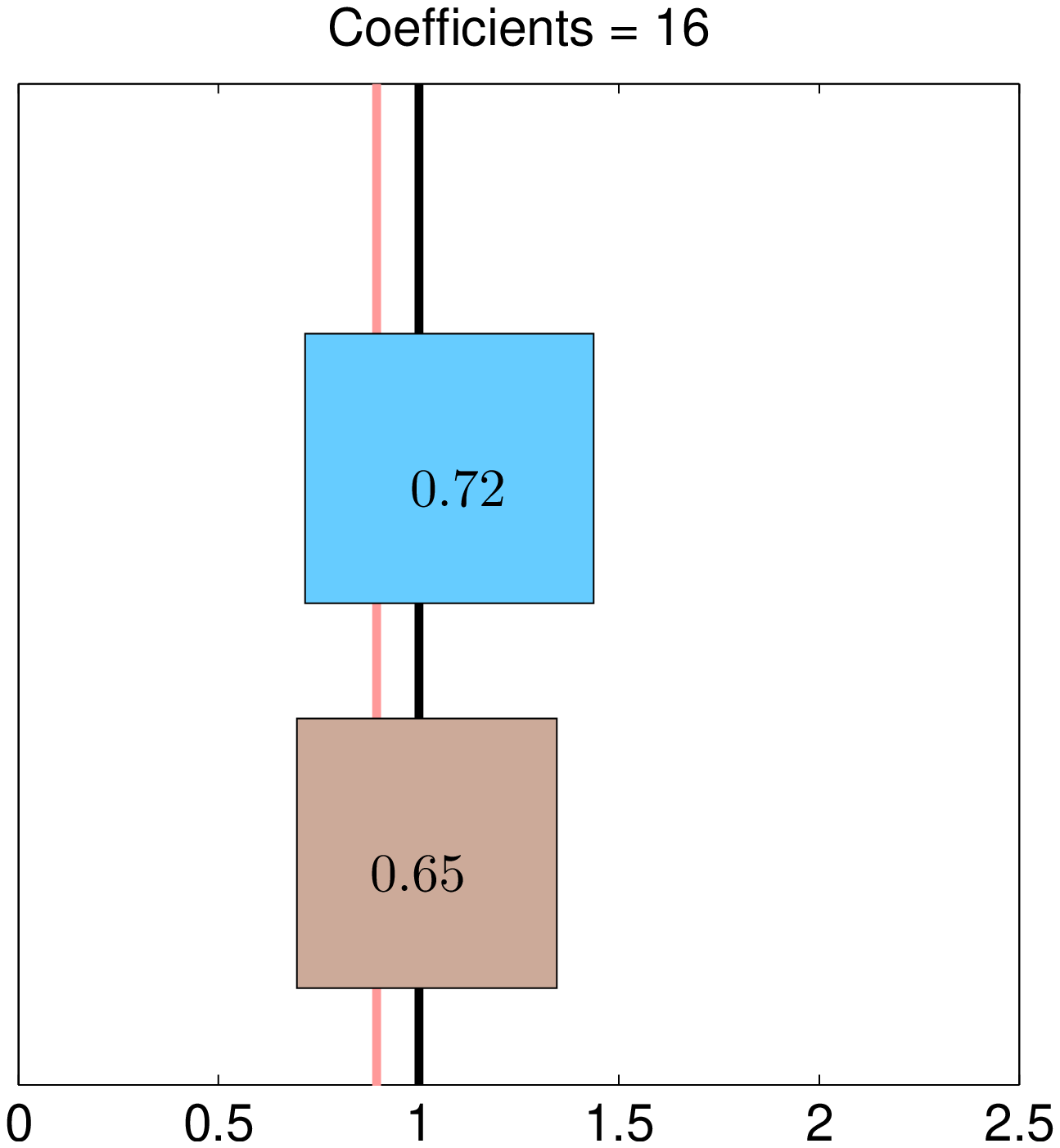} \includegraphics[width=0.222\textwidth]{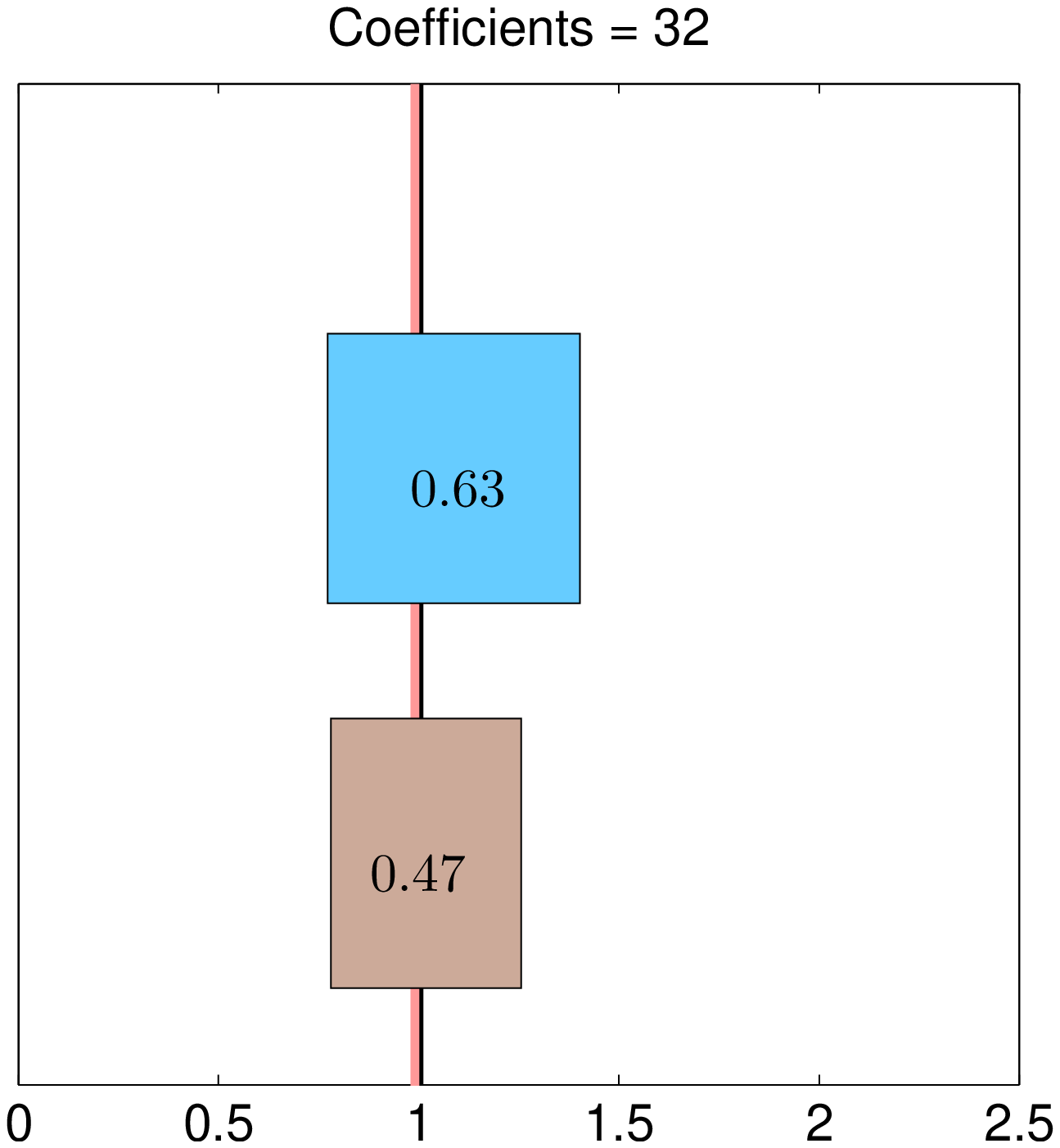}
\caption{Comparison of lower and upper bounds for the RP-based approach with the PCA-based and optimal bounds approaches. \textbf{Top row:} dense case; \textbf{Bottom row:} sparse case. All bounds are shown normalized with respect to the original distance (black vertical line) on the uncompressed data. The average PCA-based distance over all pairwise distances is denoted by a vertical red line.}
\label{fig:RPtight_sparse}
\end{figure*}

\vskip.1in
\noindent \textit{Probabilistic approaches}: Here, any estimation of distances in the original data space holds only in probability and in the asymptotic sense. The performance recorded represents only an average behavior; i.e., we can always construct adversarial cases where these schemes perform poorly for a fixed projection matrix.
\begin{itemize}
\item [$(i)$] \texttt{RP-based approach}: in this case, we randomly generate \texttt{GRP}/\texttt{BRP}/\texttt{ARP} $d$-dimensional mappings $\boldsymbol{\Phi}$. As this approach is probabilistic, we perform $10^3$ Monte Carlo iterations to independently generate $\boldsymbol{\Phi}$'s and implicitly extract an approximate range of lower and upper bounds. \vskip-.3in
\end{itemize}

\medskip \noindent
\textit{Quality of approximation:}
Using the weblog time-series data, Figure \ref{fig:tight_sparse} illustrates the results for various
compression ratios (i.e. number of coefficients  used).
Using the \texttt{Optimal bounds} scheme, we can achieve up to 27\% tighter bounds. 
The \texttt{Optimal bounds - Numerical} approach provides marginally worse results and has a very high computational cost.

The comparsion with probabilistic approaches based on Random Projections is shown in Figure \ref{fig:RPtight_sparse} (top row).
Among the \texttt{RP-based} approaches, the \texttt{GRP} matrices on average attain the 
best approximation of the original distance. We should highlight, though, that our approach computes nontrivial bounds \emph{for any pair of compressed sequences}, in contrast to RP-based schemes, where the guarantees hold in probability.

Furthermore, to measure the efficiency and robustness of our approach when the data is \emph{naturally} sparse in some basis, i.e., most of the entries in $N$-dimensions are zero, we synthetically ``sparsify'' the weblog data: given each uncompressed sequence $\X^{(i)}$, we assume that $\X^{(i)}$ is perfectly represented by using only $3s$ Fourier coefficients, where $s = \left\{16, 32, 64 \right\}$; i.e., we subsample the signal such that only $3s$ among $N$ coefficients are nonzero. 
Then, we keep $s$ coefficients. 
Figure \ref{fig:RPtight_sparse} (bottom row) illustrates the performance of our approach as compared to RP-based and PCA-based approaches. As the data under consideration are sparse, $e_x$ and $e_q$ estimates are tighter, providing better upper and lower bounds than in the non-sparse case (see Figure \ref{fig:RPtight_sparse}, top row).

\medskip
\noindent 
\textit{Running time:} The time complexity of each approach under comparison is given in Figure \ref{fig:time}. 
The graph reports the average running time (in msec) for computing the distance estimates between one pair of sequences.
It is evident that the proposed analytical solution based on double-waterfilling 
is at least \underline{two orders of magnitude faster} than the numerical approach. More importantly, the optimal solution through waterfilling is not computationally burdening:
competing approaches require 1-2 msec for computation, whereas the waterfilling approach requires around 2.5 msec. 
The small additional time is attributed to the fact that the algorithm distributes the currently
remaining energy over two to three iterations, thus incurring only minimal overhead. 
The numerical solution runs for more than 1 sec and is considered impractical for large mining tasks.
\begin{figure*}[!htpb]
\centering
\includegraphics[width=0.292\textwidth]{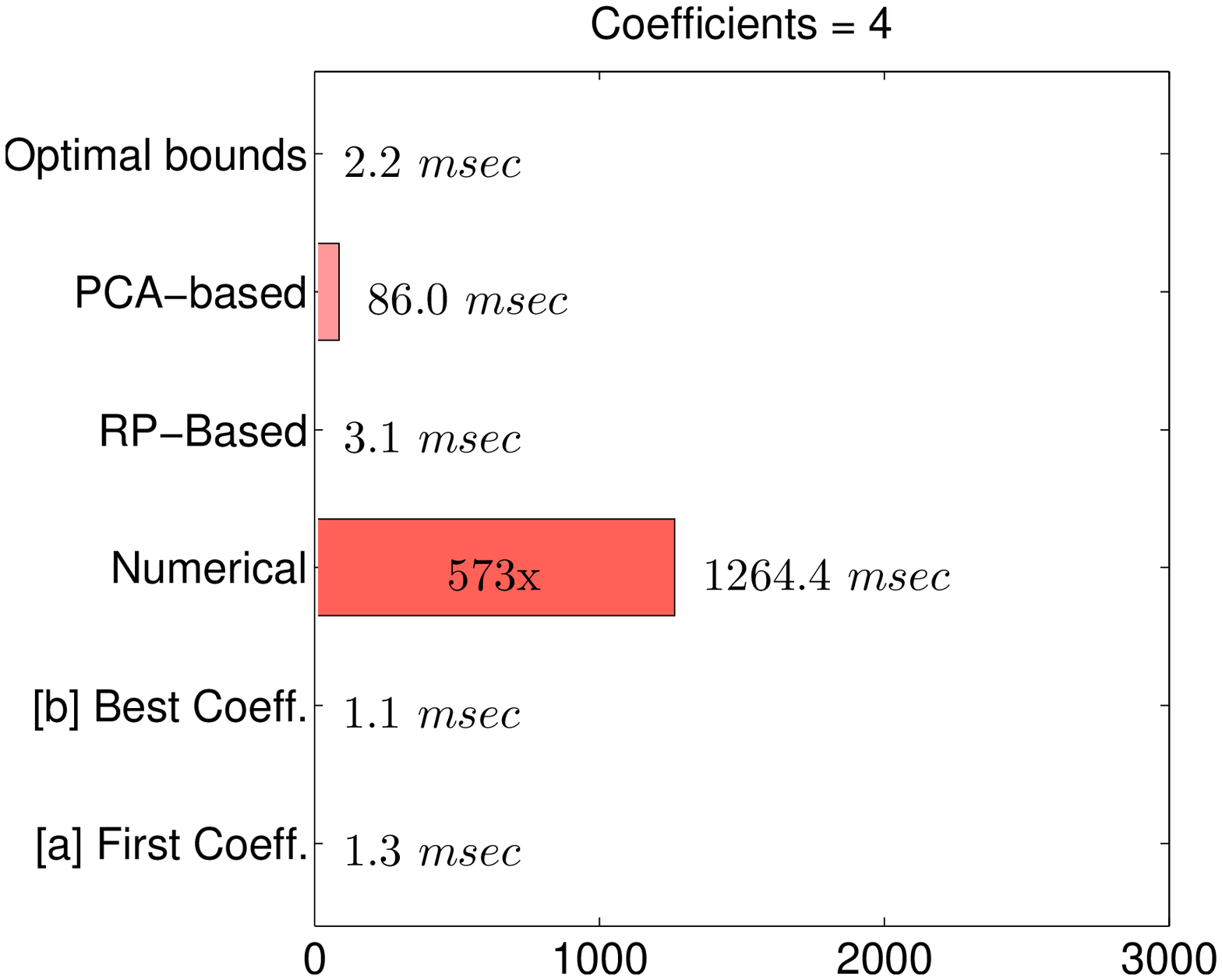} \includegraphics[width=0.222\textwidth]{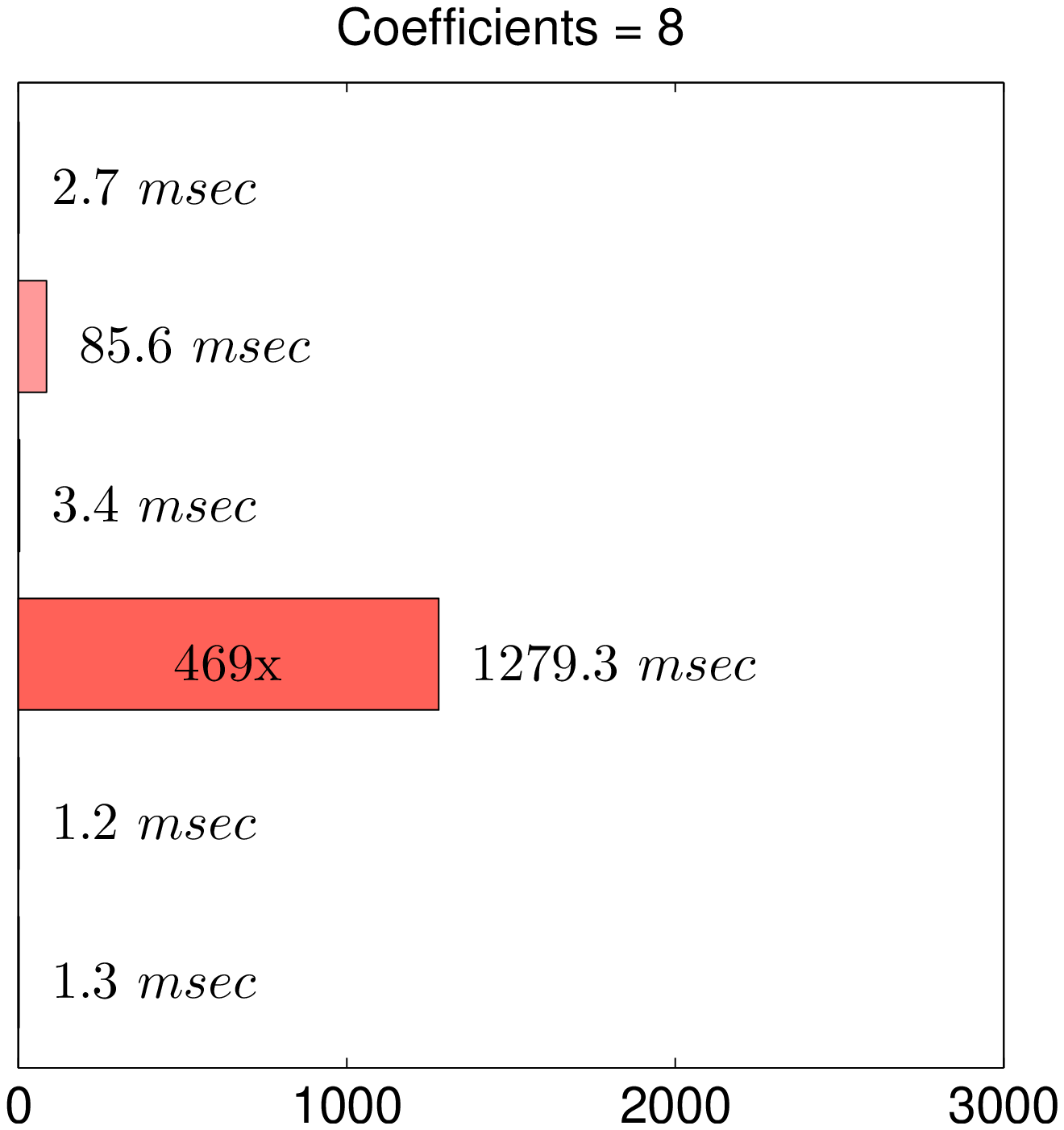} \includegraphics[width=0.222\textwidth]{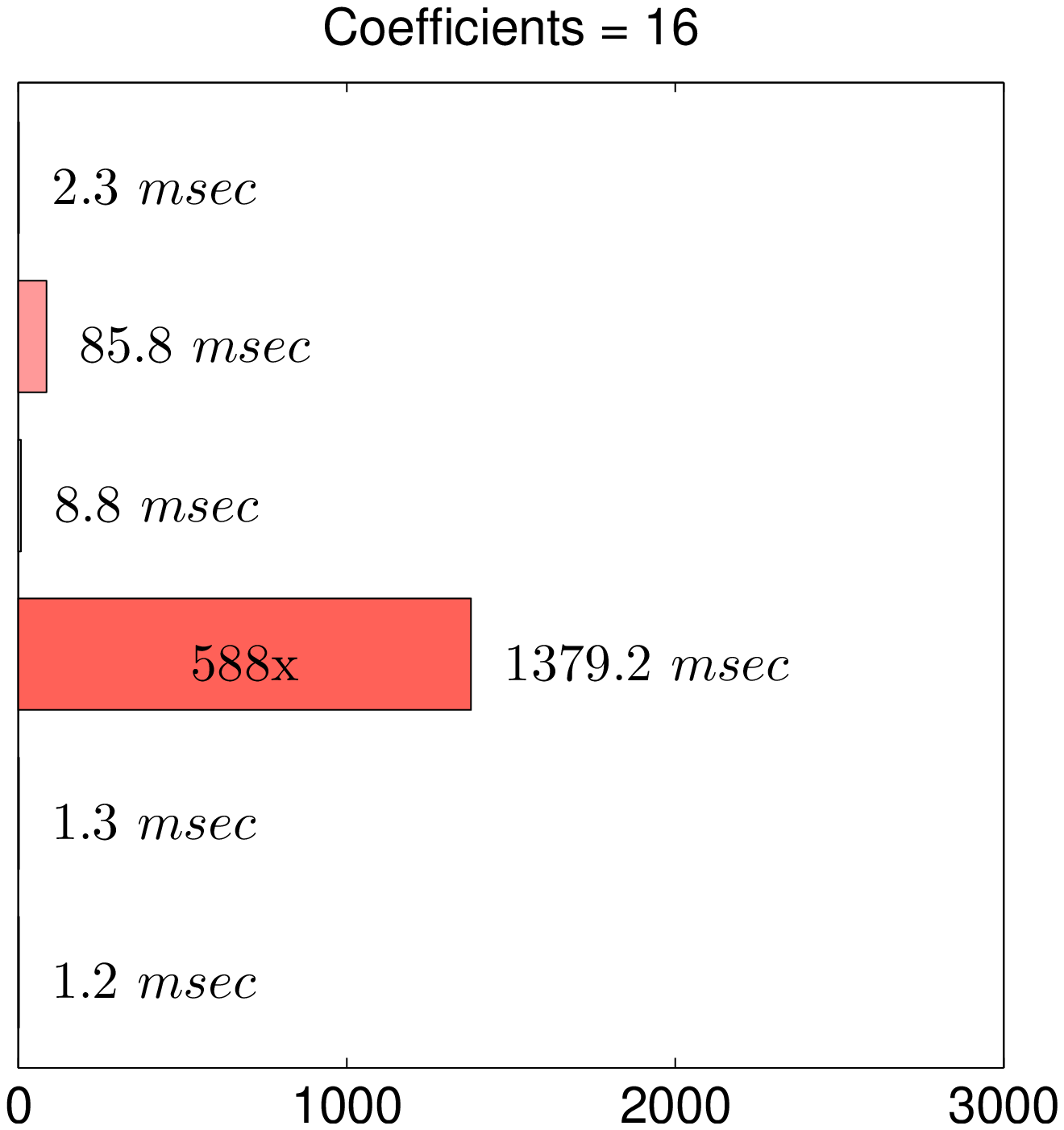} \includegraphics[width=0.222\textwidth]{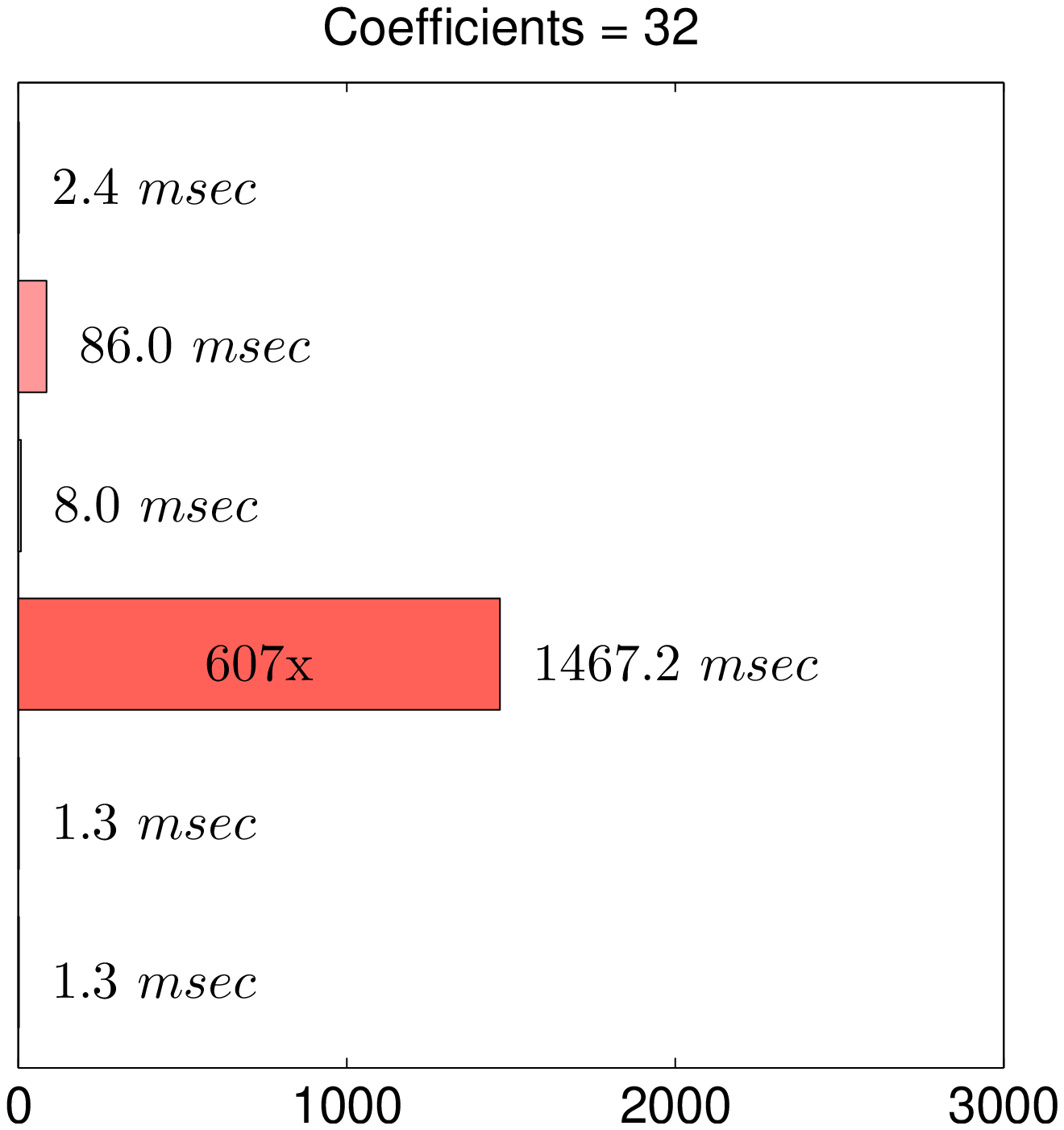}
\caption{Runtime of various techniques to compute a distance estimate for one pair of objects.}
\label{fig:time}
\end{figure*}

\subsection{Mining in the compressed domain}
We evaluate the quality of mining operations when operating 
directly on the compressed data. We compare with techniques
based on PCA and Random-Projections.

\medskip \noindent
\textit{Which distance proxy?:}
While techniques based on PCA and RP provide only a single distance estimate
in the compressed domain, our approach provides both a lower and an upper bound.
Earlier, we explained that one can use one of three potential proxies
for the distance: the upper bound $u_b$, the lower bound $\ell_b$, or the average of the two.
So, first, we evaluate which of the three metrics provides better distance estimation using
a $k$-NN task on the weblog dataset. Figure \ref{fig:metric} shows
how much large a percentage of the common $k$-NN objects is returned in the compressed domain, 
versus those we would have gotten on the uncompressed data.
The experiment is conducted under increasing number of coefficients. One can observe
that the average of the lower and upper bounds shows overall superior performance, 
and this is the distance proxy we use for the remaining of the experiments.
\begin{figure*}[!htp]
\centering
\includegraphics[width=0.25\textwidth]{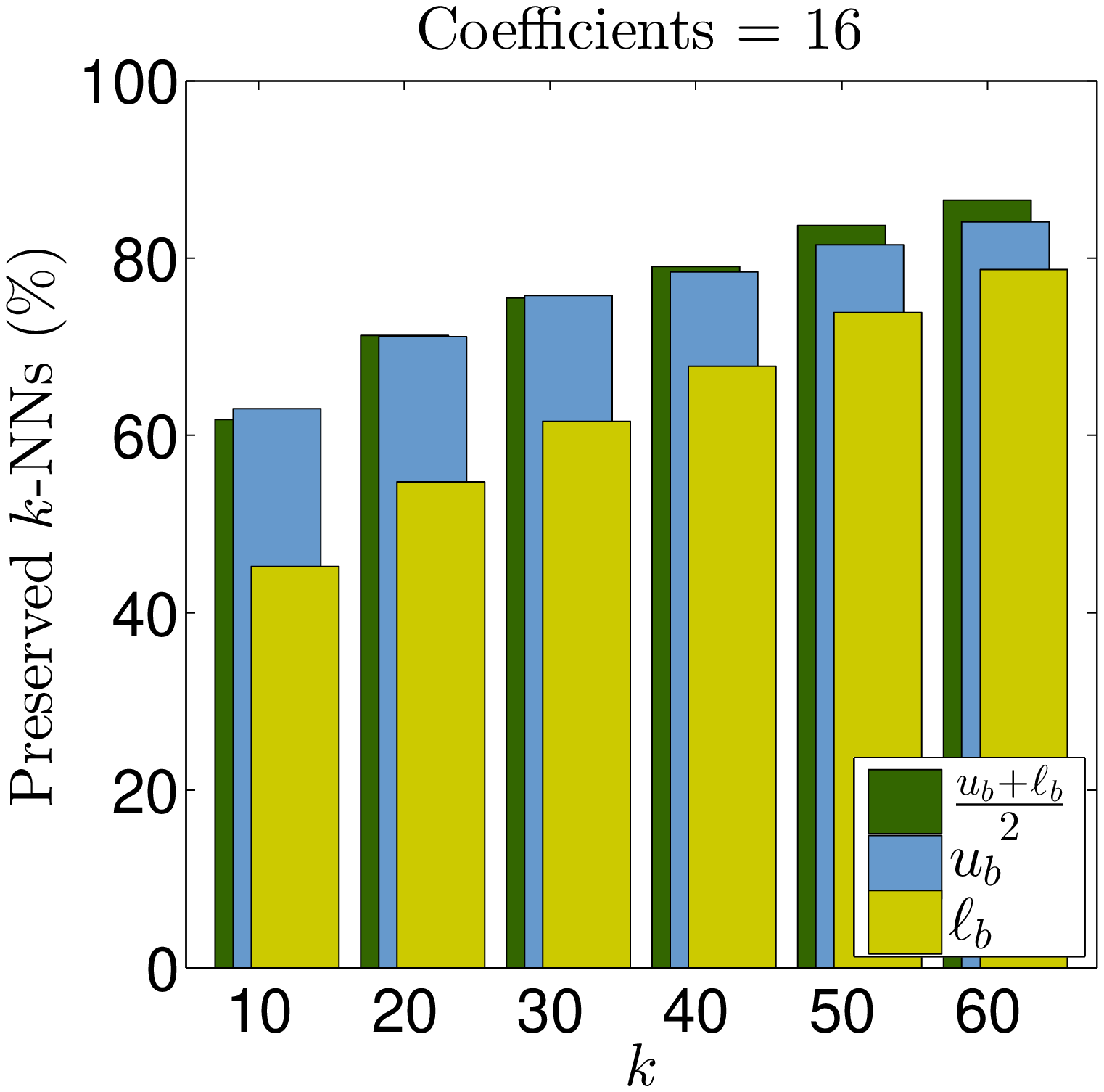} \includegraphics[width=0.25\textwidth]{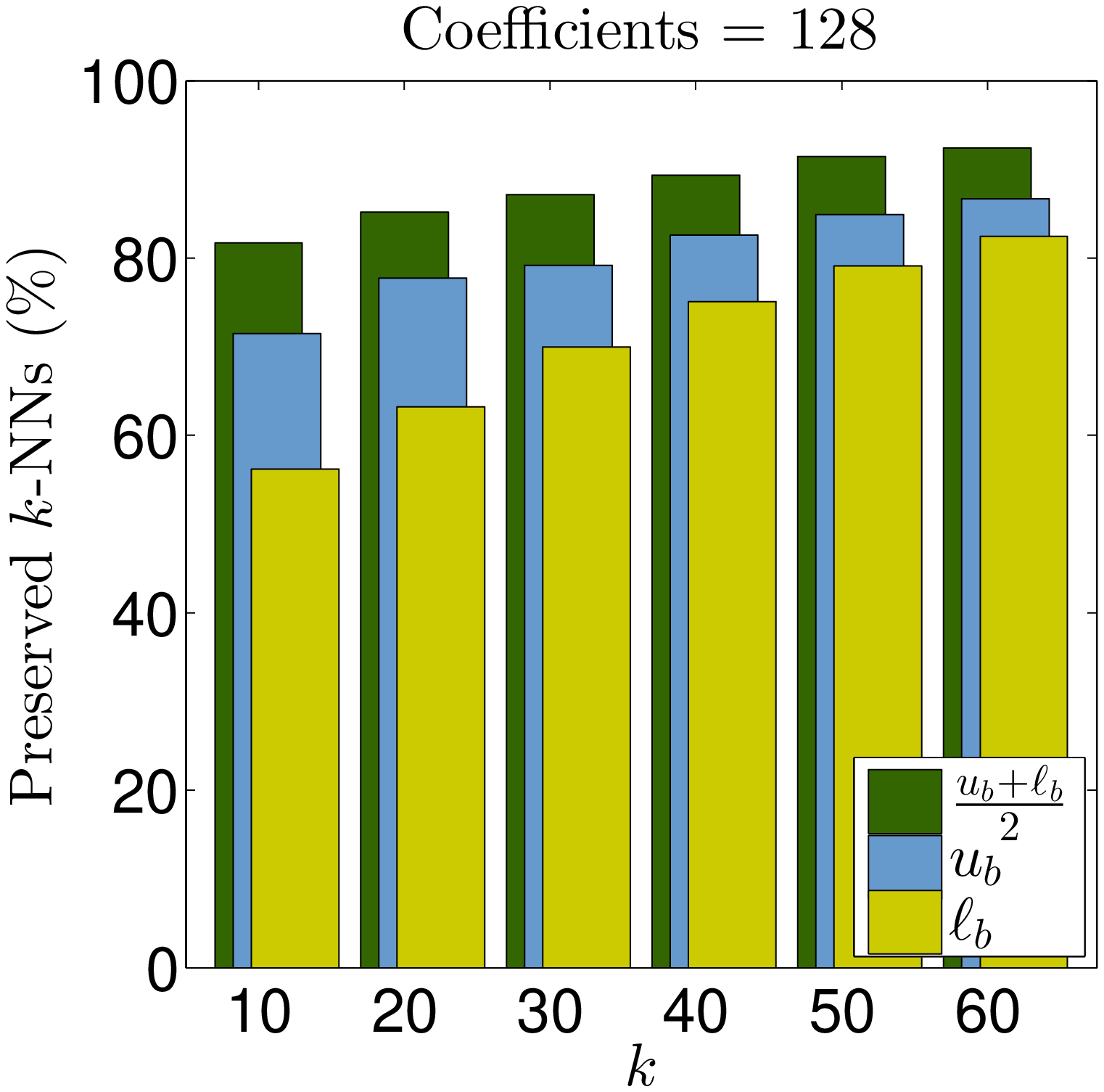} \includegraphics[width=0.25\textwidth]{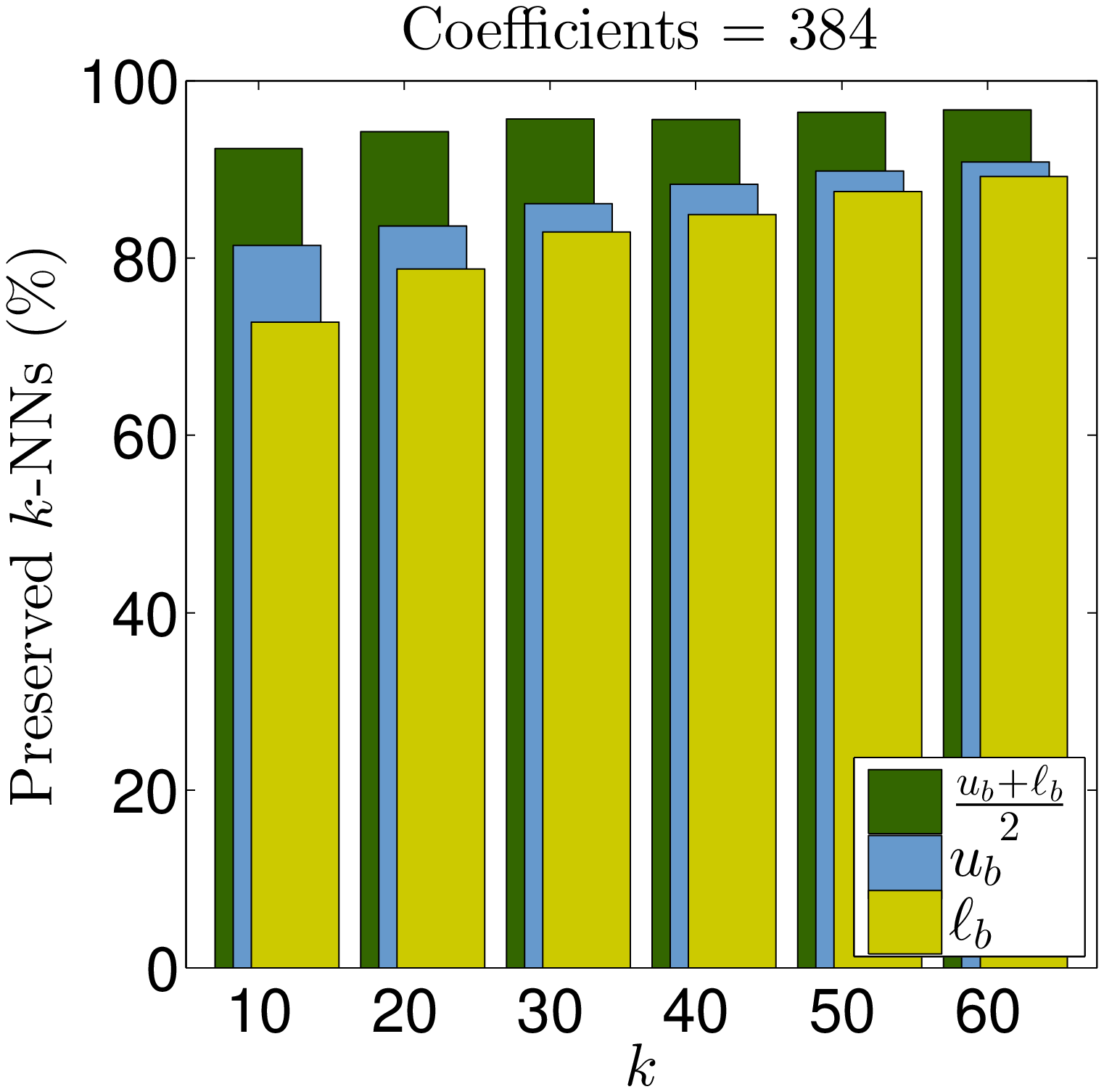}
\caption{Comparison of distance-estimation metrics for the $k$-NN task using the weblog data. The bars depict mean values of 100 Monte Carlo iterations.}
\label{fig:metric}
\end{figure*}

\begin{figure*}[!htp]
\centering
\includegraphics[width=0.325\textwidth]{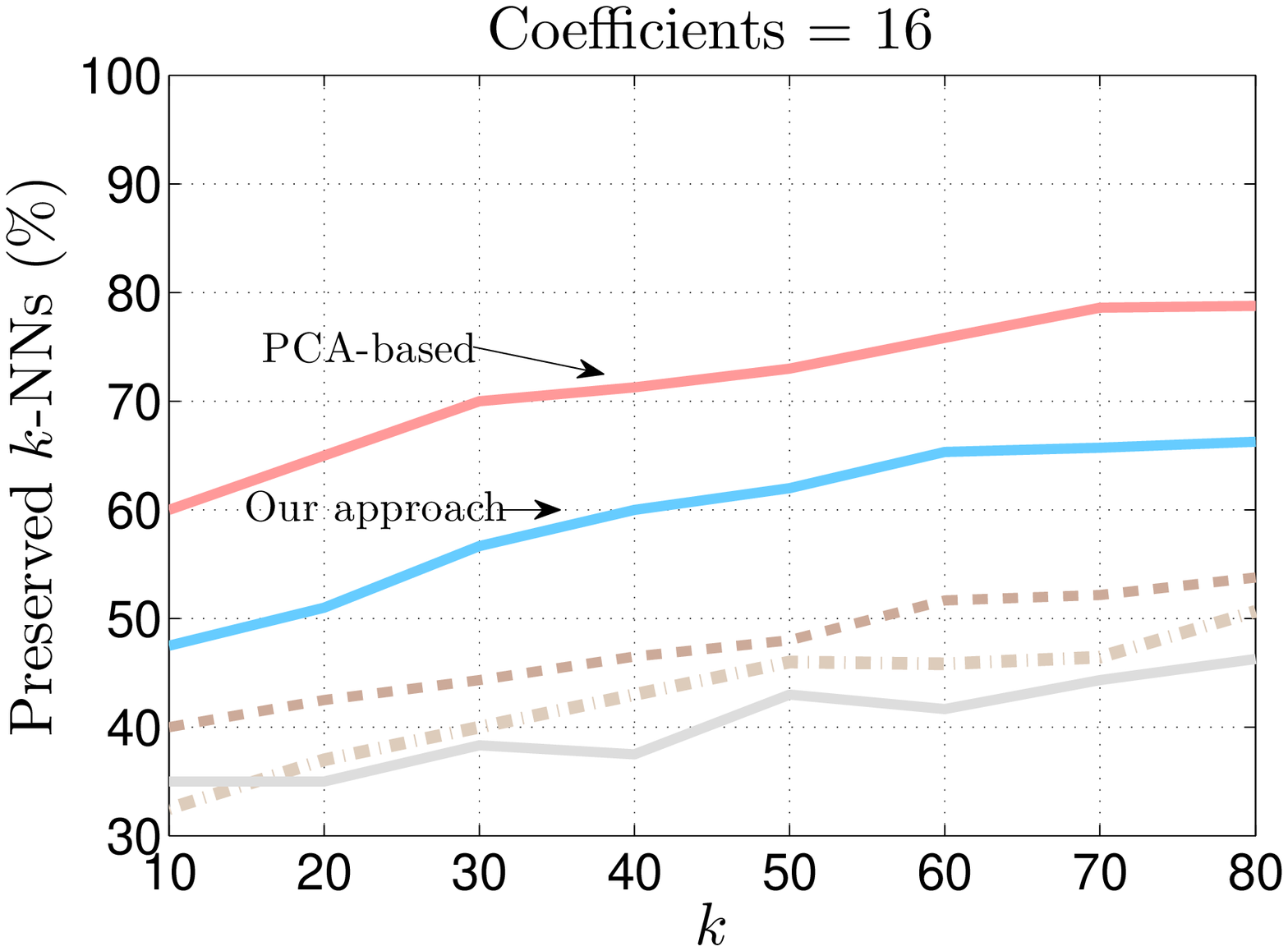} \includegraphics[width=0.325\textwidth]{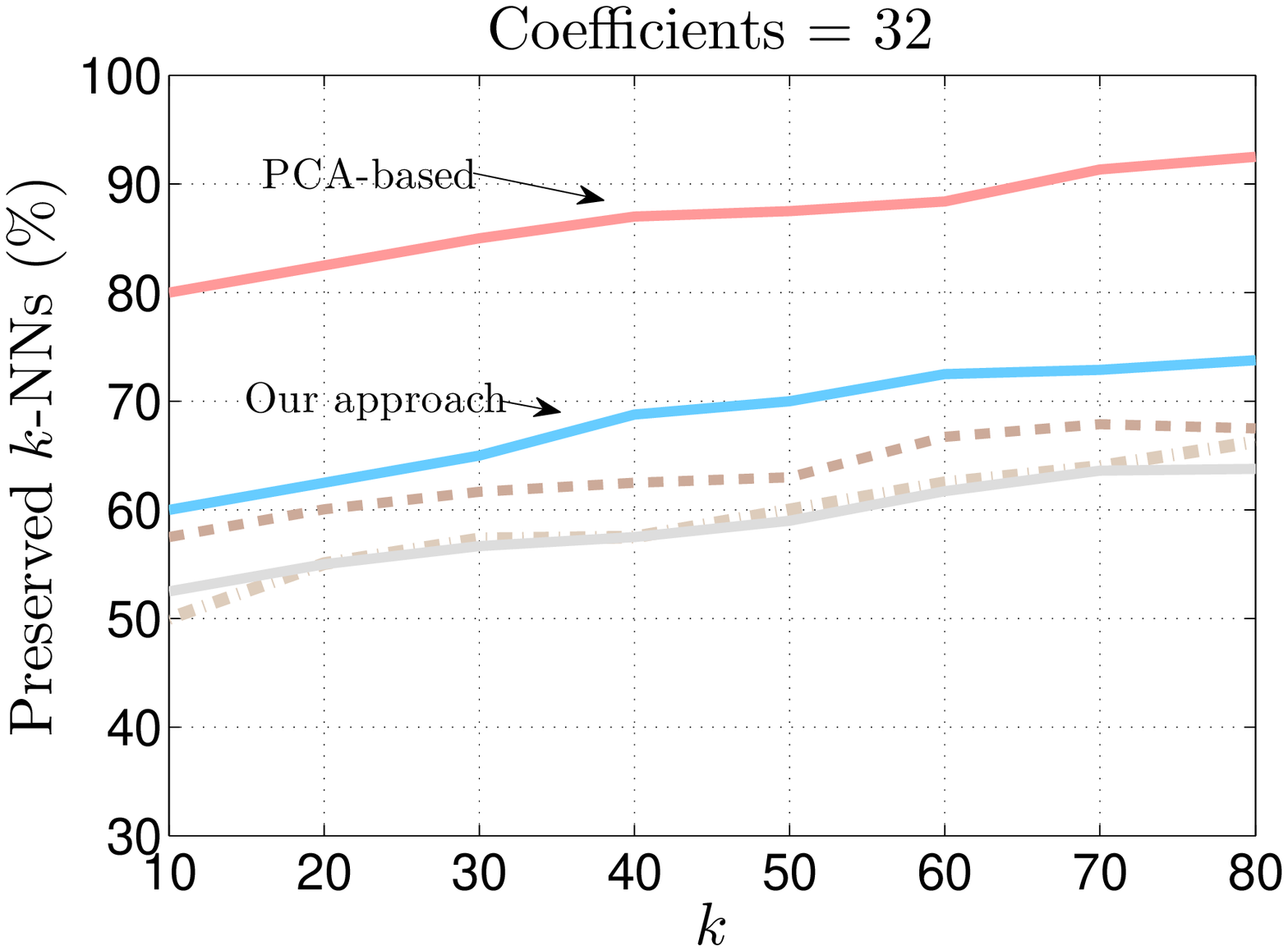} \includegraphics[width=0.325\textwidth]{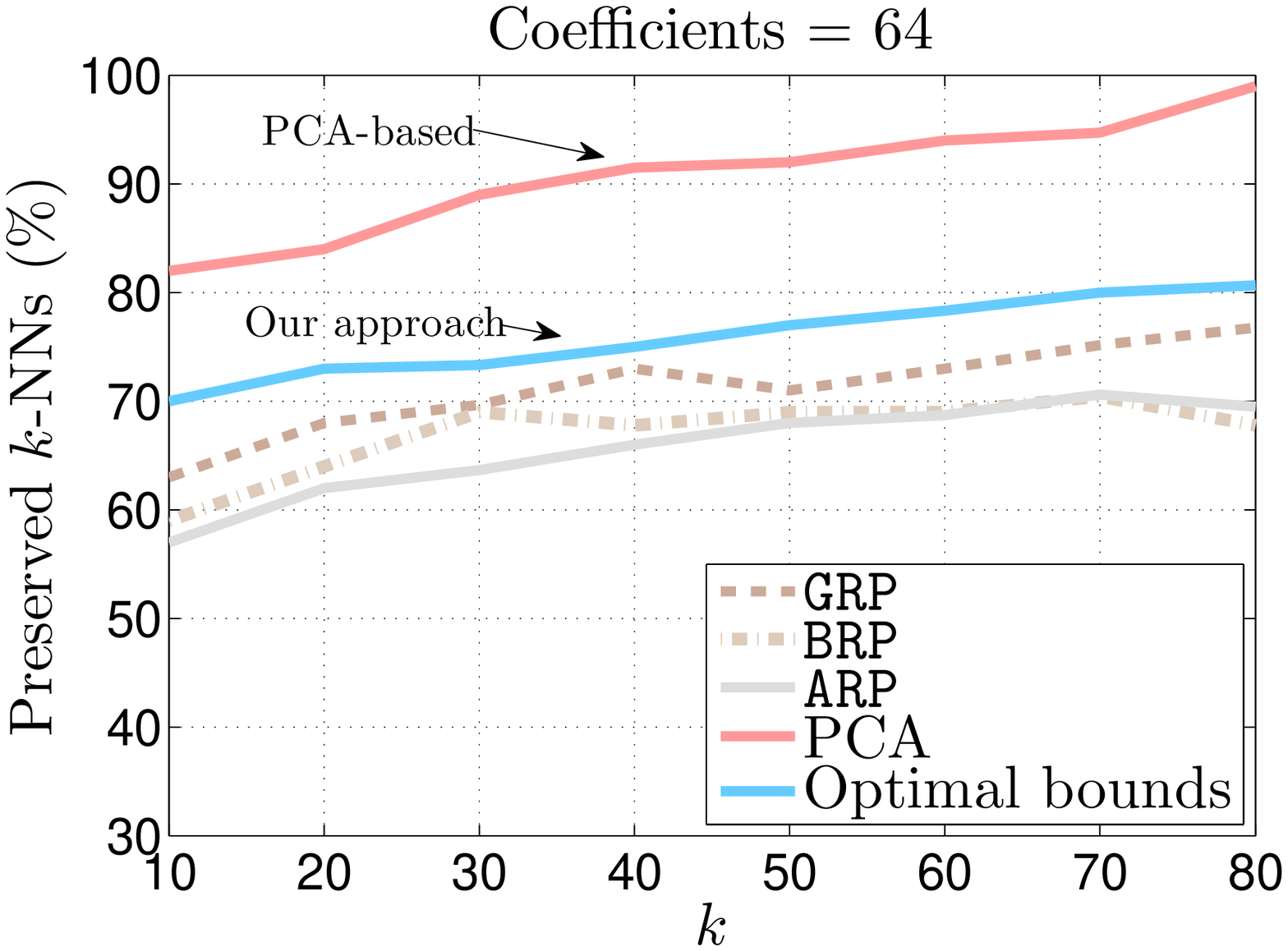} 
\caption{Comparison of the algorithms under consideration for the $k$-NN task, where $\X^{(i)}$/$\x^{(i)}$ is originally dense. The curves depict mean values of 100 Monte Carlo iterations.}
\label{fig:comp}
\end{figure*}

\medskip
\noindent \textit{$k$-NN performance:} 
As mentioned, we perform comparisons under {\it fair} settings.
Each object under our methodology is represented using $r = \left \lceil 2s + \frac{s}{2} + 1\right \rceil$ double variables
using $s$ coefficients in the Fourier basis. To compress each object using RP-based or PCA-based $k$-NN, we project each sequence onto $d$ dimensions such that the resulting low-dimensional projection point does not exceed the memory size of $r$ double values. The random matrices $\boldsymbol{\Phi}$ in the RP-based $k$-NN are {\it universally} near-isometric; i.e., with high probability, the same $\boldsymbol{\Phi}$ matrix serves as a good linear map for {\it any} input signal, the creation of $\boldsymbol{\Phi}$'s is performed once offline for each case; thus, we assume that this operation requires $\mathcal{O}(1)$ time and space complexity.

Figure \ref{fig:comp} displays the results, averaged over $100$ queries. Naturally, the PCA-based $k$-NN approach 
returns the best results because it uses all the data to construct the appropriate basis
on which to project the data. However, it requires the computation of a Singular Value or Eigenvalue Decomposition, a $\mathcal{O}(dN)$ (using Krylov power methods \cite{cullum2002lanczos}) and $\mathcal{O}(d^2 N)$ time complexity operation in the most favorable and average scenario, respectively. 

\begin{figure*}[!htpb]
\centering
\includegraphics[width=0.325\textwidth]{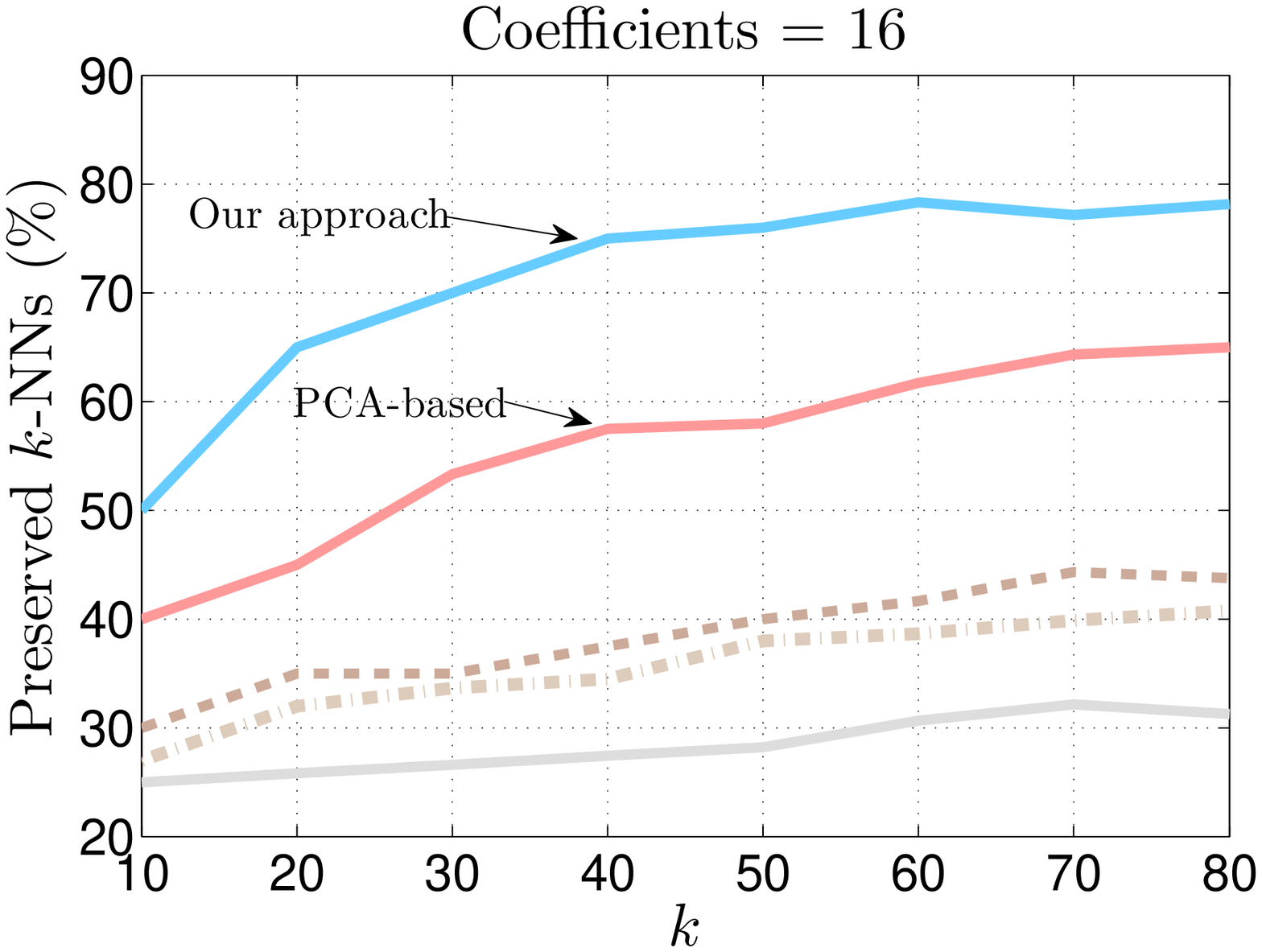} \includegraphics[width=0.325\textwidth]{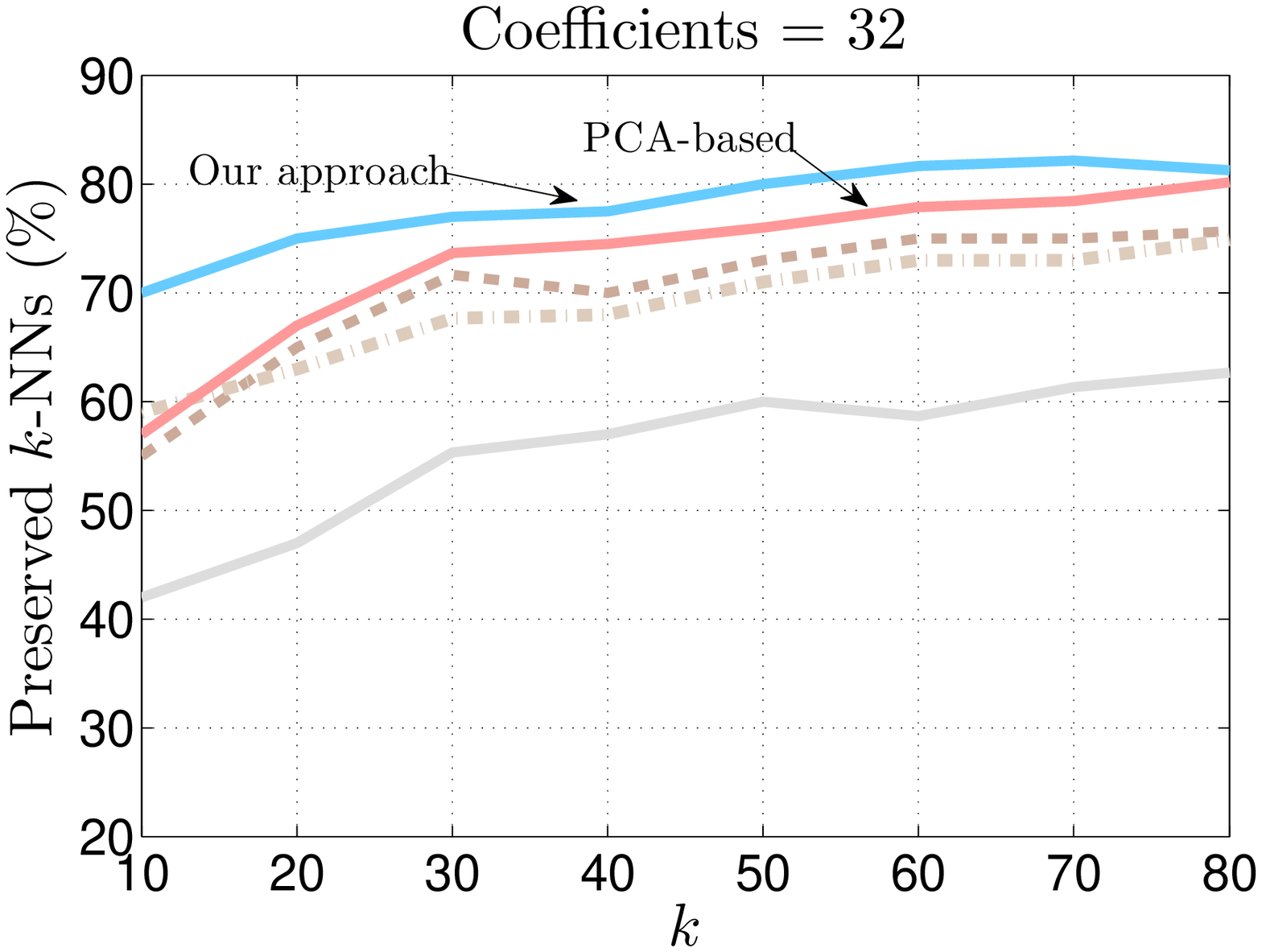} \includegraphics[width=0.325\textwidth]{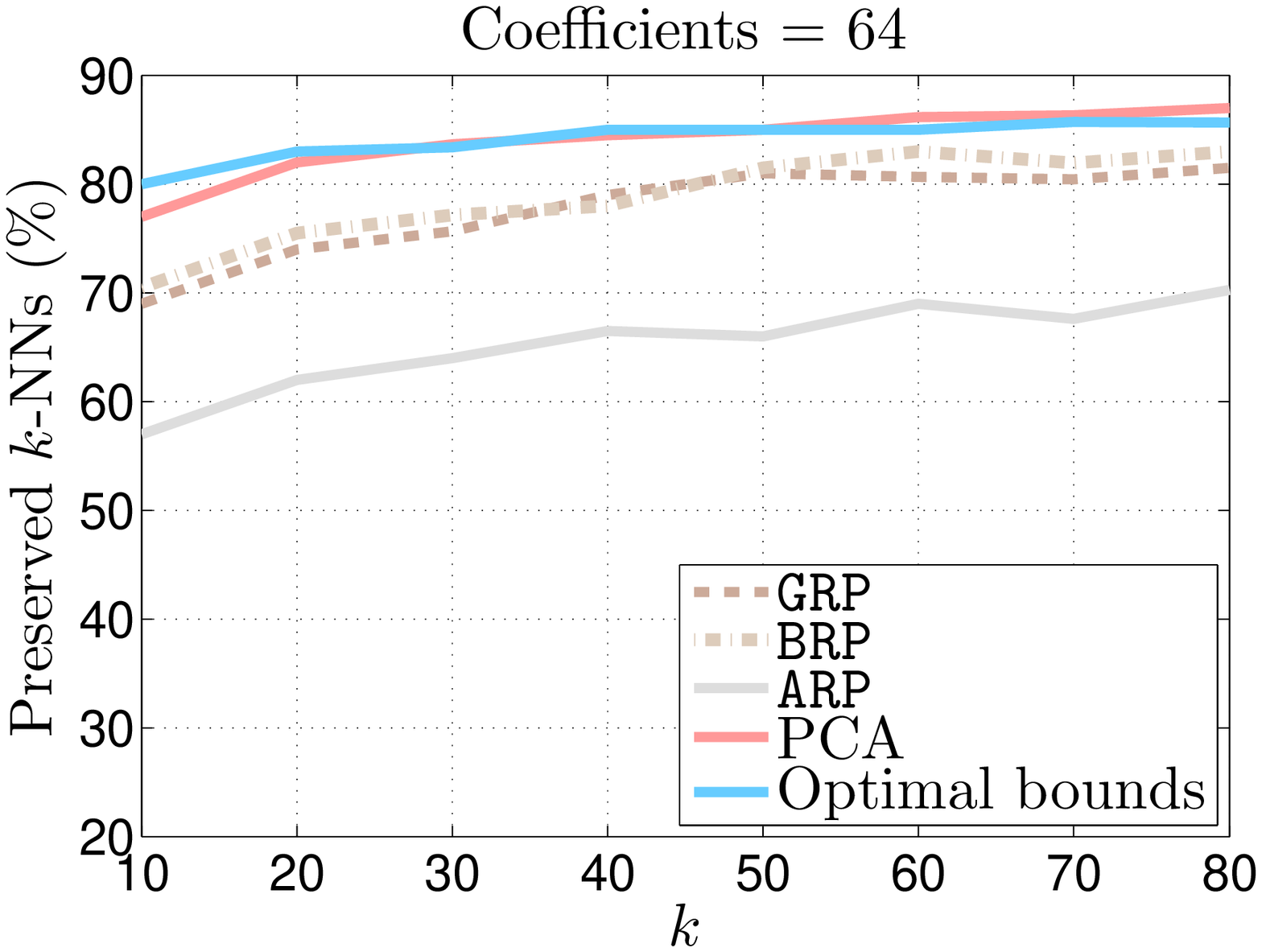} 
\caption{Comparison of the algorithms under consideration for the $k$-NN task, where each weblog sequence is sparsified to contain only $3s$ coefficients, where $s = \lbrace 16, 32, 64 \rbrace$. The curves depict mean values of 100 Monte Carlo iterations.} 
\label{fig:comp2}
\end{figure*}

Sacrificing accuracy for low complexity, the RP-based approaches constitute inexpensive solutions: constructing matrix $\boldsymbol{\Phi}$ is easy in practice, while binary-based matrix ensembles, such as \texttt{BRP} or \texttt{ARP} matrices, constitute low space-complexity alternatives. However, RP-based schemes are probabilistic in nature; they might ``break down'' for a fixed $\boldsymbol{\Phi}$: one can construct adversarial inputs where their performance degrades significantly. 

Our methodology presents a balanced approach with low time and space complexity, while retaining high accuracy of results. Our approach exhibits better performance over all compression rates, compared to the RP-based approaches, see in Figure \ref{fig:comp}.

\medskip \noindent 
\textit{When sparsity is present:} 
%
Using a similar methodology as in the previous section, we test the $k$-NN performance of the various approaches
when sparsity of the signals is present. We create sparse representations of the weblog data by subsampling their Fourier representation. 
Figure \ref{fig:comp2} illustrates the quality of $k$-NN using such a sparsified dataset. Per $\boldsymbol{\Phi} \in \mathbb{C}^{d \times N}$ matrix in the RP-based and the PCA-based $k$-NN, we project each weblog signal onto a $d$-dimensional (complex valued) space, where $d = \lceil s + \frac{s}{4} + \frac{1}{2}\rceil$. 

Figure \ref{fig:comp2} reveals a notable behavior of our model: When $\X^{(i)}$ is sparse, each $\X^{(i)}$ can be more accurately represented (and thus compressed) using fewer coefficients. 
Thus, on average, the energy discarded $e_x$ is also limited. Alternatively put, the constraint $\sum_{l \in p_{x}^{-}} |\X^{(i)}_l|^2 \leq e_x$ in \eqref{opt_d} \emph{highly restricts the candidate space that the uncompressed signal $\x^{(i)}$ resides in}, resulting in tighter upper and lower bounds. On the other hand, when compressing \emph{dense} $\X^{(i)}$'s into $s$ coefficients, where $s \ll N$, $e_x$ provides a large amount of uncertainty as to the reconstruction of $\x^{(i)}$.
This leads to less tight distance bounds and thus degraded performance. 

In summary, under high data-sparsity, our approach provides superior results in revealing 
the true $k$-NNs in the uncompressed domain. Our approach even outperforms PCA-based techniques, and more importantly, our method has a very low computational cost. 

\medskip \noindent 
\textit{Clustering quality:}
We assess the quality of $k$-Means clustering operations in the compressed domain using the original weblog dataset.
The quality is evaluated in terms of how similar the clusters are before and after compression
when $k$-Means is initialized using the same seed points. So, we use the same centroid points $C^{(t)}$ as in the uncompressed domain and then compress each $C^{(t)}$ accordingly, using the dimensionality reduction strategy 
dictated by each algorithm. 

\begin{figure*}[!htpb]
\centering
\includegraphics[width=0.33\textwidth]{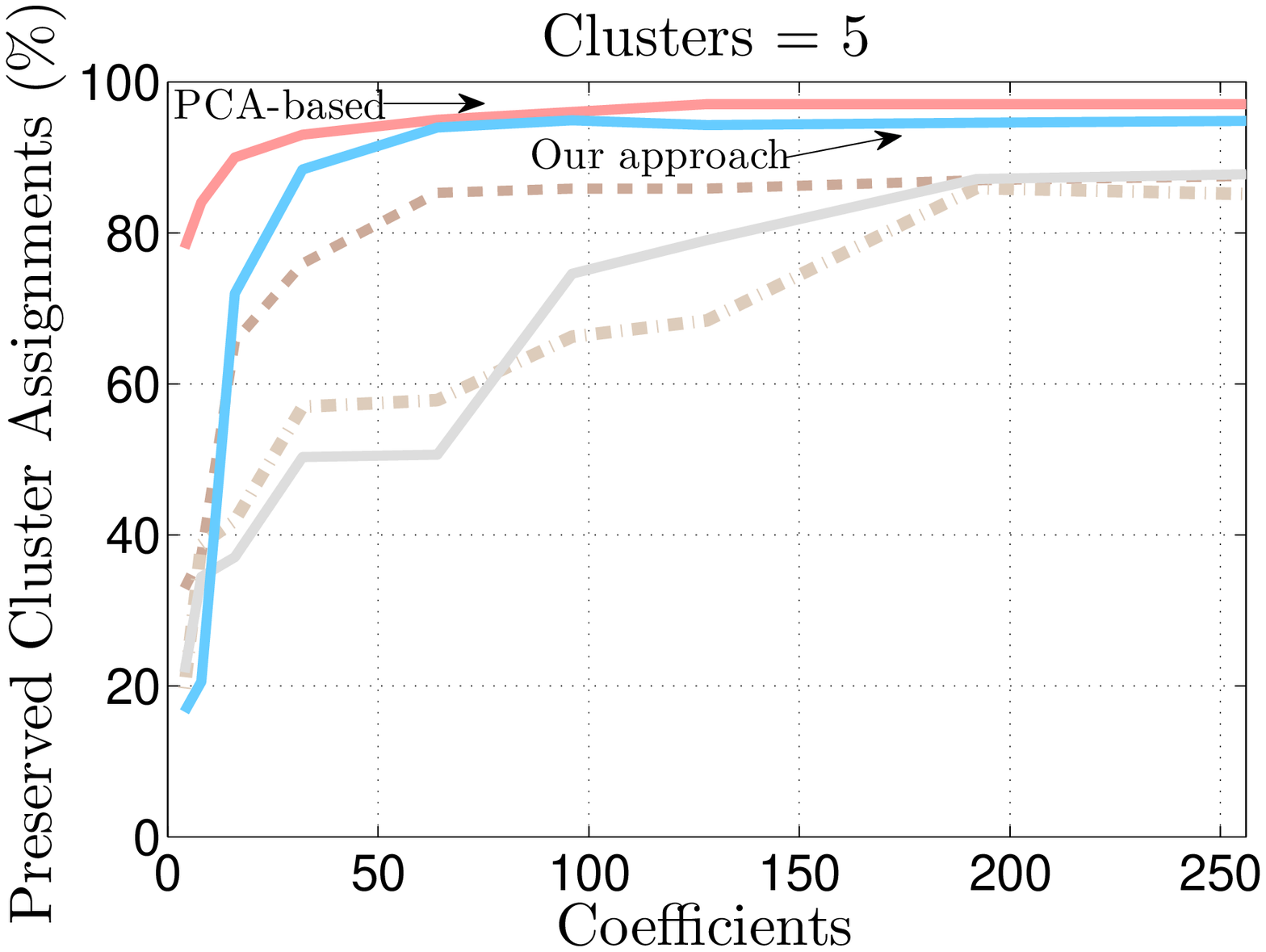} \includegraphics[width=0.33\textwidth]{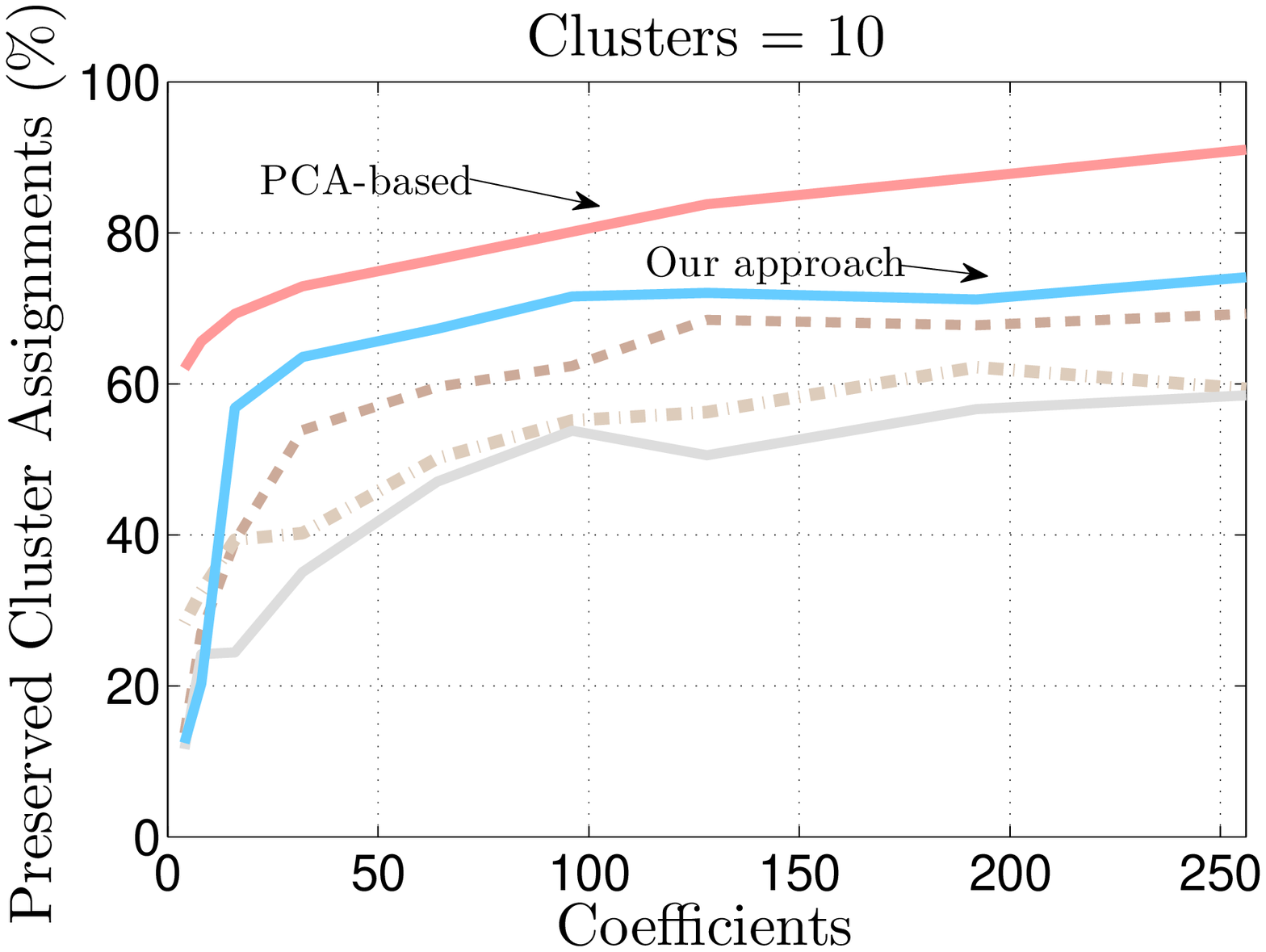} \includegraphics[width=0.33\textwidth]{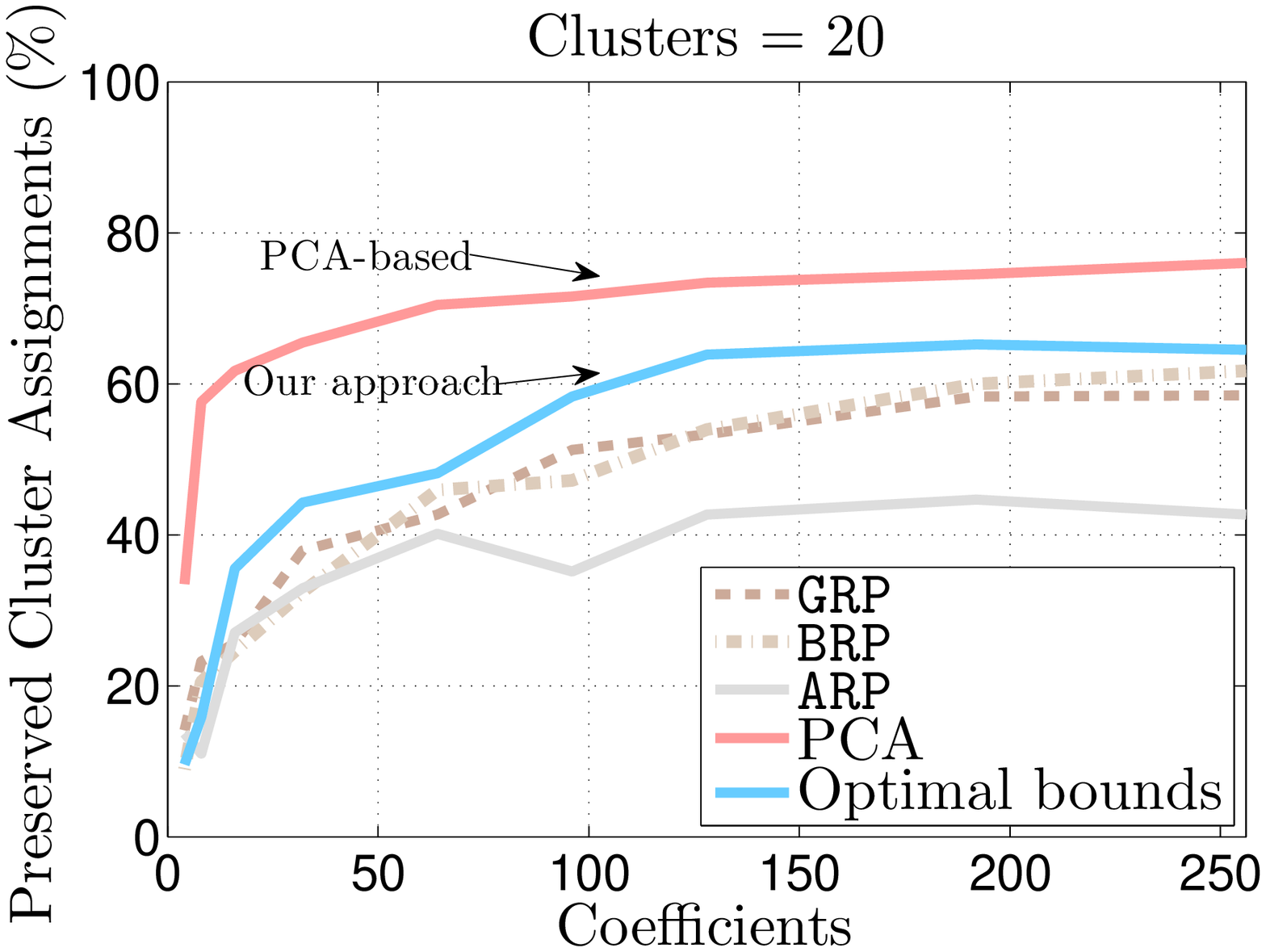}
\caption{Assessing the quality of $k$-Means clustering in the compressed domain.} 
\label{fig:kmeans}
\end{figure*}

The quality results of the algorithms under comparison are depicted in Figure \ref{fig:kmeans}. 
We perform $k$-Means for different compression ratios (coefficients) and for different numbers of clusters $k$. 
The PCA-based approach returns the best performance, but introduces very large computational demands due to SVD/Eigenvalue computation. 
The performance of our methodology lies in-between PCA and Random-Projection techniques.

\subsection{Using a different basis (wavelets)}
In the preceding sections we used Fourier decomposition as our compression method. Now we use wavelets \cite{mallatwavelet}
to show the generality of our technique. We also use a very large image dataset consisting of VLSI patterns obtained
by the semiconductor department of our organization. 

In the remainder of this subsection we $(i)$ provide an overview of the tasks and challenges related to the specific domain and, $(ii)$ show the performance of $k$-NN search operations on this large real-world dataset.

\medskip
\noindent \textit{Pattern detection on VLSIs:} 
During the production of VLSI circuits (i.e., CPUs), a series of parameters and design protocols should be satisfied to ensure that the resulting product will not fail during the manufacturing process. 
For example, there are various \emph{layouts configurations}, that have been known
to cause short-circuits and jeopardize the reliability of a circuit.
The absence of such artifacts guarantees that the circuit will connect as desired and ensures a margin of safety. 
A nonexhaustive list of parameters includes the width of metal wires, minimum distances between two adjacent objects, etc. We refer the reader to Figure \ref{fig:drc} for some illustrations.
Design-rule checking (DRC) is the process of checking the satisfiability of these rules. 

As a result of a testing process, a collection of VLSI designs are annotated
as faulty or functional. Now, each newly-produced circuit layout is classified as 
potentially-faulty  based on its similarity to an already annotated circuit. Novel designs,
never seen before, need to be tested further. Therefore, the testing process can be relegated
to a $k$-NN search operation. The whole process needs to be both expedient and accurate, so that design
and production are properly streamlined.

\begin{figure*}[!htpb]
\centering
\includegraphics[width=0.35\textwidth]{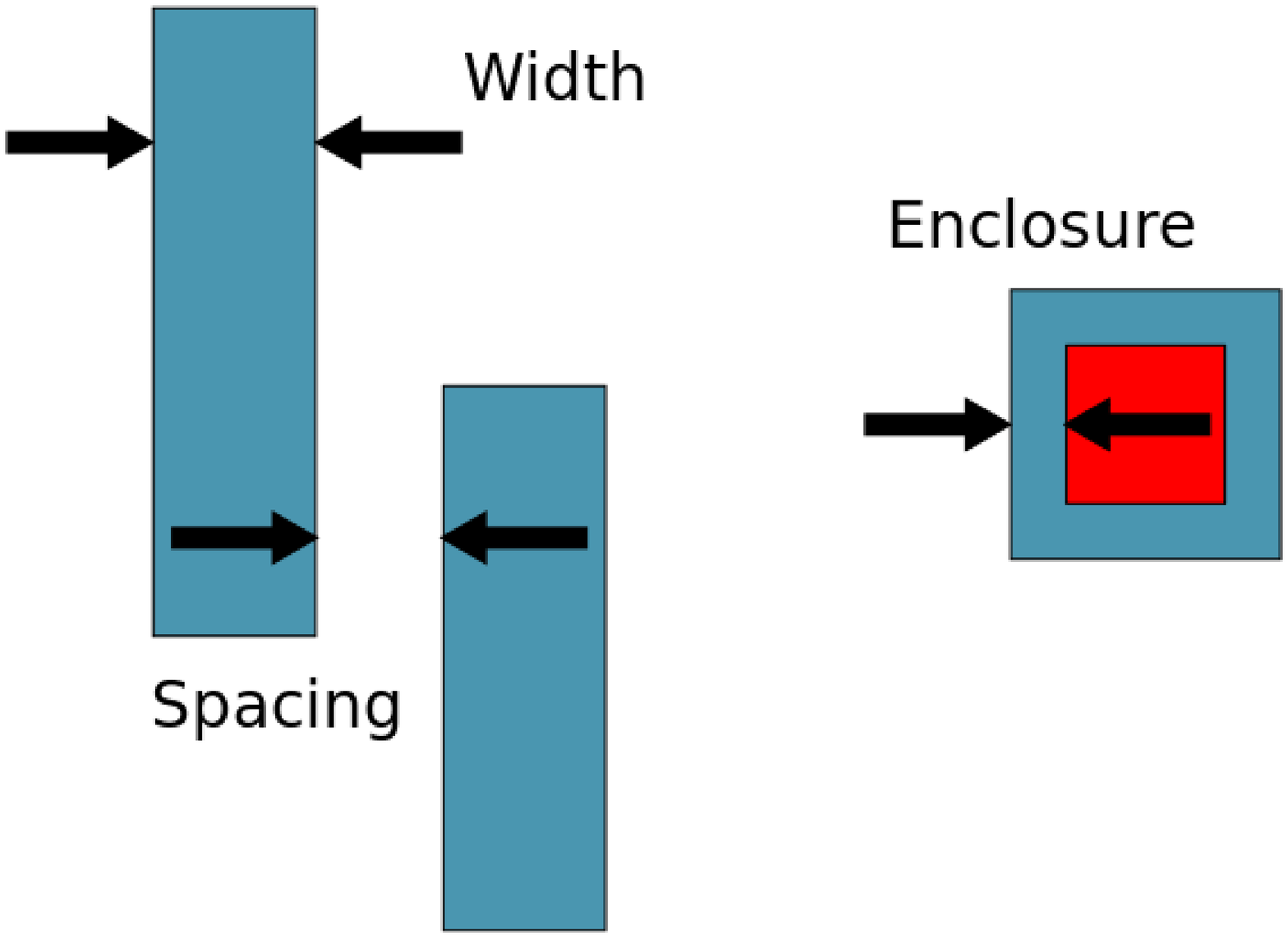} \hspace{1cm} 
\includegraphics[width=0.4\textwidth]{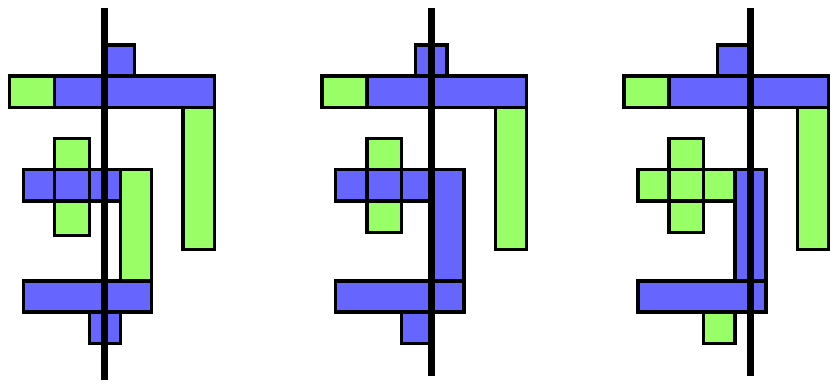}
\caption{\textbf{Left:} Subset of design parameters to be satisfied during the manufacturing \cite{wikiDRC}. Width and spacing are single layer rules, where the VLSI layout is seen as a 2D object. A width rule specifies the minimum width of objects; a spacing rule specifies the minimum distance between two adjacent objects. Enclosing deals with multi-layer rules is not considered here. \textbf{Right:} In the left-to-right scanning strategy, moving a vertical scan line horizontally across the layout, we can maintain the sum of polygons observed. As the scan line advances, new objects are added and old ones are removed. 
}  
\label{fig:drc}
\end{figure*}
\begin{figure*}[!htpb]
\centering
\includegraphics[width=0.29\textwidth]{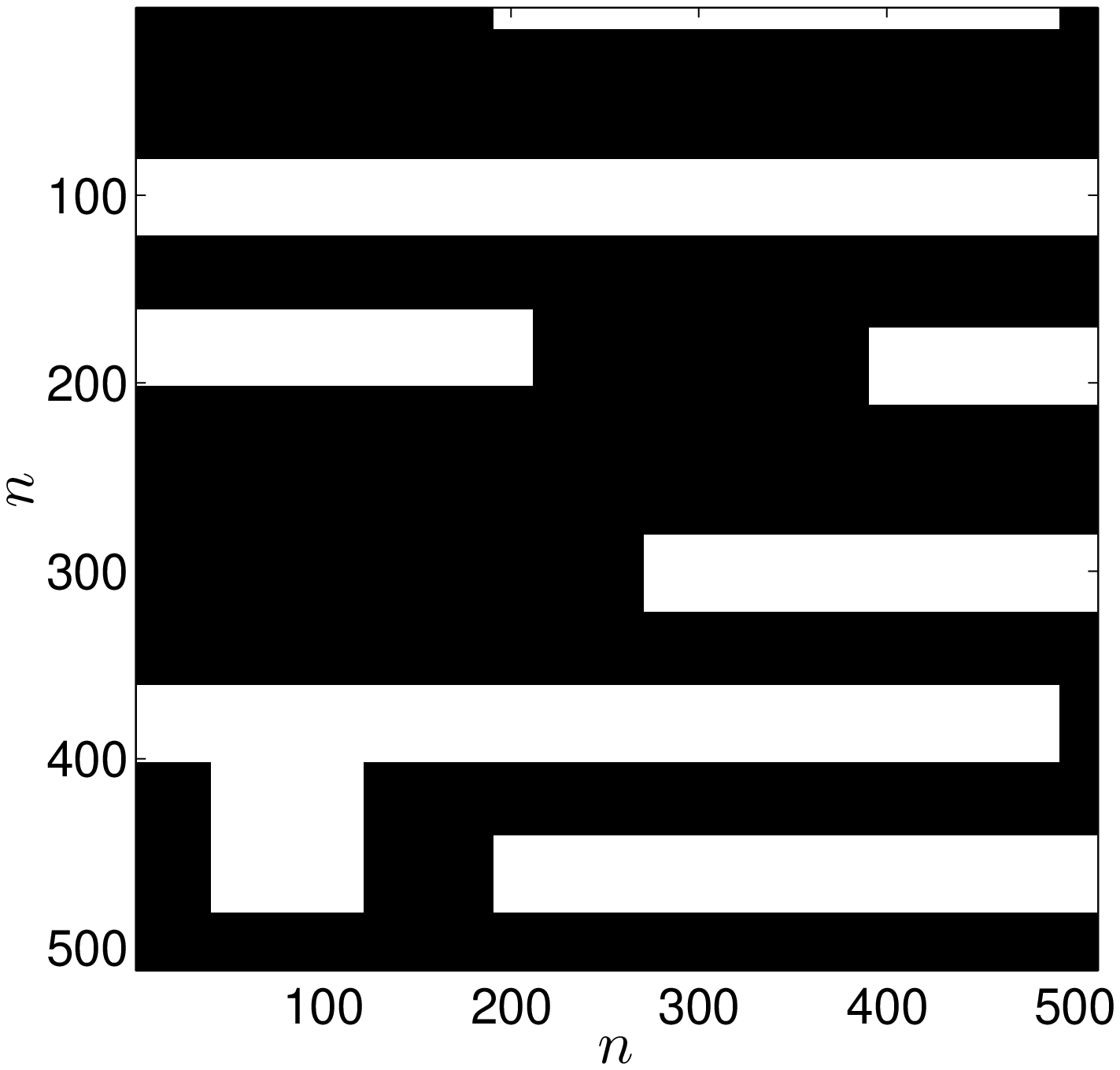} \includegraphics[width=0.34\textwidth]{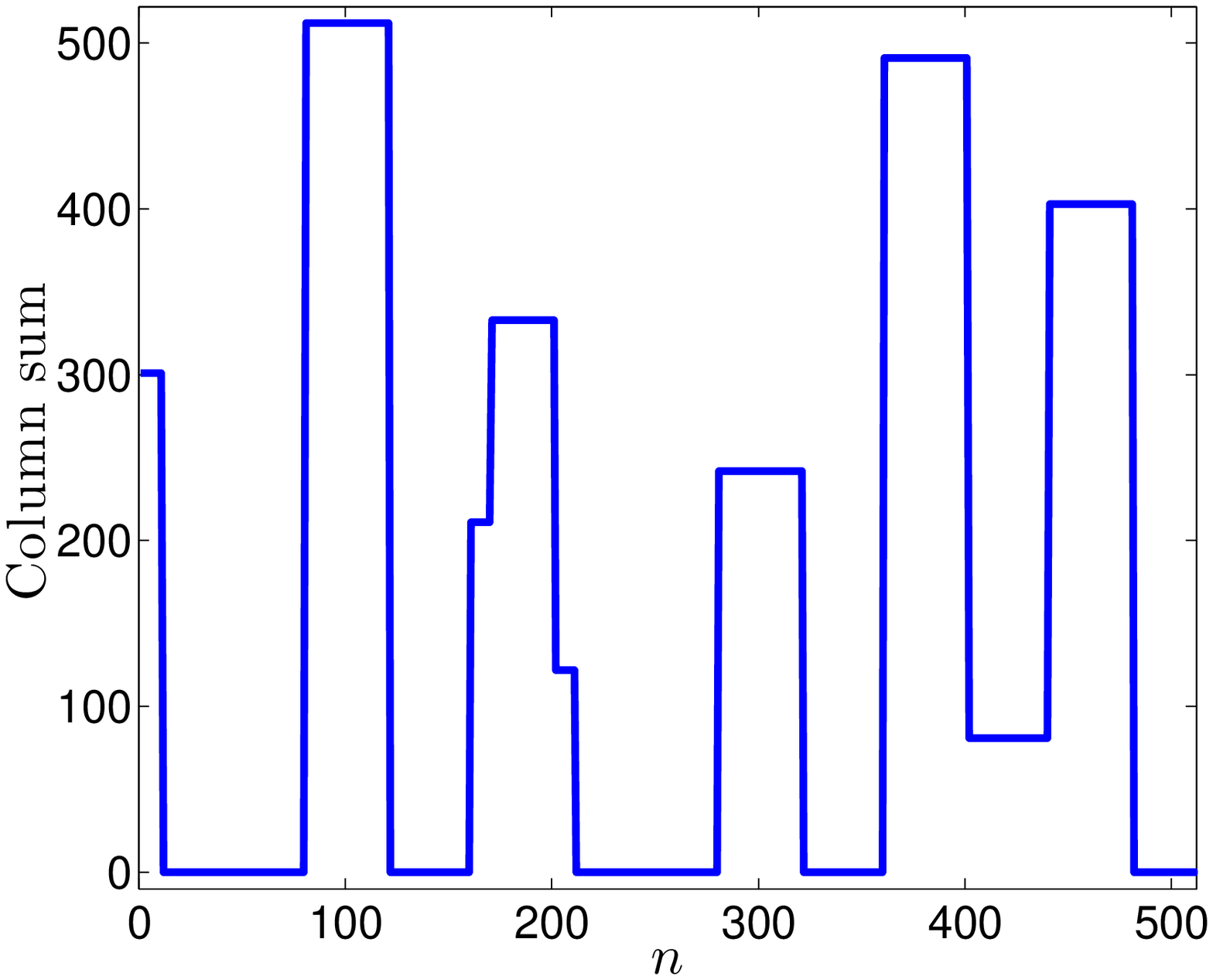} \includegraphics[width=0.345\textwidth]{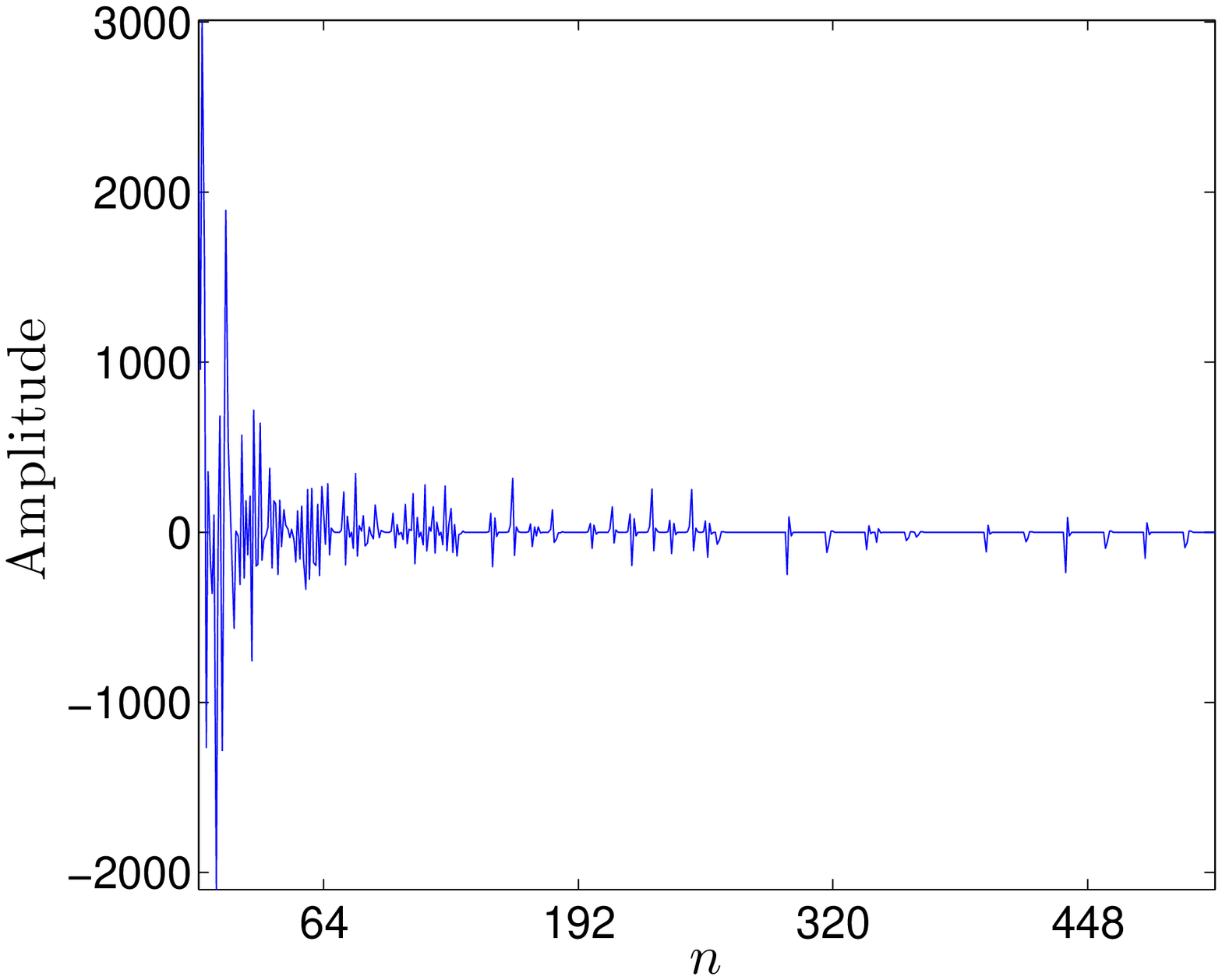}
\caption{\textbf{Left panel:} Original 2D layout example. White areas indicate the presence of objects, e.g., metal wires. \textbf{Center panel:} Column-sum representation of polygons in $\mathbb{R}^{512}$. \textbf{Right panel:} Amplitude of wavelet transformation on the column-sum representation: the majority of the energy can be captured in only a few coefficients. }  
\label{fig:layout_example}
\end{figure*}

Based on an algebra of polygons \cite{crawford1975design}, a strategy to compare different layouts is by projecting their binary, 2D description onto one of the axes. This generates a descriptive vector per layout as the summary scanned from left to right or from top to bottom, see Figure \ref{fig:drc}. Thus, each VLSI layout can be approximately (but not uniquely) represented by a \emph{signature} as the sum along rows or columns. Note that not only
this format can be stored and processed more efficiently, but it also allows a translation-invariant matching of the shape.
One can use both row-wise and column-wise signatures, but for the purposes of our experiments,
we only use the column-wise signature.

\medskip
\noindent \textit{VLSI dataset:}
Our original dataset consists of approximately $150,000$ binary VLSI images, each of dimension 512x512.
We convert each image into a signature $\ell^{(i)} \in \mathbb{R}^{512}, \forall i,$ as the column-sum of each image. 
Figure \ref{fig:layout_example} depicts an instance of a layout image and its resulting signature. 

Afterwards, we compress each signature using the wavelet transformation \cite{mallatwavelet}:
$$\mathcal{L}^{(i)} = \texttt{WVL}\left(\ell^{(i)}\right) \in \mathbb{R}^{512}, \forall i,$$ where $\texttt{WVL}\left(\cdot\right)$ represents the wavelet linear operator with minimum scale $\texttt{J}_{\text{min}} = 2$. Observe in the right part of  Figure \ref{fig:layout_example} that $\mathcal{L}^{(i)}$ is highly compressible: the energies of wavelet components decay rapidly to zero according to a power-law decay model of the form:
\begin{align}
\left|\left(\mathcal{L}^{(i)}\right)_j\right| \leq R \cdot j^{1/p}, ~R > 0, \forall j,
\end{align} for some $r$ and $p$. This suggests that each signature is highly compressible
with minimal loss in accuracy.

\begin{figure*}[!htpb]
\centering
\includegraphics[width=0.33\textwidth]{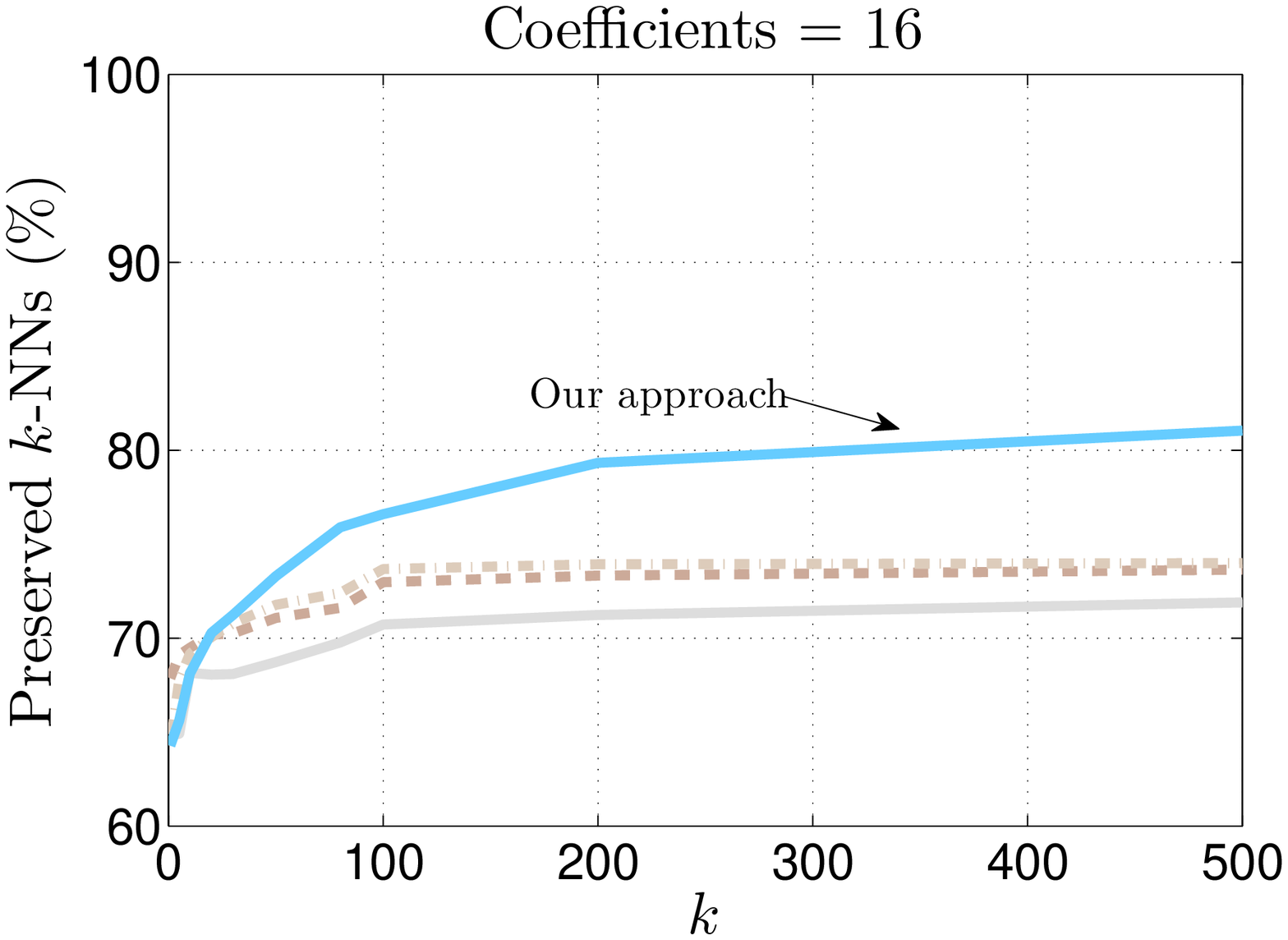} \includegraphics[width=0.33\textwidth]{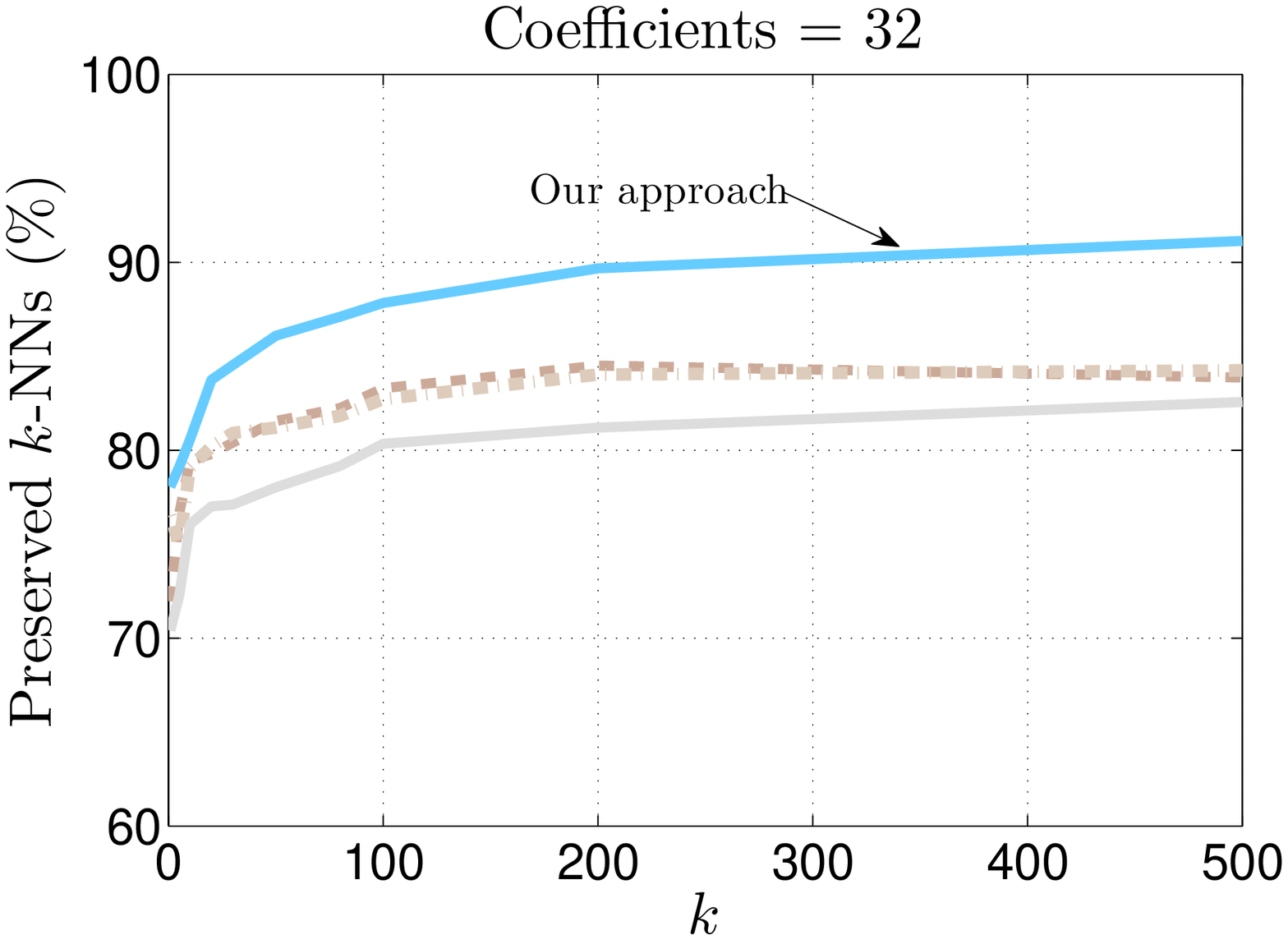} \includegraphics[width=0.32\textwidth]{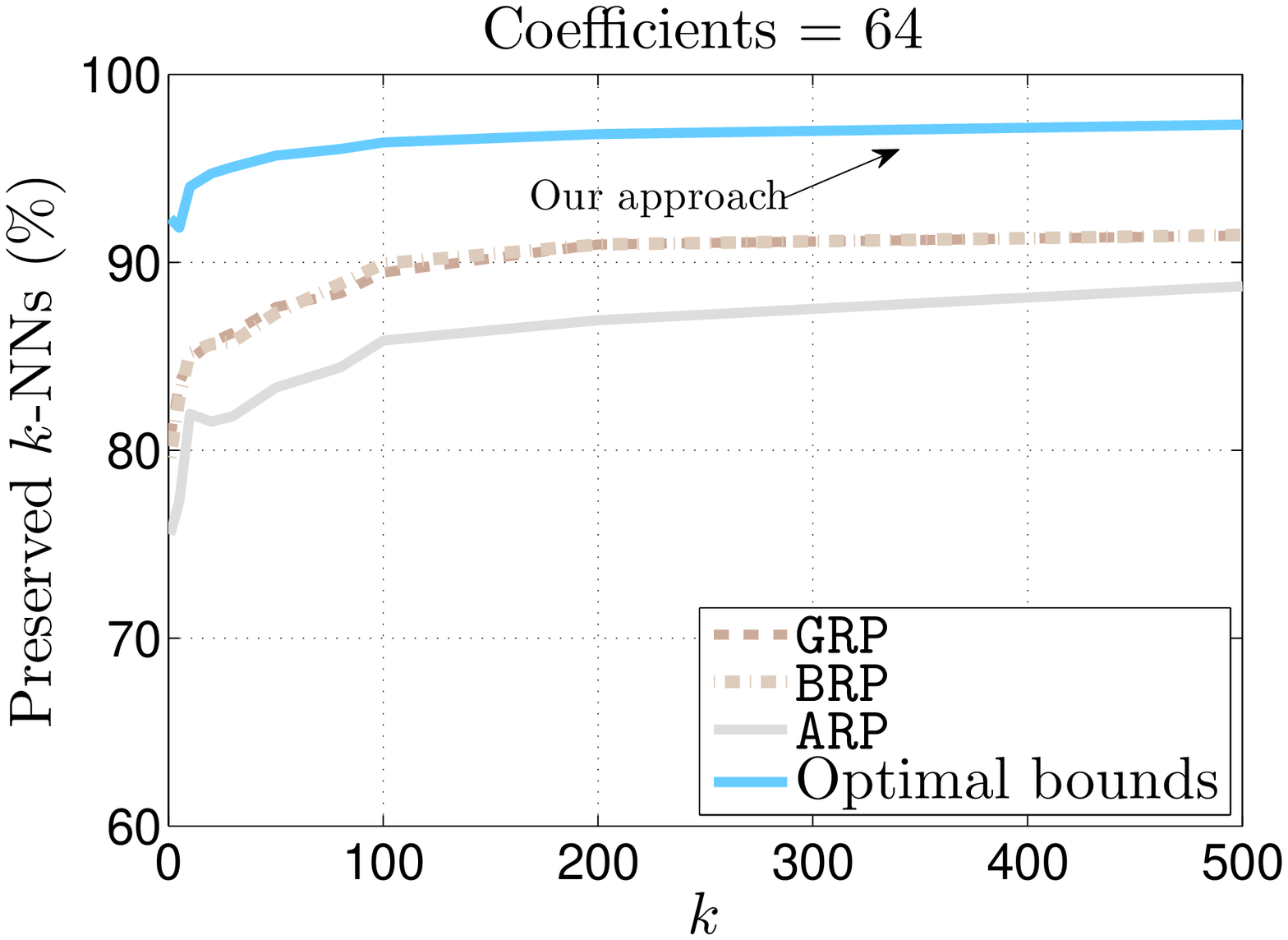}
\caption{$k$-NN preservation performance as function of $k$. Here, the cardinality of $\mathcal{DB}_r$ is $|V| = 149,245$ and the byte size of each sequence is $d$ bytes.}
\label{fig:VLSI_results}
\end{figure*}

\begin{figure*}[!htpb]
\centering
\includegraphics[width=.9\textwidth]{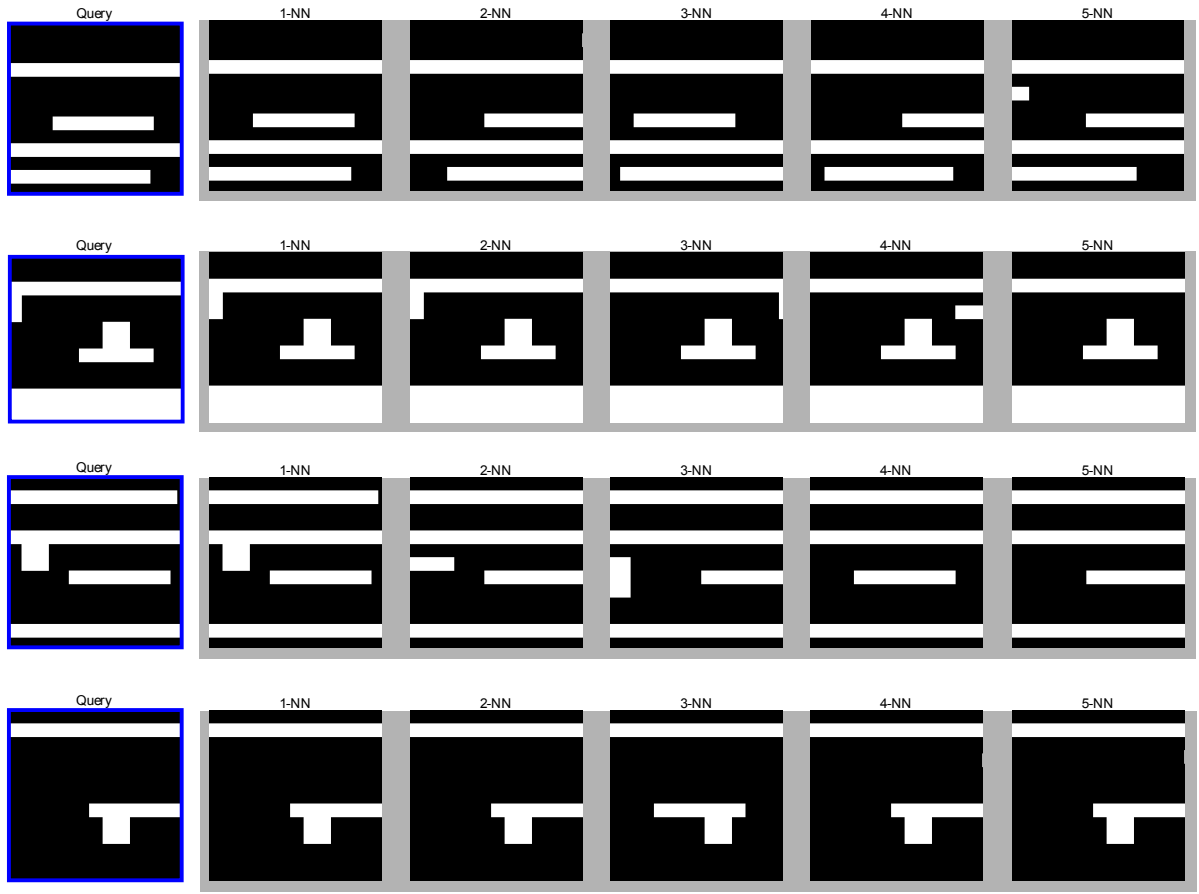}
\caption{Illustrative examples of the \texttt{Optimal Bounds} approach performance for the $k$-NN problem. 
The leftmost images are the query images, and on the right we depict the $k = 5$ nearest neighbors as computed in the compressed domain. Observe that owing to the signature extracted from each image, we can also detect translation-invariant matches.}
\label{fig:VLSI_example1}
\end{figure*}

\begin{figure*}[!htpb]
\centering
\includegraphics[width=0.33\textwidth]{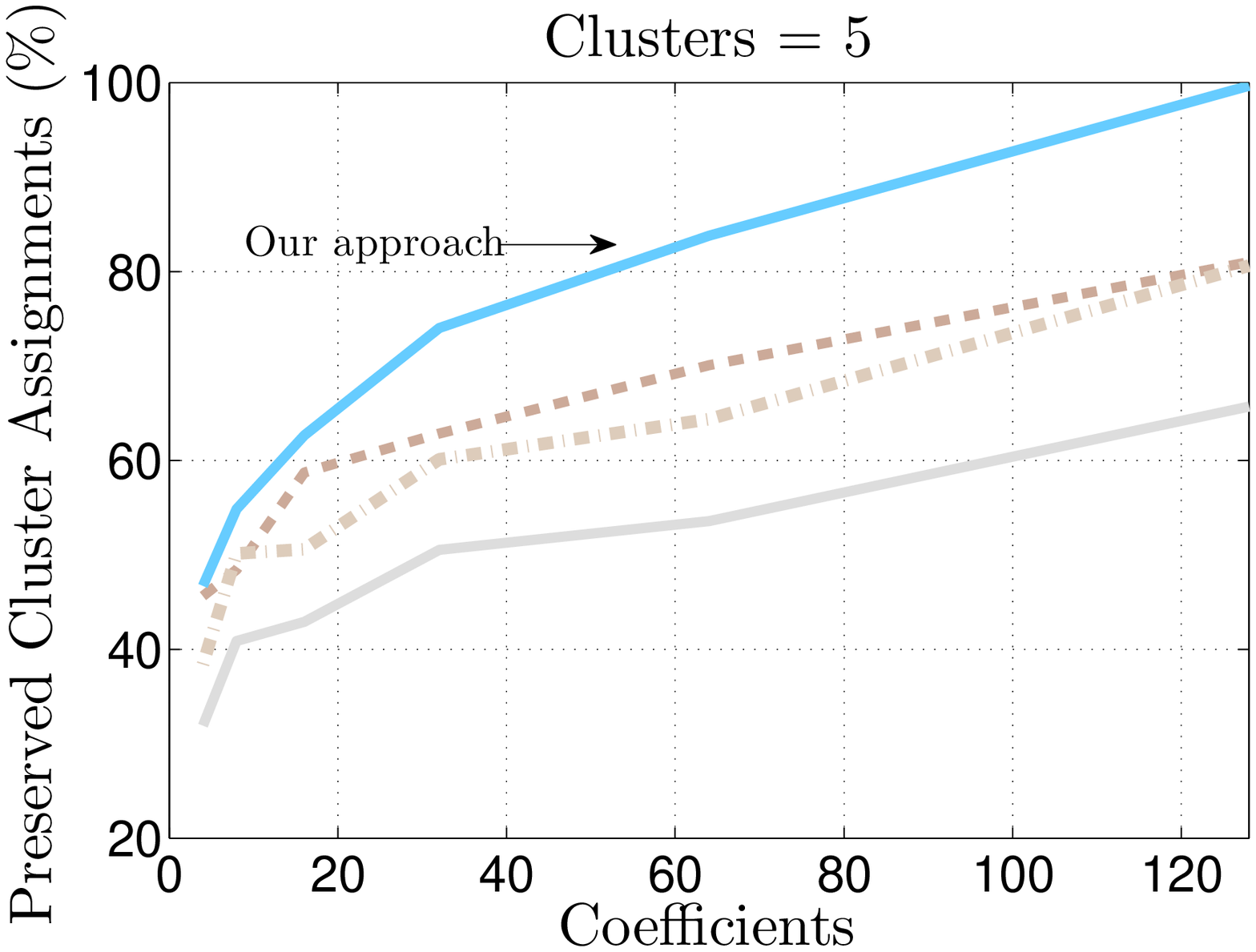} \includegraphics[width=0.33\textwidth]{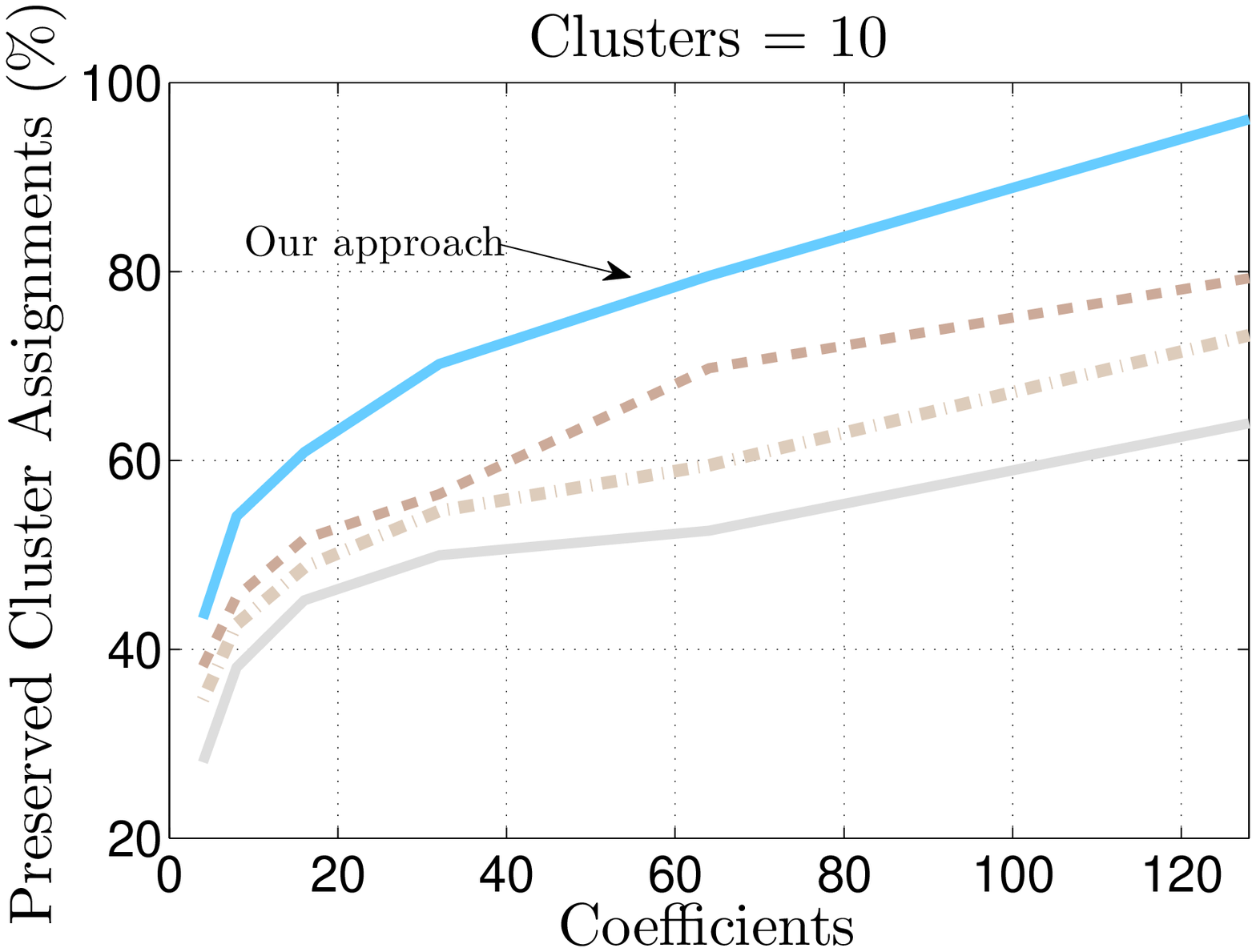} \includegraphics[width=0.33\textwidth]{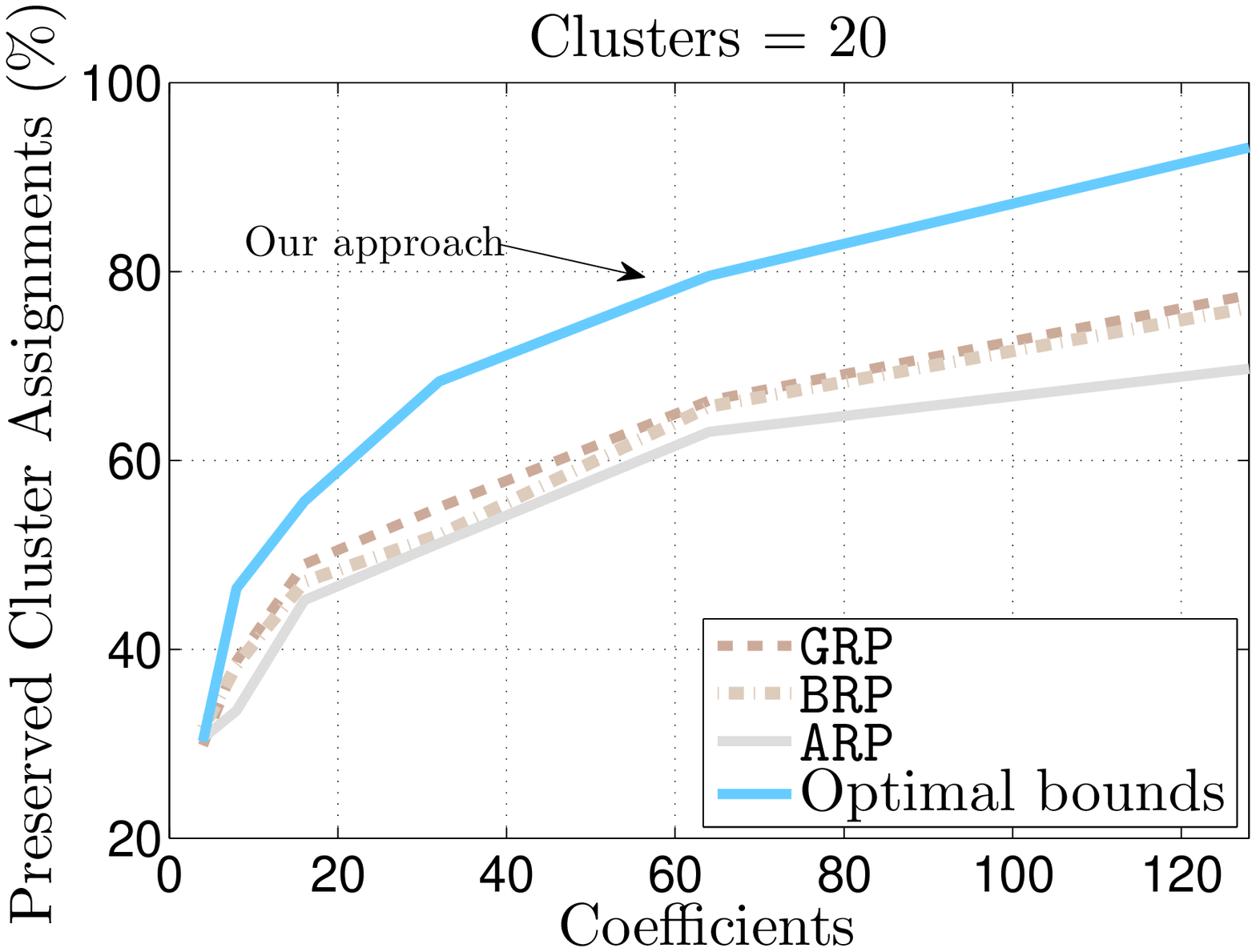}
\caption{Clustering preservation performance as function of the number of coefficients $s$. RP-based approaches project onto a $d$-dimensional subspace, where $d = \lceil 2s + \frac{s}{2} + 1 \rceil$. Here, the cardinality of the dataset is $|V| = 50,000$ and the byte size of each sequence is $d$ bytes.}
\label{fig:VLSI_kMeans_results}
\end{figure*}

\medskip
\noindent \textit{$k$-NN results on VLSI layouts:}
We evaluate the quality of the results of $k$-NN operations in the compressed
domain as compared to the $k$-NN results when operating on the uncompressed image data.
We limit our evaluation to comparing our methodology to Random-Projection techniques,
since only these approaches are scalable for large datasets.

Figure \ref{fig:VLSI_results} illustrates the performance of the following approaches: $(i)$ RP-based $k$-NN for three different random matrix ensembles (Gaussian, Bernoulli and Achlioptas' based) and $(ii)$ our \texttt{Optimal Bounds} approach. We observe that our method can improve the relative efficiency of matching by up to $20\%$, compared with the best random-projection approach. Finally, Figure \ref{fig:VLSI_example1} provides some representative examples of the $k = 5$ nearest neighbors for four randomly selected queries. Using the layout signature derived, we can discover very flexible translation-invariant matches.

\medskip \noindent
\textit{Clustering quality on VLSI layouts:}
We assess the quality of $k$-Means clustering operations in the compressed domain.
As before, the quality is evaluated as to how similar the clusters are before and after compression
when $k$-Means is initialized using the same seed points. So, we use the same centroid points $C^{(t)}$ as in the uncompressed domain and then compress each $C^{(t)}$ accordingly, using the dimensionality reduction strategy 
dictated by each algorithm. Again, we consider the $k$-Means algorithm as our baseline procedure. 

Figure \ref{fig:VLSI_kMeans_results} depicts the results for three clustering levels $k$: $5$, $10$ and, $20$ clusters. We perform $k$-Means for different compression ratios (coefficients) in the range $s = \lbrace 4, 8, 16, 32, 64, 128 \rbrace$. 
We evaluate how strong the distortion in clustering assignment is when operating in the compressed domain compared witg the
clustering on the original data.  For all cases, our approach provides cluster output that aligns better with the original clustering.
For this dataset we observe a consistent $5-10\%$ improvement in the cluster quality returned.
These trends are captured in Fig. \ref{fig:VLSI_kMeans_results}.



\medskip
In summary, the above experiments have provided strong evidence that our methodology can offer better mining quality in the compressed domain than random projection techniques, both in the traditional and in the compressed-sensing sense (i.e., high data-sparsity).
Finally, because our approach is also very fast to compute (e.g. compared with PCA), we believe that it will
provide an important building block for transitioning many mining operations into the compressed data space.


\begin{figure*}[!ht]
\centering
\includegraphics[width=0.33\textwidth]{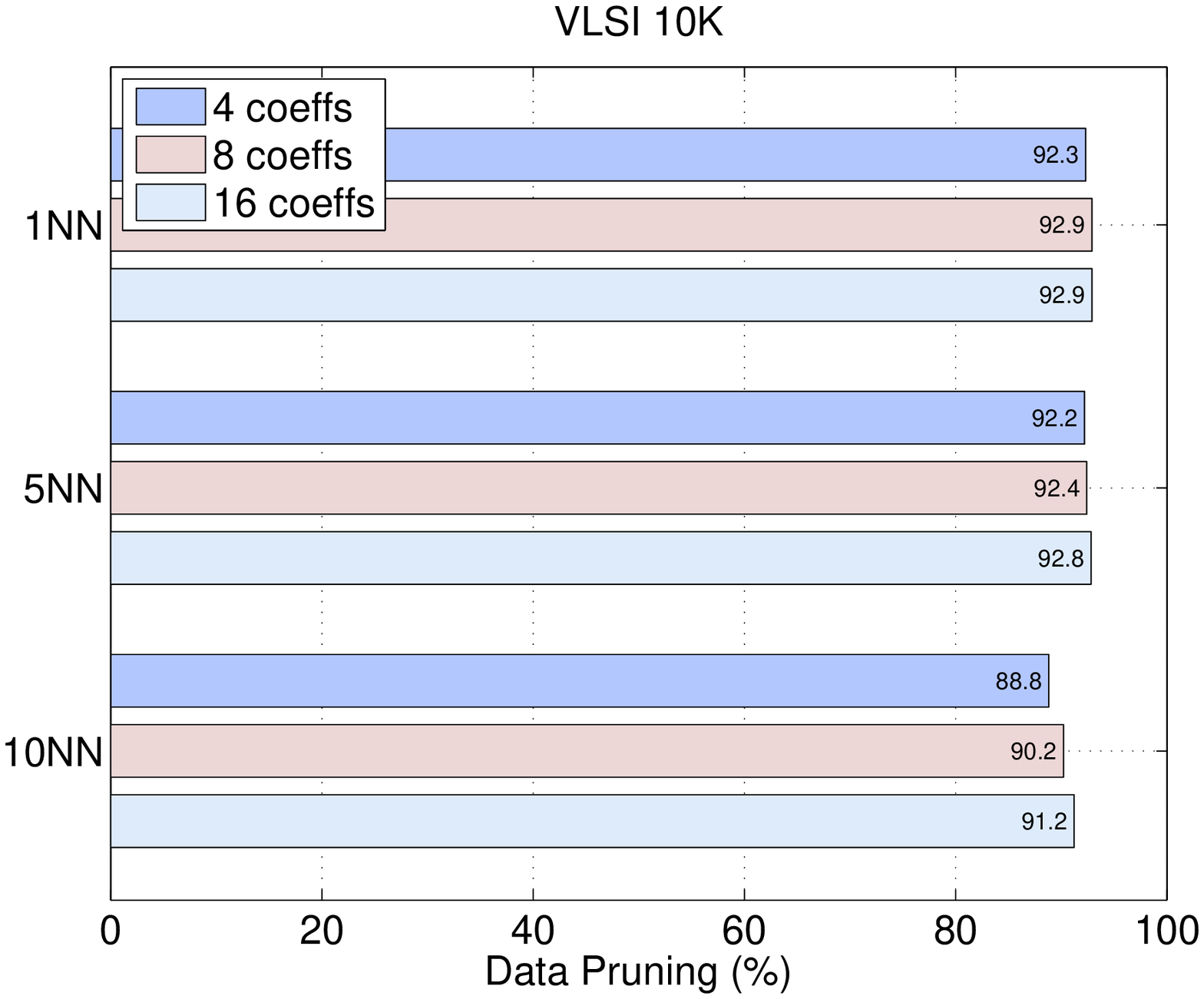} 
\includegraphics[width=0.33\textwidth]{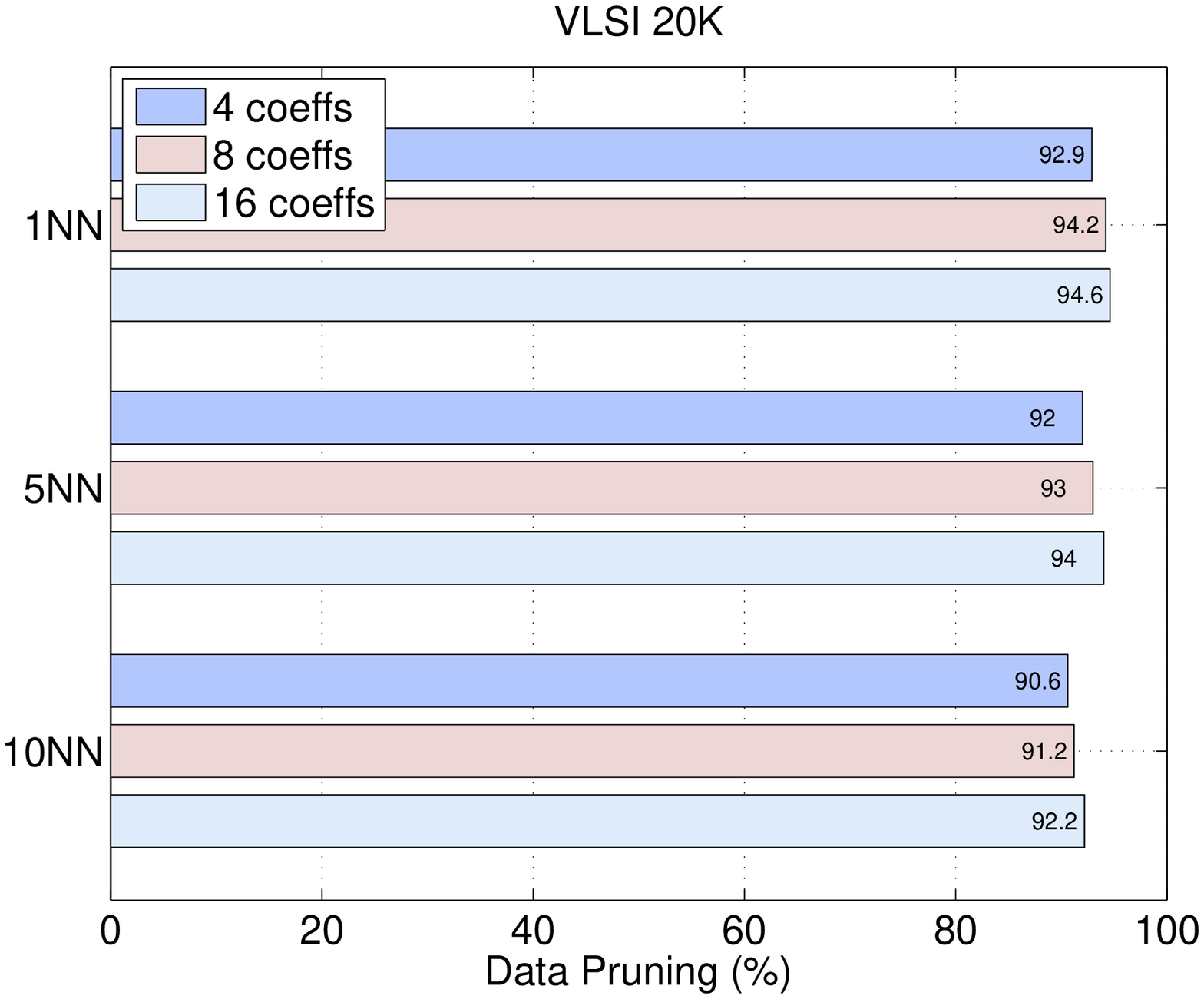}
 \includegraphics[width=0.33\textwidth]{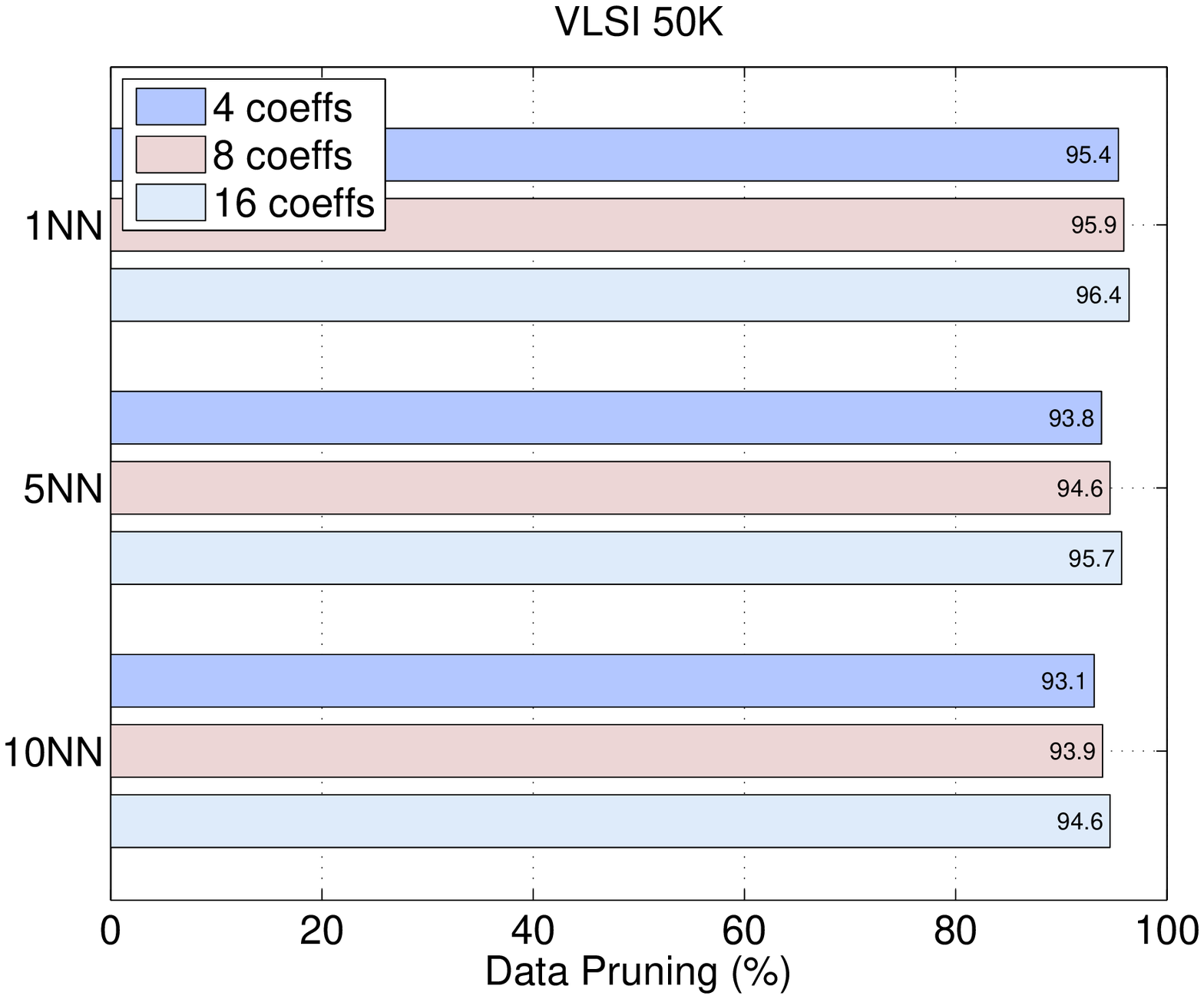} \\
\includegraphics[width=0.33\textwidth]{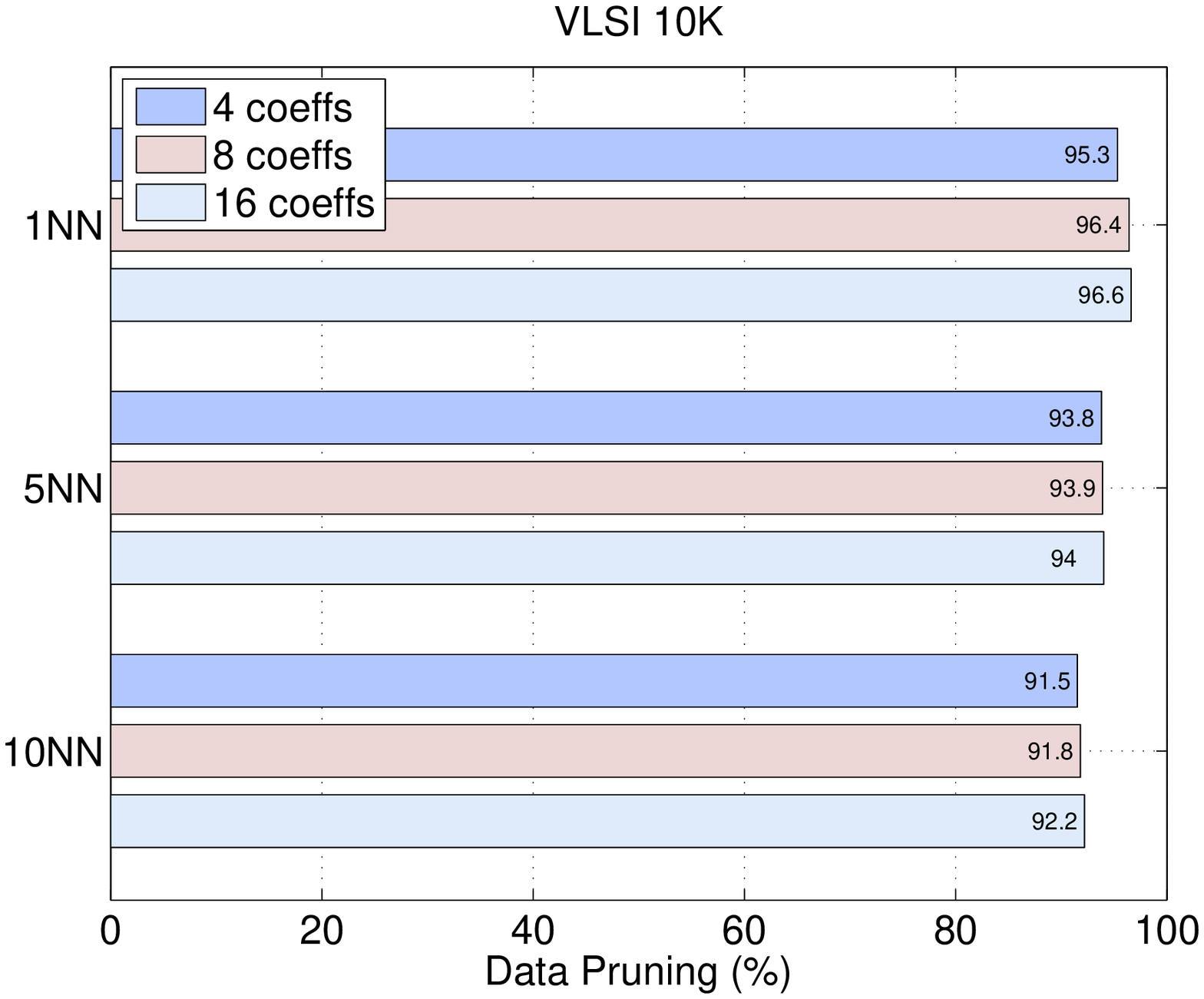} 
\includegraphics[width=0.33\textwidth]{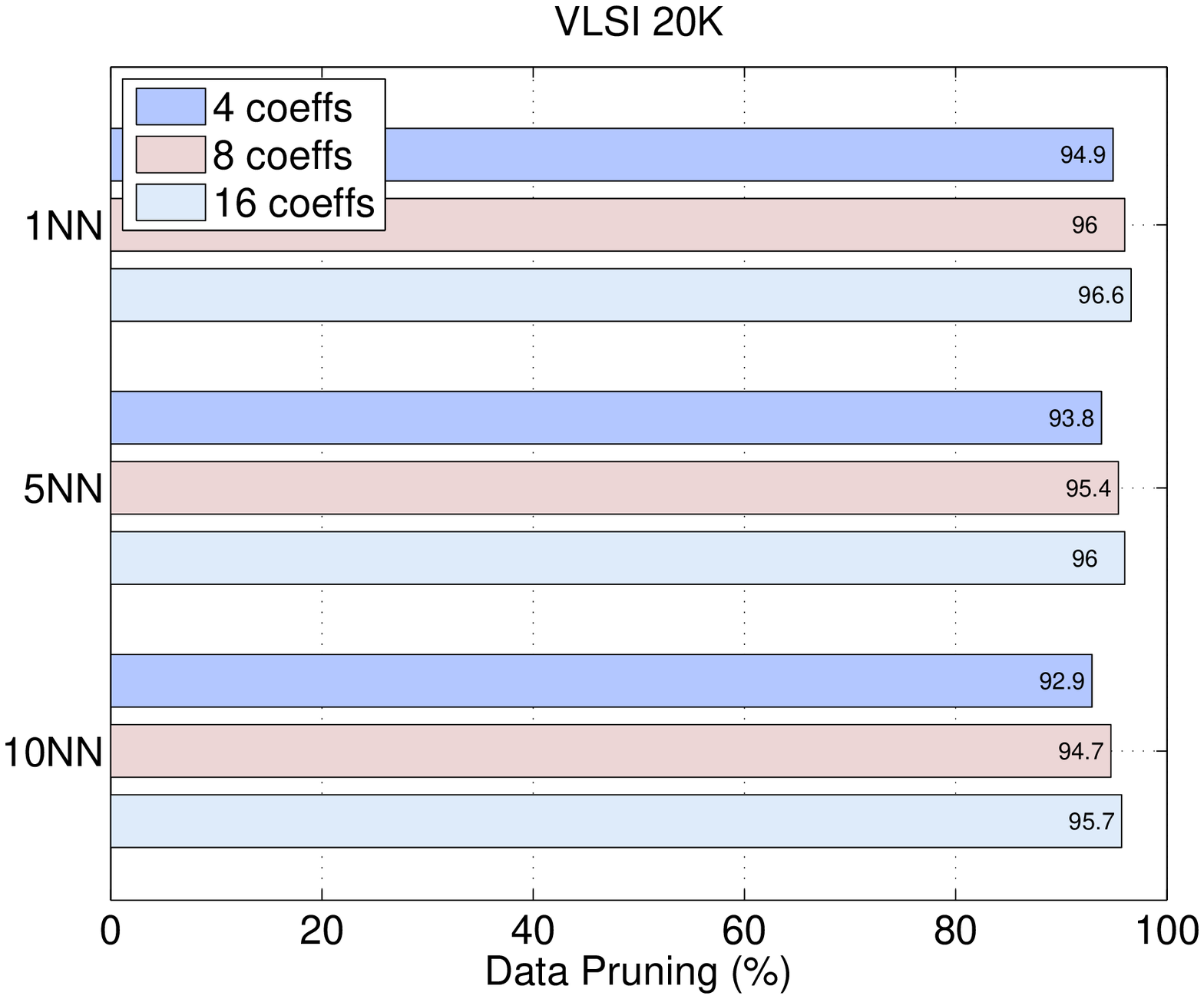}
 \includegraphics[width=0.33\textwidth]{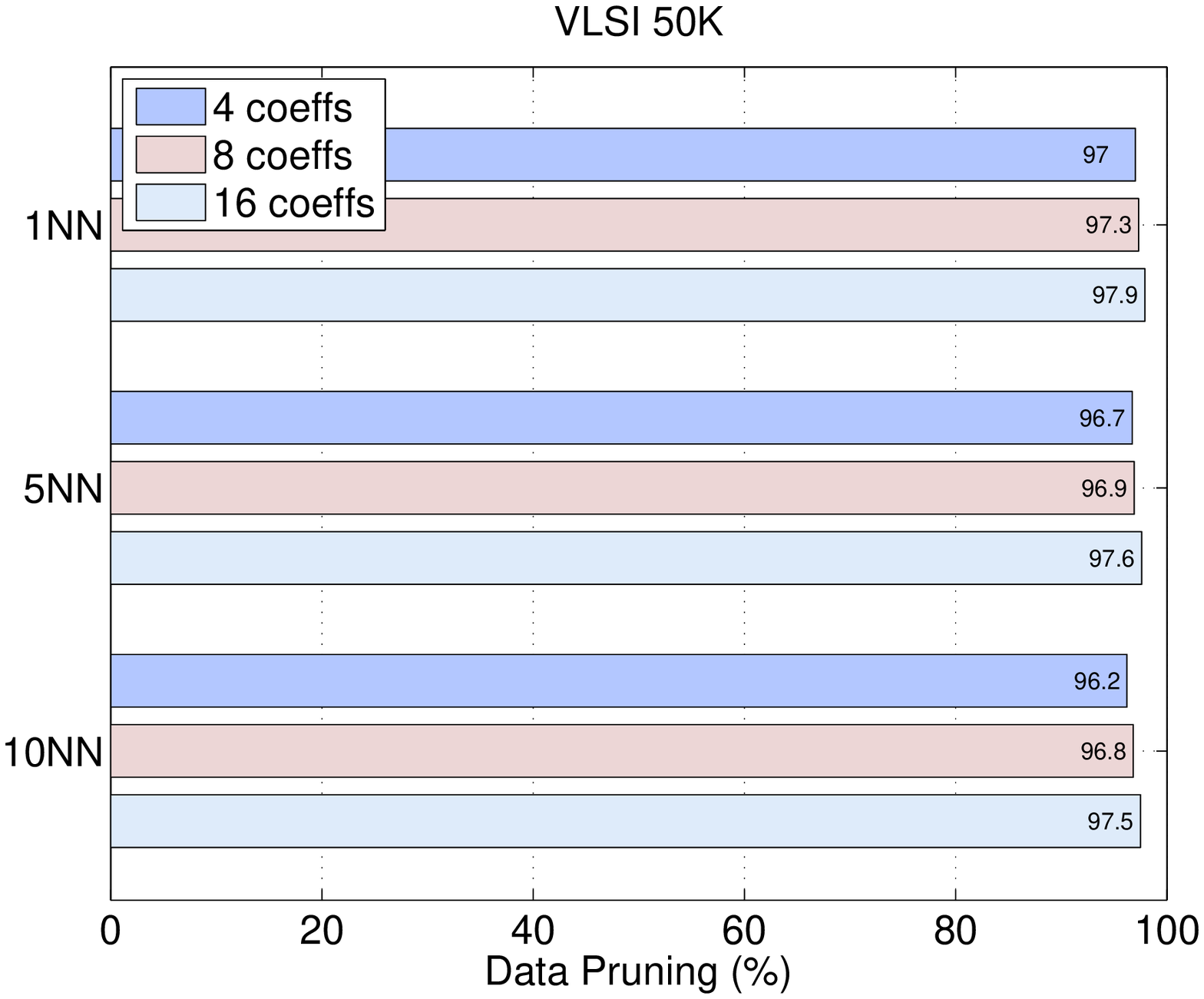} 
\caption{Pruning power of a VP-tree index when using the proposed object representation and double-waterfilling distance estimation. Top row: conservative pruning. Bottom row: aggressive pruning.}
\label{fig:index}
\end{figure*}

\subsection{Indexing}
In this final section we discuss how the proposed representation and distance estimation scheme can be leveraged for indexing. The preceding experiments have suggested that:
\begin{itemize}
    \item The proposed representation can exploit patterns in the dataset to achieve high compression. This will result in a smaller index size.
    \item The distance estimation (lower-, upper-bounds) are tighter than competitive techniques. This eventually leads to better pruning power during search.
\end{itemize}

Note that using the presented variable coefficient representation traditional space-partitioning indices, such as R-trees or KD-trees, \textit{cannot} be used.  This is because such techniques assume that each object is represented by the same set of coefficients, whereas our technique may use:  a) potentially disjoint sets of coefficients per object and/or b) variable number of coefficients per object.

Our representation can be indexed using metric trees which create a hierarchical organization of the compressed objects based on their respective distances. In previous work we have shown how VP-trees (a variant of metric trees) can be used to index representations that use variable sets of coefficients 
\cite{vlachosTWEB}. A VP-tree is constructed by recursively partitioning the objects according to their distance to some selected objects.
These objects are called \emph{vantage points} (VP), and are selected in such a way so that they provide a good separation of the currently examined subset of the dataset. 
Search and pruning of the tree is facilitated using triangle inequality. Queries are compared with the vantage point at the current tree level,
and search is directed towards the most promising part of the tree. Parts of the tree that are provably outside the search scope (invoking triangle inequality)
are pruned from examination. For additional details the interested reader can consult 
\cite{vlachosTWEB}.

\noindent 
\textit{Indexing setup: } We use three instances from the VLSI dataset: with 10K, 20K and 50K objects, compressed
using wavelet coefficients. Objects are represented in the compressed domain using $s = 4$,  $s = 8$ and  $s = 16$ coefficients per object. 
We use 100 objects as queries, which are not the same as the indexed objects.

To construct and search the tree on the compressed objects we use the presented double-waterfilling algorithm. The algorithm is used
to  compute distances both between objects for the tree construction phase, as well as for computing distances between the query posed
and the tree's vantage points. There are two search and pruning strategies one can follow: a \emph{conservative} and an \emph{aggressive} pruning strategy. For the conservative
strategy both lower and upper bounds (double-waterfilling process) from the query to each vantage point are used to navigate the tree and prune nodes. 
For the aggressive pruning strategy only the average distance $(ul+lb)/2$ is used as the distance proxy between the query and a vantage point. The aggressive
strategy achieves greater pruning but this results in slightly lower precision compared to the conservative strategy. However the precision of the aggressive search is still kept at very levels, ranging from $75-90\%$ across all experiments.

In Figure \ref{fig:index} we report the pruning power achieved with the use of the index for both conservative and aggressive pruning strategies.
Pruning power is evaluated as the number of leaves accessed over the total number of objects in the index.
We calculate the pruning power when running 1NN, 5NN and 10NN (Nearest Neighbor) search on the query objects.
One can see that with the use of indexing we can refrain from examining a very big part of the dataset, with the pruning consistently exceeding $90\%$. It is important to note that the pruning power grows for increasing dataset sizes.

\section{Conclusion}
We have examined how to compute optimally-tight distance bounds on compressed data representations
under any orthonormal transform. We have demonstrated that our methodology can 
retrieve more relevant matches than competing approaches during mining operations on the compressed data. 
A particularly interesting result is that for data with very high redundancy/sparsity (see, for example, magnetic resonance imaging \cite{candes2006stable}), our approach may even provide better search performance than compressed sensing approaches,
which have been designed to specifically tackle the sparsity issue. In such scenarios, our method may even
outperform PCA-based techniques, owing to its capability to use different sets of high-energy coefficients per object. 

As future work, we intend to continue to investigate the merits of our methodology under a broader variety of distance-based operations such as anomaly detection and density-based clustering.

\medskip \noindent 
\footnotesize{
\textbf{Acknowledgements:}
The research leading to these results has received funding from the European Research Council under the European Union's Seventh Framework Programme (FP7/2007-2013) / ERC grant agreement no. 259569.}\normalsize{ }

\bibliographystyle{ieeetr}
\bibliography{optimalBoundsBib}

\section{Appendix}
\renewcommand{\theequation}{A-\arabic{equation}}
\setcounter{equation}{0}  

\subsection{Existence of solutions and necessary and sufficient conditions for optimality}\label{app:properties}
The constraint set is a compact convex set, in fact, a compact \emph{polyhedron}.
The function $g(x,y):= -\sqrt{x}\sqrt{y}$ is convex but not strictly convex on $\R^2_+$. To see this, note that the \emph{Hessian} exists for all $x,y>0$ and equals
\vspace{-0.5\baselineskip}
\begin{equation*}
\triangledown^2g = \frac{1}{4}\left(
\begin{array}{cc}
x^{-\frac{3}{2}}y^{-\frac{1}{2}} & -x^{-\frac{1}{2}}y^{-\frac{1}{2}}\\
-x^{-\frac{1}{2}}y^{-\frac{1}{2}} & x^{-\frac{1}{2}}y^{-\frac{3}{2}}
\end{array}\right)
\end{equation*}
\noindent
with eigenvalues $0, \frac{1}{\sqrt{xy}}(\frac{1}{x} + \frac{1}{y})$, and hence is positive semi-definite, which in turn implies that $g$ is convex \cite{BV04}. Furthermore, $-\sqrt{x}$ is a strictly convex function of $x$ so that the objective function of (\ref{opt3}) is convex, and strictly convex only if $p^-_x\cap p^-_q = \emptyset$. It is also a continuous function 
so solutions exist, i.e., the optimal value is bounded and is attained.
It is easy to check that the \emph{Slater condition} holds, whence the problem satisfies \emph{strong duality} and there exist Lagrange multipliers~\cite{BV04}. We skip the technical details for simplicity, but we want to highlight that this property is crucial because it guarantees that the Karush--Kuhn--Tucker (KKT) necessary conditions~\cite{BV04} for Lagrangian optimality are also \emph{sufficient}. Therefore, if we can find a solution that satisfies the KKT conditions for the problem, we have found an \emph{exact} optimal solution and the \emph{exact} optimal value of the problem. The Lagrangian is
\vspace{-0.5\baselineskip}
\begin{small}
\begin{align}\label{lagr1}
L(\y, \z, \lambda, \mu, \bm{\alpha}, \bm{\beta}) &:= -2\!\!\sum_{i\in P_1} \!\!b_i \sqrt{z}_i - 2\!\!\sum_{i\in P_2} \!\!a_i \sqrt{y}_i - 2\!\!\sum_{i\in P_3} \!\! \sqrt{z}_i\sqrt{y}_i\\
	&	+ \lambda \Big(\sum_{i\in p^-_x}( z_i - e_x)\Big) + \mu\Big(\sum_{i\in p^-_q}( y_i - e_q)\Big)\nonumber\\
	&	+ \sum_{i\in p^-_x}\alpha_i(z_i-Z) + \sum_{i\in p^-_q}\beta_i(y_i-Y)\nonumber \;.
\end{align}
\end{small}

\noindent
The KKT conditions are as follows\footnote{The condition (\ref{O}) excludes the cases that for some $i$ $z_i = 0$, or $y_i=0$, which will be treated separately.}:
\begin{subequations}\label{kkt_opt3}
\begin{small}
\begin{eqnarray}
0\le z_i  \le Z, \ 0\le y_i  \le Y,  & \text{(PF)}\label{PF}\\
\sum_{i\in p^-_x} z_i \le e_x, \ \sum_{i\in p^-_x} z_i \le e_Q & \nonumber\\
\lambda, \mu, \alpha_i, \beta_i \ge 0 & \text{(DF)}\label{DF}\\
\alpha_i(z_i-Z) = 0, \ \ \beta_i(y_i-Y)=0 & \text{(CS)}\label{CS}\\
\lambda \Big(\sum_{i\in p^-_x}( z_i - e_x)\Big) = 0, \ \ \mu\Big(\sum_{i\in p^-_q}( y_i - e_q)\Big) = 0 \nonumber\\
i\in P_1: \  \frac{\partial L}{\partial z_i} = -\frac{b_i}{\sqrt{z_i}} + \lambda + \alpha_i = 0   &  \text{(O)}\label{O}\\
i\in P_2: \  \frac{\partial L}{\partial y_i} = -\frac{a_i}{\sqrt{y_i}} + \mu + \beta_i = 0 &  \nonumber\\
i\in P_3: \  \frac{\partial L}{\partial z_i} = -\frac{\sqrt{y_i}}{\sqrt{z_i}} + \lambda + \alpha_i = 0 & \nonumber\\
\tab         \frac{\partial L}{\partial y_i} = -\frac{\sqrt{z_i}}{\sqrt{y_i}} + \mu + \beta_i = 0 \;, & \nonumber
\end{eqnarray}
\end{small}
\end{subequations}
\noindent where we use shorthand notation for \emph{Primal Feasibility} (PF), \emph{Dual Feasibility} (DF), \emph{Complementary Slackness} (CS), and \emph{Optimality} (O)~\cite{BV04}.

\subsection{Proof of theorem \ref{prop_thm}}\label{app:proof}
For the first part, note that problem (\ref{opt3}) is a double minimization problem over $\{z_i\}_{i\in p^-_x}$ and $\{y_i\}_{i\in p^-_q}$. If we fix one vector in the objective function of (\ref{opt3}), then the optimal solution with respect to the other one is given by the waterfilling algorithm. In fact, if we consider the KKT conditions (\ref{kkt_opt3}) or the KKT conditions to (\ref{opt2}), they correspond exactly to (\ref{nec_kkt_waterfill}). The waterfilling algorithm has the property that if $\mathbf{a} = \text{waterfill }(\mathbf{b},e_x,A)$, then $b_i>0$ implies $a_i>0$. Furthermore, it has a monotonicity property in the sense that $b_i \le b_j$ implies $a_i \le a_j$.
Assume that, at optimality, $a_{l_1} < a_{l_2}$ for some $l_1\in P_1,l_2\in P_3$. Because $b_{l_1} \ge B \ge b_{l_3}$ we can swap these two values to decrease the objective function, which is a contradiction. The exact same argument applies for $\{b_l\}$, so $\min_{l\in P_1}a_l \ge \max_{l\in P_3}a_l, \min_{l\in P_2}b_l \ge \max_{l\in P_3}b_l$.

For the second part, note that $-\sum_{i\in P_3} \sqrt{z_i}\sqrt{y_i} \ge -\sqrt{e_x'}\sqrt{e_q'}$. If $e_x'e_q' >0$, then at optimality this is attained with equality for the particular choice of $\{a_l,b_l\}_{l\in P_3}$. It follows that all entries of the optimal solution $\{a_l,b_l\}_{l\in p^-_x \cup p^-_q}$ are strictly positive, hence (\ref{O}) implies that 

\vspace{-\baselineskip}
\begin{subequations}\label{kkt_pos}
\begin{eqnarray}
a_i &=& \frac{b_i}{\lambda + \alpha_i}, \ \ i\in P_1\\
b_i &=& \frac{a_i}{\mu + \beta_i}, \ \ i \in P_2\\
a_i &=& (\mu + \beta_i)b_i, \ \ i \in P_3\\
b_i &=& (\lambda + \alpha_i)a_i, \ \ i \in P_3.\nonumber
\end{eqnarray}
\end{subequations}

For the particular solution with all entries in $P_3$ equal $ \left( a_l = \sqrt{e_x'/|P_3|}, b_l = \sqrt{e_q'/|P_3|} \right)$, (\ref{lagrange_mult}) is an immediate application of (\ref{kkt_pos}.c).  The optimal entries $\{a_l\}_{l\in P_1}, \{b_l\}_{l\in P_2}$ are provided by waterfilling with available energies $e_x - e_x', e_q-e_q'$, respectively, so (\ref{non_lin}) immediately follows.

For the third part, note that the cases that either $e_x'=0,e_q'>0$ or $e_x'>0,e_q'=0$ are excluded at optimality by the first part, cf. (\ref{nec_kkt_waterfill}).

For the last part, note that when $e_x' = e_q' = 0$, equivalently $a_l=b_l = 0$ for $l\in P_3$, it is not possible to take derivatives with respect to any coefficient in $P_3$, so the last two equations of (\ref{kkt_opt3}) do not hold. In that case, we need to perform a standard perturbation analysis. Let $\bm{\epsilon} := \{\epsilon_l\}_{l \in P_1 \cup P_2}$ be a sufficiently small positive vector. As the constraint set of (\ref{opt3}) is linear in $z_i,y_i$, any \emph{feasible} direction (of potential decrease of the objective function) is of the form $z_i \leftarrow z_i - \epsilon_i, i\in P_1$, $y_i \leftarrow y_i - \epsilon_i, i \in P_2$, and $z_i,y_i \ge 0, i \in P_3$ such that $\sum_{i\in P_3}z_i = \sum_{i\in P_1}\epsilon_i, \sum_{i\in P_3}y_i = \sum_{i\in P_2}\epsilon_i$. The change in the objective function is then equal to (modulo an $o(||\bm{\epsilon}||^2)$ term)
\begin{small}
\begin{align}
g(\bm{\epsilon}) & \approx \frac{1}{2} \sum_{i \in P_1} \frac{b_i}{\sqrt{z}_i}\epsilon_i + \frac{1}{2} \sum_{i \in P_2} \frac{a_i}{\sqrt{y}_i}\epsilon_i - \sum_{i\in P_3}\sqrt{z_i}\sqrt{y_i}\\
&\ge \frac{1}{2} \sum_{i \in P_1} \frac{b_i}{\sqrt{z}_i}\epsilon_i + \frac{1}{2} \sum_{i \in P_2} \frac{a_i}{\sqrt{y}_i}\epsilon_i - \sqrt{\sum_{i\in P_1}\epsilon_i}\sqrt{\sum_{i\in P_2}\epsilon_i} \nonumber\\
&\ge \frac{1}{2}\min_{i\in P_1}\frac{b_i}{\sqrt{z}_i}\epsilon_1 + \frac{1}{2} \min_{i \in P_2} \frac{a_i}{\sqrt{y}_i}\epsilon_2 - \sqrt{\epsilon_1}{\epsilon_2},\nonumber
\end{align}
\end{small}
\noindent where the first inequality follows from an application of the Cauchy--Schwartz inequality to the last term, and in the second one we have defined $\epsilon_j = \sum_{i\in P_j}\epsilon_i, i=1,2$.
Let us define $\epsilon := \sqrt{\epsilon_1/\epsilon_2}$. 
From the last expression, it suffices to test for any $i\in P_1,j\in P_2$:
\begin{small}
\begin{align}
g(\epsilon_1,\epsilon_2) &= \frac{1}{2} \frac{b_i}{\sqrt{z}_i}\epsilon_1 + \frac{1}{2} \frac{a_j}{\sqrt{y}_j}\epsilon_2 - \sqrt{\epsilon_1}\sqrt{\epsilon_2} = \frac{1}{2}\sqrt{\epsilon_1}\sqrt{\epsilon_2} g_1(\epsilon)\\
g_1(\epsilon) &:= \frac{b_i}{\sqrt{z}_i}\epsilon + \frac{a_j}{\sqrt{y}_j}\frac{1}{\epsilon} -2 \ge \frac{1}{\epsilon}g_2(\epsilon)\nonumber\\
g_2(\epsilon) &:=  \frac{b_i}{A}\epsilon^2 -2\epsilon + \frac{a_i}{B},\nonumber
\end{align}
\end{small}
\noindent where the inequality above follows from the fact that $\sqrt{z_i} \le A, i\in P_1$ and $\sqrt{y_i} \le B, i\in P_2$.
Note that $h(\epsilon)$ is a quadratic with a nonpositive discriminant $\Delta := 4(1 - \frac{a_ib_i}{AB}) \le 0$ as, by definition, we have that $B \le b_i, i\in P_1$ and $A\le a_i, i\in P_2$. Therefore $g(\epsilon_1,\epsilon_2) \ge 0$ for any $(\epsilon_1,\epsilon_2)$ both positive and sufficiently small, which is a necessary condition for local optimality. By convexity, the vector pair $(\mathbf{a},\mathbf{b})$ obtained constitutes an optimal solution. 
\hfill$\blacksquare$

\subsection{Energy allocation in double-waterfilling}\label{app:fixed_point}

Calculating a fixed point of $T$ is of interest only if $e_x'e_q'>0$ at optimality. We know that we are not in the setup of Theorem \ref{prop_thm}.4, therefore we have the additional property that either $e_x > |P_1|A^2$,  $e_q > |P_2|B^2$ or both. 
Let us define
\begin{small}
\begin{eqnarray}
\tab \gamma_a := \inf \Big\{\gamma> 0: \sum_{l \in P_2}\min \Big( a_l^2\frac{1}{\gamma},B^2 \Big) \le e_q\Big\}\label{gammaab}\\
\tab \gamma_b := \sup \Big\{\gamma\ge 0: \sum_{l \in P_1}\min \Big( b_l^2\gamma,A^2 \Big)\le e_x\Big\}.\nonumber
\end{eqnarray}
\end{small}

\noindent
Clearly if $e_x > |P_1|A^2$ then $\gamma_b = +\infty$, and for any $\gamma\ge \max_{l\in P_1}\frac{A^2}{b_l^2}$ we have $ \sum_{l \in P_1}\min(b_l^2\gamma,A^2) = |P_1|A^2$. Similarly, if $e_q>|P_2|B^2$ then $\gamma_a = 0$, and for any $\gamma\le \min_{l\in P_2}\frac{a_l^2}{B^2}$ we have $ \sum_{l \in P_2}\min(a_l^2\frac{1}{\gamma},B^2) = |P_2|B^2$. 
If $\gamma_b < +\infty$, we can find the exact value of $\gamma_b$ analytically by sorting $\{\gamma_l^{(b)}:=\frac{A^2}{b_l^2}\}_{l\in P_1}$, --i.e., by sorting $\{b_l^2\}_{P_1}$ in decreasing order--and considering
\begin{equation*}
h_b(\gamma) := \sum_{l \in P_1}\min(b_l^2\gamma_i^{(b)},A^2) - e_x
\end{equation*}
and $v_i := h_b(\gamma_i^{(b)})$. In this case,  $v_1<\hdots<v_{|P_1|}$, and $v_{|P_1|}>0$, and there are two possibilities: 1) $v_1>0$ whence $\gamma_b < \gamma_1^{(b)}$, or 2) there exists some $i$ such that $v_i<0<v_{i+1}$ whence $\gamma_i^{(b)} < \gamma_b < \gamma_{i+1}^{(b)}$. For both ranges of $\gamma$, the function $h$ becomes \emph{linear} and strictly increasing, and it is elementary to compute its root $\gamma_b$.
A similar argument applies for calculating $\gamma_a$ if $\gamma_a$ is strictly positive, by defining 

$$h_a:= \sum_{l \in P_2}\min \Big( a_l^2\frac{1}{\gamma},B^2 \Big) - e_q$$.
%
\vspace{-0.1cm}
\begin{figure}[!ht]
\centerline{
\includegraphics[width=0.5\textwidth]{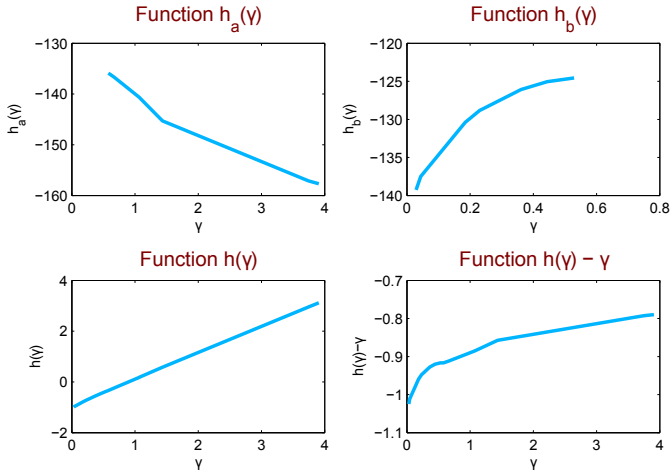}}
\vspace{-\baselineskip}
\caption{Plot of functions $h_a,h_b,h$. Top row: $h_a$ is a bounded decreasing function, which is piecewise linear in $\frac{1}{\gamma}$ with nonincreasing slope in $\frac{1}{\gamma}$; $h_b$ is a bounded increasing piecewise linear function of $\gamma$ with nonincreasing slope. Bottom row: $h$ is an increasing function; the linear term $\gamma$ dominates the fraction term, which is also increasing, see  bottom right.}
\label{fig:hplot}
\end{figure}
\vspace{-0.2cm}

\begin{theorem}[Exact solution of (\ref{non_lin})]\label{opt_lem}~\\
If either $e_x > |P_1|A^2$, $e_q > |P_2|B^2$  or both, then
the nonlinear mapping $T$ has a unique fixed point $(e_x', e_q')$ with $e_x',e_q'>0$.
The equation
\begin{equation}
\frac{e_x - \sum_{l \in P_1}\min(b_l^2\gamma,A^2)}{e_q - \sum_{l \in P_2}\min(a_l^2\frac{1}{\gamma},B^2)} = \gamma
\end{equation}
has a unique solution $\bar{\gamma}$ with $\gamma_a \le \bar{\gamma}$ and  $\gamma_a \le \gamma_b$ when $\gamma_b <+\infty$.
The unique fixed point of $T$ (solution of (\ref{non_lin})) satisfies
\begin{eqnarray}\label{exeq_res}
e_x' &=& e_x - \sum_{l \in P_1}\min \left( b_l^2\bar{\gamma},A^2 \right)\\
e_q' &=& e_q - \sum_{l \in P_2}\min \left( a_l^2\frac{1}{\bar{\gamma}},B^2 \right).\nonumber
\nonumber
\end{eqnarray}
\end{theorem}
\begin{proof}
Existence\footnote{An alternative and more direct approach of establishing the existence of a fixed point is by considering all possible cases and defining an appropriate compact convex set $E\subset \mathbb{R}^2_+ \setminus (0,0)$ so that $T(E)\subset E$, whence existence follows by the Brower's fixed point theorem~\cite{basar}, as $T$ is continuous.} of a fixed point is guaranteed by existence of solutions and Lagrange multiplies for (\ref{opt3}), as by assumption we are in the setup of Theorem \ref{prop_thm}.2.
Define $\gamma := \frac{e_x'}{e_q'}$; a fixed point $(e_x',e_q') = T((e_x',e_q')), e_x',e_q'>0$, corresponds to a root of
\begin{equation}
h(\gamma):= -\frac{e_x - \sum_{l \in P_1}\min(b_l^2\gamma,A^2)}{e_q - \sum_{l \in P_2}\min(a_l^2\frac{1}{\gamma},B^2)} + \gamma
\end{equation}
For the range $\gamma \ge \gamma_a$ and $\gamma\le \gamma_b$, if $\gamma_b<+\infty$, we have that $h(\gamma)$ is continuous and strictly increasing.  
The fact that $\lim_{\gamma\searrow \gamma_a} h(\gamma)<0,\lim_{\gamma\nearrow\gamma_b} h(\gamma)>0$ shows the existence of a unique root $\bar{\gamma}$ of $h$ corresponding to a unique fixed point of $T$, cf. (\ref{exeq_res}).  \hfill$\blacksquare$
\end{proof}

\begin{remark}[Exact calculation of a root of $h$]\label{root_calc}
We seek to calculate the root of $h$ exactly and efficiently. In doing so, consider the points $\{\gamma_l\}_{l\in P_1\cup P_2}$, where
$\gamma_l := \frac{A}{b_l^2}, \ l\in P_1, \ \gamma_l := \frac{a_l^2}{B}, \ l\in P_2$. Then, note that for any $\gamma \ge \gamma_l, l\in P_1$ we have that $\min(b_l^2\gamma,A^2) = A^2$.
Similarly, for any $\gamma \le \gamma_l, l\in P_2$, we have that $\min(a_l^2\frac{1}{\gamma},B^2) = B^2$.
We order all such points in increasing order, and consider the resulting vector $\bm{\gamma}':=\{\gamma_i'\}$  excluding any points below $\gamma_a$ or above $\gamma_b$. Let us define $h_i := h(\gamma_i')$.
If for some $i$, $h_i = 0$ we are done. Otherwise there are three possibilities:
1) there is an $i$
%
such that $h_i< 0 < h_{i+1}$;
2) $h_1>0$, or 3) $h_N<0$. In all cases, the numerator (denominator) of $h$ is linear in $\gamma$ ($\frac{1}{\gamma}$) for the respective range of $\gamma$. Therefore, 
$\bar{\gamma}$ is obtained by solving the linear equation
\begin{equation}
f(\gamma) := e_x - \sum_{l \in P_1}\min(b_l^2\gamma,A^2) - \gamma \left( e_q - \sum_{l \in P_2}\min \left( a_l^2\frac{1}{\gamma},B^2 \right) \right).
\end{equation}
\noindent
Note that there is no need for further computation to set this into the form $f(\gamma) = \alpha\gamma + \beta$ for some $\alpha,\beta$. Instead, we use 
the elementary property that a linear function $f$ on $[x_0,x_1]$ with $f(x_0)f(x_1)<0$ has a unique root given by
$$\bar{x} = x_0 -\frac{x_1-x_0}{f(x_1)-f(x_0)}f(x_0)  \;\;. $$
\end{remark}
\end{document}